\documentclass[twoside]{article}
\ifdefined\pdfminorversion\pdfminorversion=7\fi
\usepackage[preprint]{aistats2027}
\usepackage[T1]{fontenc}
\usepackage{lmodern}
\usepackage[utf8]{inputenc}
\usepackage{amsmath,amssymb,amsthm,booktabs,graphicx,array}
\usepackage[round]{natbib}

\usepackage[hidelinks]{hyperref}
\usepackage{alphalph}
\usepackage{enumitem}
\setlist{topsep=2pt, partopsep=0pt, itemsep=1pt, parsep=0pt}
\setlist[itemize,1]{leftmargin=1.1em, labelsep=0.4em}
\usepackage{caption}
\newcommand{\ci}[2]{{\scriptsize\,[#1,\,#2]}}
\usepackage{tikz}
\usetikzlibrary{arrows.meta}
\definecolor{figwash}{HTML}{EEF4FB}
\definecolor{figaverage}{HTML}{EAF4F1}
\definecolor{figgray}{HTML}{F4F4F2}
\definecolor{figaxis}{HTML}{8A8984}
\definecolor{figref}{HTML}{5C5B56}
\pdftrailerid{}
\newtheorem{theorem}{Theorem}
\newtheorem{proposition}[theorem]{Proposition}
\newtheorem{corollary}[theorem]{Corollary}
\newcommand{\E}{\mathbb E}
\newcommand{\Pp}{\mathbb P}

\newcommand{\norm}[1]{\lVert#1\rVert}

\hypersetup{pdfauthor={Haiyi Li},pdftitle={Sampling and Curvature in First-Exit Comparisons for Momentum SGD}}
\begin{document}
\runningtitle{Sampling and Curvature in First-Exit Comparisons}
\runningauthor{Haiyi Li}
\twocolumn[%
\aistatstitle{Sampling and Curvature in First-Exit Comparisons\\
for Momentum SGD}%
\aistatsauthor{Haiyi Li}%
\aistatsaddress{Harvard University}%
]
\begin{abstract}
First-exit prediction for momentum SGD depends on sampling information
that pointwise gradient laws omit. Strongly convex finite pools with
identical pointwise laws and matched iid momentum dynamics can have
different reshuffled exits; even analytic pools with uniformly close
pointwise laws can have separated exit probabilities. For persistent
sampling in one dimension, we establish an explicit $O(h)$ mean-exit
comparison with a finite-rate switching process under nonvanishing
component velocities and step-size conditions; iid sampling instead
averages to an ODE. Retaining realized minibatches and sampled curvature
yields finite-window exit brackets beyond this regime. In 32 trained
ridge-logistic runs, 29 meet prespecified targets with rounding-certified
brackets; an enclosure control explains most of the width advantage.
Exact-real certificates establish finite means at the original tested
radii for three ResNet-18 checkpoints and give local finite-mean
conditions for tanh and ViT networks; their upper bounds remain loose,
and original-radius ViT mean guarantees remain unresolved despite
128 fresh continuation paths.
A two-network replication gives 32/32 successful seeds in each of four
finite-path persistence comparisons. Independent equal-length controls
retain slower exits after refresh resets than without resets, but the
refresh-versus-random ordering reverses with update length at both networks.
These interventions do not causally decompose raw momentum SGD. Neural
step-size scaling does not validate the switching limit.
\end{abstract}

\section{Introduction}
First-exit times measure how an optimizer leaves a prescribed region;
escape from sharp minima is a standard example
\citep{zhu2019anisotropic,xie2021diffusion}. For momentum SGD they are
usually predicted with continuum models, such as stochastic differential
equations or colored-noise and telegraph approximations, that are
calibrated from the one-step law of the minibatch gradient. Practical samplers have memory: random reshuffling (RR) visits every
example once per epoch, while data echoing repeats the active batch.
We ask what exit prediction must retain, and when its continuum
reference is a limit of the discrete algorithm.

The sampler matters beyond the one-step gradient law. In trained
networks with matched marginal gradient laws, reusing each
minibatch four times on average gave observed capped exit means
31--60\% smaller, while reusing it fifty times increased them by
factors of 1.4--3.4 (Section~\ref{sec:networks}); no model
calibrated from one-step noise statistics can express this. Four results
support the claim from complementary directions:
\begin{itemize}
\item \textbf{Information.} Pointwise gradient laws do not determine
reshuffled exits, not even approximately (Section~\ref{sec:information}).
\item \textbf{Continuum limits.} Persistent sampling gives heavy ball an
explicit $O(h)$ mean-exit bracket against a finite-rate switching
process; iid sampling averages to an ODE (Section~\ref{sec:bridge}).
\item \textbf{Computable comparisons.} Retaining the realized sampler and
sampled curvature gives certified exit brackets; 29 of 32 trained
ridge-logistic pools meet declared targets (Section~\ref{sec:curvature}).
\item \textbf{Trained networks.} Prospective studies on five datasets
document faster exits at intermediate persistence and the limits of this
effect (Section~\ref{sec:networks}).
\end{itemize}

\section{Setting}\label{sec:process}
A finite pool of $N$ examples defines minibatch gradients
$g_B(w)=|B|^{-1}\sum_{i\in B}\nabla f_i(w)$, including any ridge term.
Let $x_n$ be the displacement from a center $w_*$ and $u_n$ the momentum
in gradient units:
\[
u_n=\beta u_{n-1}-g_{B_n}(w_*+x_{n-1}),\qquad x_n=x_{n-1}+hu_n .
\]
The sampler generates the minibatches $B_1,B_2,\dots$. We compare three
samplers with the same one-step law, in which each $B_n$ is a uniformly
random $b$-subset of the pool marginally (with a uniform initial batch):
\begin{itemize}
\item \emph{iid}: draws a fresh subset at every update;
\item \emph{random reshuffling}: splits a fresh permutation of the pool
into consecutive batches in every epoch;
\item \emph{persistent}: keeps $B_{n-1}$ with probability $1-p$ and
otherwise draws a fresh subset. This is data echoing
\citep{choi2019faster} or minibatch persistency
\citep{fischetti2018faster} with geometric echo counts.
\end{itemize}
These equal batch marginals do not imply equal conditional gradient laws
at the evolving iterate: the iterate and the remaining pool can be dependent.
The event of interest is the integer-observed exit
$\tau=\inf\{n\ge1:x_n\notin\mathcal D\}$ from an open region
$\mathcal D\ni x_0$, for example a ball or, for a projection, an interval.
Over a window of $J$ updates we summarize it by its distribution function
and its restricted mean $\E\min(\tau,J)$ (RMST). A noise model that
retains only the one-step law assigns the same exit law to all three
samplers.

\section{Pointwise gradient laws do not determine exits}\label{sec:information}
\paragraph{Reshuffling separates identical pointwise laws.}
A pointwise oracle returns the gradient law
$\nu_x=N^{-1}\sum_i\delta_{\nabla f_i(x)}$ at each query $x$, but not
which sample produced which gradient. Under iid sampling this suffices:
it specifies the transition kernel of the full position--velocity state,
hence the path law. Under reshuffling the state also contains the unused
labels. We construct two globally strongly convex finite pools that agree
on every $\nu_x$ (Appendix~\ref{sec:pointwise-law-rr}). Iid sampling gives
identical full momentum-state path laws, but reshuffling gives two-step
exit probabilities $1/2$ and $0$ on a nonempty interval of radii. A
predictor that returns the same number $q$ for both pools errs by
$\max\{|q-1/2|,|q|\}\ge1/4$ on one of them, even when given the entire
pointwise law. The bound does not apply to methods that receive labeled
component functions or the remaining pool, which are the inputs of our
positive comparisons.

\paragraph{Analytic pools: identification without stability.}
For analytic component losses, equality of pointwise laws forces a fixed
global relabeling. Approximate equality, however, gives no stability. We
construct analytic pools with fixed strong-convexity and smoothness
constants whose pointwise $W_\infty$ distance tends to zero, while their
two-step RR exit CDFs remain $1/2$ apart
(Appendix~\ref{sec:analytic-identification}).

The separation holds on an interval of radii and an explicit
neighborhood of the learning rate, momentum and initial state. A
stability assumption must therefore control label correspondence or
function complexity, not merely local derivatives; fixed polynomial
degree restores stability through Remez bounds
(Appendix~\ref{sec:polynomial-matching}), but unrestricted networks are
not covered. The question is linked to phase-retrieval stability
\citep{brudnyi2013norming,bandeira2014saving}.

\paragraph{Covariance and a constructive counterpart.}
Mean and covariance functions can be insufficient even under iid
sampling (Appendix~\ref{sec:full-covariance-pairing}). Conversely, a
gradient table that keeps sample labels on nine grid nodes (18 values
per pool) yields disjoint exact two-step exit brackets on the constructed pair
(Appendix~\ref{sec:labelled-grid}). Retained sample identity is what
makes the exit law computable here.

\section{Sampler persistence and the telegraph limit}\label{sec:bridge}
Finite-rate telegraph and run-and-tumble processes have closed-form mean
exit times with a ballistic boundary correction
\citep{angelani2014first}. Linear Gaussian momentum recursions cannot
converge to them (Appendix~\ref{app:obstructions}), so persistence must
come from somewhere other than the momentum filter.

Consider sampler states $s\in\mathcal S$ (samples or minibatches) with
component gradients $g_s$; the active state $S_n$ replaces the minibatch
$B_{n+1}$ of Section~\ref{sec:process}. A persistent sampler has
transition matrix $I+hQ$; lazy resampling, for example, has
$Q=(p/h)(\Pi-I)$ with $\Pi_{ss'}=1/|\mathcal S|$. In one dimension, with
$\mathcal D$ an open interval $\mathcal I$, let $\mathsf T_s(x)$ be the
mean exit time from $\mathcal I$ of the switching process
$dY_t=v_{S(t)}(Y_t)\,dt$ started at $(x,s)$, where $v_s=-g_s/(1-\beta)$
and the state jumps with generator $Q$. For two states, lazy resampling
gives flip rate $p/(2h)$, not the rate $p/h$ of an explicit flip with
probability $p$.

\begin{theorem}[Persistent-sampling bridge]\label{thm:telegraph-bridge}
Fix $0\le\beta<1$, a finite irreducible generator $Q$, and
$C^1$ component velocities. Suppose every $v_s$ keeps a constant sign,
with $|v_s|\ge c_0>0$, near
$\bar{\mathcal I}$. Let $U=\max(|u_0|,\sup|v_s|)$ and
$\epsilon=h\beta U/(1-\beta)$. For every $h$ below the explicit threshold
given below, $h\E\tau$ lies between the mean exit times from $\mathcal I$ shrunk and
enlarged by $\epsilon$, up to explicit $O(h)$ terms. For a shifted start inside the shrunk interval,
$|h\E\tau-\mathsf T_{S_0}(x_0)|\le C_Ah$, with $C_A$ computed from the
backward equation of the switching process (Theorem~\ref{br:thm:A}).
\end{theorem}

\emph{Proof idea.} The averaged iterate
$z_n=x_n+h\beta u_n/(1-\beta)$ satisfies $z_{n+1}=z_n+hv_{S_n}(x_n)$
exactly \citep{sebbouh2021almost} and stays within $\epsilon$ of $x_n$,
so exits of $x$ are sandwiched between exits of $z$ from the shrunk and
enlarged intervals. Along $z$, the backward solution
$\mathsf T_{S_n}(z_n)$ decreases by $h$ per step up to $O(h^2)$, and
optional stopping, first at bounded times to establish integrability,
turns this into the bracket; a sampler jump on the exit-causing update
adds a separate $O(h)$ boundary term. Uniqueness and a maximum
principle for the backward system come from its PDMP representation
\citep{davis1993markov}.

The switching state is the active sample, not a momentum sign. The sign
condition concerns each component, not the mean: the mean gradient may
vanish inside $\mathcal I$, as in the well, saddle and logistic pools of
Figure~\ref{fig:bridge}, but no single sampler state may have a
stationary point there. In this limit momentum enters only through the speed factor
$1/(1-\beta)$ and the lag $\epsilon$, the distance traveled on residual
momentum after a switch. A continuous analogue,
$\ddot X=-\mu\dot X-g_S(X)$ with friction $\mu$ and switching slowed by
$1/\mu$, has the same limit in time $t/\mu$, with error $O(\mu^{-2})$
for initial velocity $O(1/\mu)$
(Theorem~\ref{br:thm:B}). A Brownian (Kac) limit requires the switching
rate to diverge (Proposition~\ref{br:prop:kac}).

Under iid sampling the switching probability does not vanish, so the
rate diverges. Heavy ball then averages to an ODE on the time scale
$t=nh$ (Proposition~\ref{br:prop:iid}), and a finite-rate telegraph is not
a limit on that time scale. Reshuffling also refreshes the batch at every
update, so Theorem~\ref{thm:telegraph-bridge} does not cover it either. The three regimes $p/h\to\infty$,
$p/h\to\lambda$ and $p/h\to0$ give the averaged ODE, the switching limit
and the frozen sample, as for Markov-modulated stochastic approximation
\citep{yin2004regime}. Figure~\ref{fig:limits} summarizes how these
limits fit together.

\begin{figure}[t]
\centering
\includegraphics[width=\columnwidth]{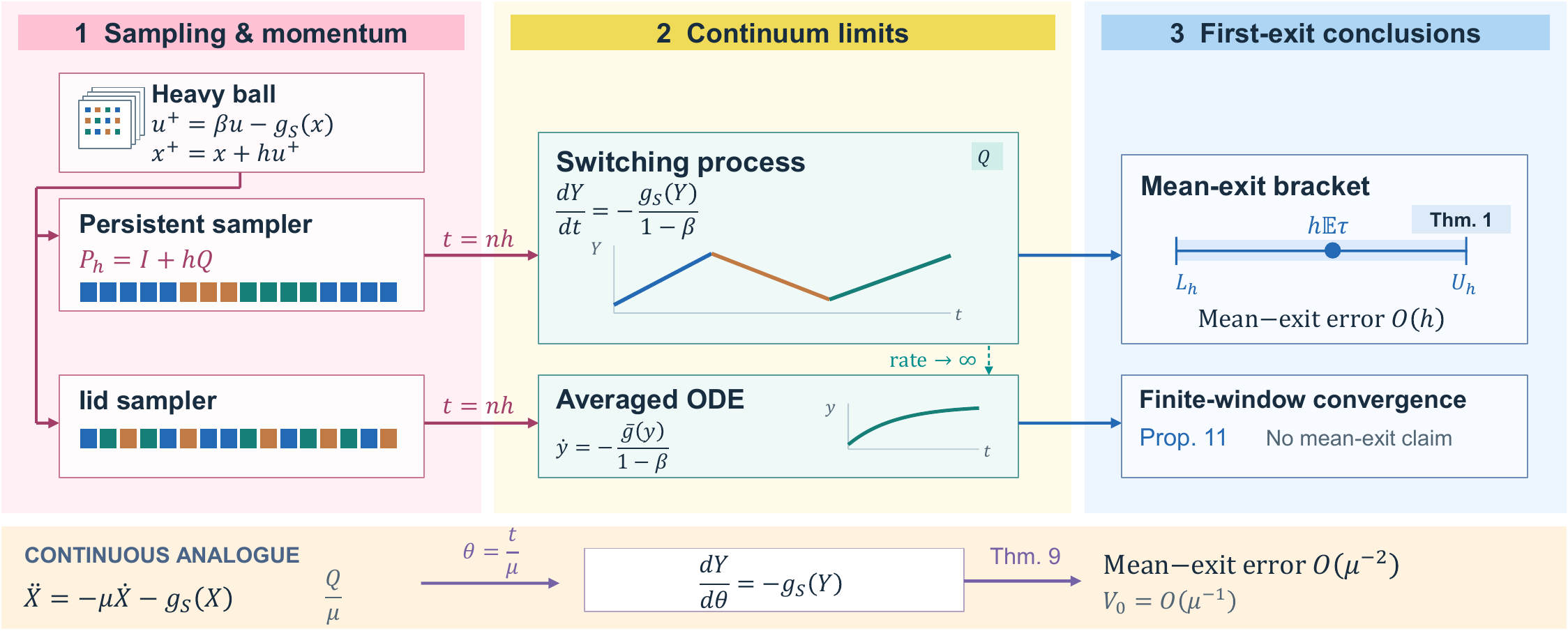}
\caption{Sampling determines the limit. Discrete models use $t=nh$;
second-order dynamics use $\theta=t/\mu$. Paths and intervals are schematic;
$L_h,U_h$ denote the full mean-exit bounds of Theorem~\ref{thm:telegraph-bridge}.
All conclusions require the cited assumptions; iid trajectory convergence
alone does not control complete mean exit times.}
\label{fig:limits}
\end{figure}

Of 120 prespecified two-sample cells, 68 satisfy the declared-margin
conditions; a later minimum-margin analysis admits 82. Both analyses
contain every applicable 99.9\% Monte Carlo interval in its bracket.
Errors are consistent with $O(h)$, also with eight sampler states
(Figure~\ref{fig:bridge}; Appendix~\ref{app:telegraph-bridge}). The
explicit constant is conservative, by a median factor of 8.2.

\begin{figure}[t]
\centering
\includegraphics[width=\columnwidth]{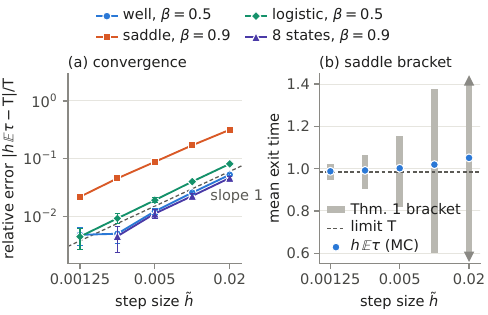}
\caption{Persistent-sampling bridge. (a) Relative error of $h\E\tau$
against the switching-process mean $\mathsf T$ for three two-sample
pools and an eight-sample pool; steps are $h=\tilde h(1-\beta)$ and bars
are 95\% Monte Carlo intervals. (b) Bracket of
Theorem~\ref{thm:telegraph-bridge} at each step (gray; arrows mark ends
beyond the axis), Monte Carlo means and the limit $\mathsf T$ for the
saddle pool with $\beta=0.5$.}
\label{fig:bridge}
\end{figure}

\paragraph{Explicit constants.}
The reference solves the backward system
$v_s\mathsf T_s'+\sum_{s'}Q_{ss'}(\mathsf T_{s'}-\mathsf T_s)=-1$, with
zero boundary values only for outgoing states (right-moving states at
the right endpoint, left-moving at the left); the $\epsilon$-enlarged and
$\epsilon$-shrunk intervals need their own solutions $\mathsf T^+$ and
$\mathsf T^-$. Let
$V=\max_s\sup|v_s|$, $L=\max_s\sup|v_s'|$,
$\bar q=\max_s\sum_{s'\ne s}Q_{ss'}$, and let $M_k^\pm$ bound the $k$th
derivative of $\mathsf T^\pm$ ($k=0,1,2$). With
$z_0=x_0+h\beta u_0/(1-\beta)$,
\begin{align*}
 K_\pm&=\tfrac{L\beta U M_1^\pm}{1-\beta}
       +\tfrac{V^2M_2^\pm}{2}+2\bar qVM_1^\pm,\\
 B_\pm&=VM_1^\pm+\bar q M_0^\pm,\quad
 C_\pm=B_\pm+
 \tfrac{K_\pm(\mathsf T^\pm_{S_0}(z_0)+hB_\pm)}{1-K_\pm h},
\end{align*}
the bracket is
$\mathsf T^-_{S_0}(z_0)-hC_-\le h\E\tau\le\mathsf T^+_{S_0}(z_0)+hC_+$
for a nonempty shrunk interval when $K_\pm h<1$ and
$\max(2\epsilon,\epsilon+hV)$ is at most the
verified extension radius (the lower bound is zero if $z_0$ lies outside
the shrunk interval). The three terms of $K_\pm$ control momentum
lag, Taylor error and switching during an update, so a small learning
rate alone does not guarantee applicability; the 38 study cells that
fail these conditions are reported as failures. For the flat pool
$g_\pm=\mp q$ on $(-R,R)$ with flip rate $\lambda$, $c=q/(1-\beta)$ and
$D=c^2/(2\lambda)$, the reference mean from $0$ is
$R^2/(2D)+R/c$, so a matched diffusion misses the finite-speed
correction $R/c$.

\section{Computable comparisons that retain sampled curvature}\label{sec:curvature}
The bridge needs one dimension, sign-definite velocities and small
steps. Without these assumptions, we compare the actual process with a
reference that retains the realized minibatches and their curvature, and
bound the difference pathwise. For the realized batches $B_n$, the affine
reference is
\[
v_n=\beta v_{n-1}-c_{B_n}-H_{B_n}y_{n-1},\qquad y_n=y_{n-1}+hv_n,
\]
with $c_B=g_B(w_*)$ and $H_B=\nabla g_B(w_*)$. Its initial position and
velocity agree with the actual process, and $\|x_0\|_2<R$.

Suppose known nondecreasing functions $\eta_B$ bound the regional
remainder, $\|g_B(w_*+z)-c_B-H_Bz\|_2\le\eta_B(r)$ for $\|z\|_2\le r\le R$.
With the transition matrices $A_n$ and injection $D$ of
Appendix~\ref{sec:heterogeneous}, let $K_{n,j}=PA_n\cdots A_{j+1}D$,
$E_0=0$ and
\[
E_n=\sum_{j=1}^n\|K_{n,j}\|_2\,\eta_{B_j}\bigl(\min\{R,\|y_{j-1}\|_2+E_{j-1}\}\bigr).
\]
$E_n$ bounds the position error through the exit-causing update. The
reference exits $T_\pm$, defined by the tests $\|y_n\|_2\mp E_n\ge R$,
therefore satisfy $T_-\le\tau\le T_+$ pathwise. This yields finite-window
CDF and restricted-mean brackets under any sampler, with no independence
assumption.

The sampling schedule also constrains reachability. For a stable scalar
common-curvature model, Theorem~\ref{thm:reach} computes the exact
reachable supremum under epoch permutations. Boundaries strictly above
it are never reached under RR, while iid sampling exits with finite mean.

\subsection{Constructing the tube and reading its precision}
The matrices needed for the comparison are explicit:
\[
 A_n=\begin{pmatrix}I-hH_{B_n}&h\beta I\\-H_{B_n}&\beta I\end{pmatrix},
 \quad D=\binom{hI}{I}.
\]
Set $P=\begin{pmatrix}I&0\end{pmatrix}$ and $r_B(z)=g_B(w_*+z)-c_B-H_Bz$. Subtracting the two momentum
recursions and iterating gives
\[
 x_n-y_n=-\sum_{j=1}^n K_{n,j}r_{B_j}(x_{j-1}).
\]
For $n\le\tau$, every input $x_{j-1}$ lies inside the region.
Induction gives its admissible radius
$\min\{R,\|y_{j-1}\|+E_{j-1}\}$ and hence the displayed tube.
This includes $n=\tau$: it is the input to the exit-causing update,
not its output, that must remain inside the region. Exit before $T_-$,
or $T_+$ before exit, would contradict the triangle inequality, which
establishes the bracket.
The propagators depend on later sampled batches, so they are not used
as predictable martingale weights; any verified upper bound on
$\|K_{n,j}\|$ can replace its norm. For history-dependent samplers the
comparison is coupled along the realized batches; an autonomous
reference that generates its own histories need not reproduce it.

\paragraph{A nonempty computable region.}
For ridge-logistic gradients
$g_i(w)=(\sigma(a_i^\top w)-b_i)a_i+\lambda w$, the global bound
$\sup|\sigma''|=1/(6\sqrt3)$ gives
\[
 \eta_B(r)=\frac{r^2}{12\sqrt3|B|}\sum_{i\in B}\|a_i\|_2^3.
\]
This bound is quadratic in the radius. For a common Hessian $\bar H$,
whose curvature error is discussed by \citet[Appendix C.1]{domingo2022},
replace this by $\eta_B(r)+\|H_B-\bar H\|_2r$: sampled curvature removes
the linear heterogeneity term, which explains a small-radius advantage but not
universal width dominance, since the propagators also differ. The
constants cover every point of the ball and every batch of the pool,
which residuals measured on saved trajectories do not.

Whole-parameter tanh bounds fail the overall numerical validation gate
(Appendix~\ref{app:fullparam-followup}). Tail certificates cover tanh,
local ViT regions and original ResNet radii
(Appendix~\ref{app:complete-head-cones}), with loose bounds.
For original-radius ViT, a 128-path continuation study measures the
omitted finite-window tail but leaves uniform tail control open
(Appendix~\ref{app:vit-adaptive-energy}).

\paragraph{Boundary sensitivity.}
With $\underline T=T_-\wedge J$, $\overline T=T_+\wedge J$ and the
reference's capped exit $\hat\tau=\tau_y\wedge J$, both $\hat\tau$ and
$\tau\wedge J$ lie in $[\underline T,\overline T]$, so
$|\tau\wedge J-\hat\tau|\le\max\{\hat\tau-\underline T,\overline T-\hat\tau\}$. A small position tube can
give a wide exit bracket when the reference runs near the boundary,
and a nonzero tube can fix the integer exit update when the
margins are larger; we therefore report exit brackets, not trajectory
distances (radius--time map in Appendix~\ref{sec:radius-window-map}).

\subsection{Separating finite-window and complete means}
The pathwise bracket bounds the CDF and RMST:
\begin{align*}
 \Pp(T_+\le n)&\le\Pp(\tau\le n)\le\Pp(T_-\le n),\\
 \E(T_-\wedge J)&\le\E(\tau\wedge J)\le\E(T_+\wedge J).
\end{align*}
For a probabilistic tube holding through $J$ with failure probability
at most $\delta$, a bounded-time comparison pays at most
$(J-1)\delta$. A complete mean additionally needs a uniform bound $U_0$
on the expected remaining exit time at every surviving history,
including its velocity and, under reshuffling, remaining sample pool.
Under this extra assumption,
\begin{align*}
 \E\tau\le{}&\E(T_+\wedge J)+(J-1)\delta\\
 &+U_0\min\{1,\Pp(T_+>J)+\delta\}.
\end{align*}
This follows by conditioning the tail
$\E[(\tau-J)\mathbf1_{\{\tau>J\}}]$ on the history at $J$; a small
tube-failure probability alone cannot control it. If calibration is
random, the comparison is conditional on it.
Appendix~\ref{app:probtransfer} gives the iid construction. The network
experiments below therefore report RMST, not complete mean exit times.

\subsection{Trained-pool evidence and enclosure controls}
In a prespecified study, 32 ridge-logistic runs each draw a 2,048-example
pool, with eight runs per MNIST/Fashion-MNIST and $\beta\in\{0.5,0.9\}$
stratum. Each run trains for two reshuffled epochs and is then continued
from its full trained state under iid and reshuffled sampling. The
primary endpoint asks for a sample-curvature bracket at most 20 updates
wide that is at least five updates narrower than the common-curvature
bracket.

\begin{figure}[t]
\centering
\includegraphics[width=\columnwidth]{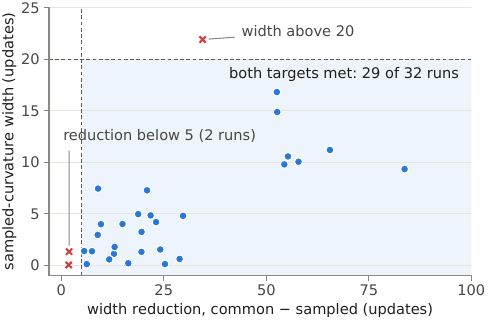}
\caption{Run-level applicability. Width reduction is common minus sampled
width. Shading marks both prespecified targets; red crosses identify
two gain failures and one width failure. All 32 runs are shown without jitter.}
\label{fig:curvature-targets}
\end{figure}

Of 32 runs, 29 succeed (Figure~\ref{fig:curvature-targets}), with
one-sided 95\% lower bound 0.690; recomputing all 6,144 brackets with
rigorous rounding-error bounds changed none (Appendix~\ref{sec:fresh-pool}).
Table~\ref{tab:trained-comparison} separates enclosure width from point
prediction error: both point predictors are accurate, while mean bracket
widths fall from 30.86 updates with the mean Hessian to 5.11 with
sampled curvature.

\begin{table}[t]
\centering
\caption{Fresh-pool comparison: run-averaged bracket widths and paired
capped-time errors (updates) with common and sampled curvature. All
methods match the complete initial state.}
\label{tab:trained-comparison}
\small
\setlength{\tabcolsep}{0pt}
\begin{tabular*}{\columnwidth}{@{\extracolsep{\fill}}lrrrr@{}}
\toprule
 & \multicolumn{2}{c}{Bracket width} & \multicolumn{2}{c}{Capped-time MAE}\\
\cmidrule(lr){2-3}\cmidrule(l){4-5}
Start & Common & Sampled & Common & Sampled\\
\midrule
Trained & 30.86 & 5.11 & 0.544 & 0.094\\
Stationary & 27.73 & 3.79 & 0.263 & 0.122\\
\bottomrule
\end{tabular*}
\end{table}

Features are fixed 17-dimensional pooled images with an intercept ($h=0.2$, $b=64$, ridge 0.1); each run has eight paths per
sampler and initial state, three radii and a 200-update horizon
(Appendix~\ref{sec:fresh-pool}). Two of the three unsuccessful runs
already have narrow common-curvature brackets and miss the improvement
target; the third misses the width target. The stationary start changes
position and velocity together, so its contrast with the trained start
does not isolate the initial velocity.

\begin{figure}[t]
\centering
\includegraphics[width=\columnwidth]{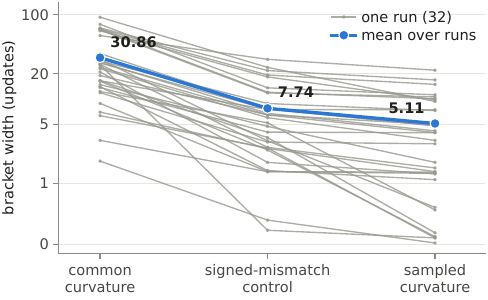}
\caption{Paired enclosure widths for all 32 runs (gray); equal-run means
are blue. The vertical scale is logarithmic above one update and linear
below it. This post-hoc, two-reference control closes 89.8\% of the mean
width gap; lines are not confidence intervals.}
\label{fig:enclosure-control}
\end{figure}

\paragraph{A control for how the mismatch is bounded.}
Around the common-curvature reference $z_n$, the envelope
$F_n=\|y_n-z_n\|+E_n$ is valid by the triangle inequality; it keeps the
propagated mismatch as a vector before taking a norm. Its mean
trained-state width is 7.74 updates
(Figure~\ref{fig:enclosure-control}), closing 89.8\% of the gap between
30.86 and 5.11, so much of the original gap comes from how the error
is enclosed. This post-hoc control still uses sampled curvature and both
reference paths, and its bracket contains the sampled-reference bracket
by construction.

\paragraph{Cost.}
Direct simulation costs 0.005 seconds per path against 0.18 for the
bound. Dense propagation costs $O(J^2d^3)$, confining it to low dimensions;
an explicit short-memory bound gives a median $2.5\times$ speedup
(Appendix~\ref{sec:shortmemory}). The bounds are not faster estimators.

\section{Sampler dependence in trained networks}\label{sec:networks}
Sections~\ref{sec:information} and~\ref{sec:bridge} imply that, with the
one-step law fixed, the sampler alone can change exit times. We tested
this prospectively in four trained networks: an MLP on MNIST, a CNN on
Fashion-MNIST, and ResNet-18 \citep{he2016deep} and a ViT on CIFAR-10
\citep{krizhevsky2009learning}.

\paragraph{Design.} Each network continues from a common full state
$(\theta_0,u_0)$, obtained by 100 iid momentum steps on a fixed pool of
4,096 training examples ($b=64$, $\beta=0.9$, BatchNorm statistics
frozen). From it, 32 paths per sampler run for 1,000 updates; all
samplers share the marginal law over 64-subsets, hence the one-step
gradient law at a fixed parameter. We record exits of $\|\theta_n-\theta_0\|$ from a ball
and of the projections on the top eigenvector $e_{\rm noise}$ of the
per-example gradient covariance and on a random direction. Learning
rates and radii were fixed on a separate pilot seed, and hypotheses and
analyses were recorded before each study
(Appendix~\ref{app:nn-sampler}). The primary endpoint is the RMST of the
ball exit relative to iid. With only 3--4 independent training seeds,
the archived 95\% bootstrap ranges are descriptive, not calibrated
population confidence intervals or significance tests
(Appendix~\ref{app:nn-seedlevel}).

\paragraph{Study 1: the declared comparisons.} Four seeds per network
were continued under iid, RR and persistent sampling with $p=0.25$ and
$p=0.05$ (mean echo counts 4 and 20). We declared three directional
hypotheses: RR exits later than iid, $p=0.05$ exits earlier than iid,
and the RMST decreases from iid to $p=0.25$ to $p=0.05$. Moderate
persistence gave ball-RMST
ratios of 0.61 (MLP), 0.43 (CNN), 0.91 (ResNet-18) and 0.41 (ViT),
relative to iid (Figure~\ref{fig:nn-study1}a); the ResNet-18
interval includes one. Stronger persistence gave ratios of 1.61, 0.86,
1.21 and 0.72, and RR ratios ranged from 0.81 to 1.07. No network showed the predicted monotone ordering: $p=0.05$ had a larger point RMST than $p=0.25$ in all four.
Moderate persistence shifts the whole survival curve to earlier exits in
every network (Figure~\ref{fig:nn-study1}d--g); the random projection
behaves like the ball, and the $e_{\rm noise}$ projection is noisier
(Figure~\ref{fig:nn-study1}b,c; Table~\ref{tab:nn-all}). One ResNet-18
burn-in diverged before any sampler ran, leaving three seeds.

\begin{figure*}[t]
\centering
\includegraphics[width=\textwidth]{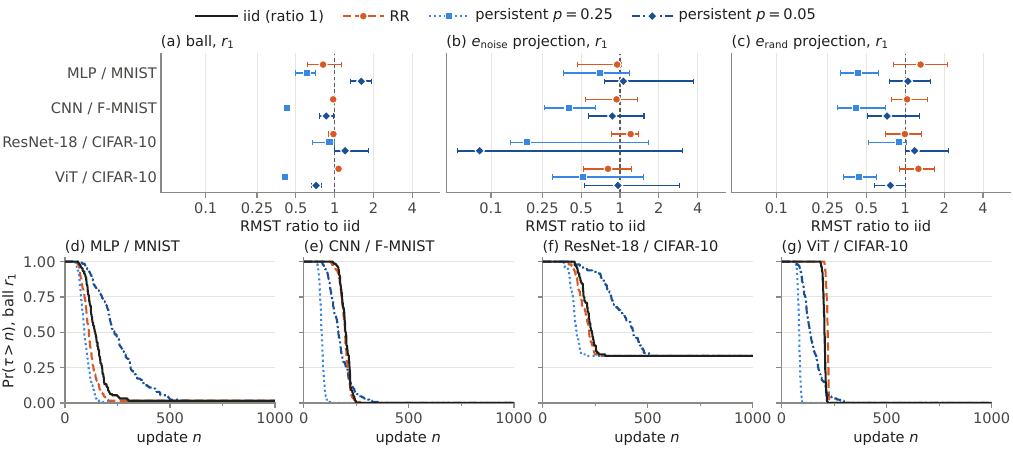}
\caption{Study 1 (declared comparisons). (a--c) RMST ratios to iid at
the first radius with descriptive bootstrap ranges over seeds
and paths; ratios below one mean earlier exits. (d--g) Survival of the
ball exit at $r_1$, pooled over seeds and paths (128 paths per sampler;
96 for ResNet-18, which uses three seeds). All samplers share the
marginal batch law.}
\label{fig:nn-study1}
\end{figure*}

\begin{figure*}[t]
\centering
\includegraphics[width=\textwidth]{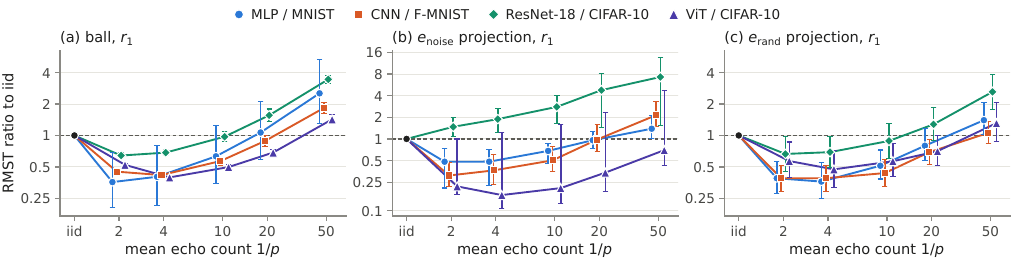}
\caption{Study 2 (new seeds): RMST ratios to iid against the mean echo
count $1/p$ at the first radius for (a) the ball, (b) the $e_{\rm noise}$
projection and (c) a random projection, with descriptive bootstrap
ranges; networks are offset horizontally for legibility. All samplers
share the marginal batch law. ResNet-18 uses three seeds.}
\label{fig:nn-dose}
\end{figure*}

More persistence thus does not always mean faster escape, which a single
persistence level would conceal. These are local displacement events, not basin or generalization metrics.

\paragraph{Study 2: a sampled interior minimum.} The regime
trichotomy of Section~\ref{sec:bridge} suggests a finite-window reading
of this failure. Frequent refreshes average the forcing, as iid sampling
does; rare refreshes approach the frozen-sample regime and leave few
switches within the window. Before a second study on four new seeds per
network, we declared that $p=0.25$ exits earlier than both iid
and $p=0.02$, with $p\in\{1,0.5,0.25,0.1,0.05,0.02\}$ and RR. ResNet-18
again lost one seed to a diverging burn-in. The point ordering appeared in all
four networks (Figure~\ref{fig:nn-dose}a): $p=0.25$ gave ball-RMST ratios
of 0.40--0.69 to iid and 0.16--0.28 to $p=0.02$, with every interval,
and every per-seed ratio, below one. The sampled minimum occurs at mean echo counts 2--4; echo counts
of 50 give ratios of 1.4--3.4 to iid. The random projection also has
an interior minimum.

\paragraph{Study 3: not the switching limit.} The trichotomy reading has
a sharper consequence. If the networks were in the switching regime of
Theorem~\ref{thm:telegraph-bridge}, exit times measured in time $t=nh$
would depend on persistence only through the rate $p/h$, while iid
sampling would average out as $h$ decreases. A third declared study
reran the Study 2 starts of the MLP, CNN and ViT at $h_0/2$ and $h_0/4$,
with $p$ scaled by $h/h_0$ and the window fixed at $1000h_0$
(Figure~\ref{fig:nn-scaling}). Iid exits slowed as predicted: between
$h_0$ and $h_0/4$ the time-unit RMST grew by factors of 1.5--4.8 (all
intervals above one), and no MLP or CNN iid path left the ball at the
smaller steps. Persistent exits slowed too, however: the declared
invariance between $h_0/2$ and $h_0/4$ at matched rate 0.25 failed for
the CNN (RMST ratio 3.20) and the ViT (1.95), and was uninformative for the MLP
because both arms were censored. The fastest rate stayed at 0.25--0.5.
The persistence effects of Studies 1--2 are therefore tied to practical
step sizes; they are not the switching limit. Consistently, probed
minibatch curvatures give $h_0\kappa_B/(1-\beta)$ of 6.9--12.6, far from
the small-step regime (Appendix~\ref{app:nn-sampler}).

\paragraph{Studies 4--7: generality and a limit.} Four further declared
studies tested the interior minimum outside the original setting
(Figure~\ref{fig:nn-robust}; Appendix~\ref{app:nn-studies456}). The point ordering recurred
on new data, for the ViT on CIFAR-100 and the CNN on SVHN (ResNet-18 on
CIFAR-100 failed the pilot and was not run), in all eight cells with pool
sizes 1024 and 16384 and batch sizes 32 and 128, and after a tenfold
longer burn-in; in several cells iid paths never left the ball while
persistent ones did. It did not hold without
heavy-ball momentum. At $\beta=0.5$ and $\beta=0$, with the same raw step,
the ViT's exits slowed monotonically with persistence (iid fastest,
$p=0.02$ slower by factors of 3.1--3.4), and the CNN never reached the
ball. The speed-up thus belongs to momentum SGD at practical steps; the
sampler dependence persists without momentum. Archived paths suggest a role for accumulated momentum
(Appendix~\ref{app:nn-mechanism}). Raw buffer resets removed the speed-up
(Study 8, Appendix~\ref{app:nn-study8}), but random resets harmed at
least as much as resets at refreshes, contrary to prediction.

\begin{figure*}[t]
\centering
\includegraphics[width=\textwidth]{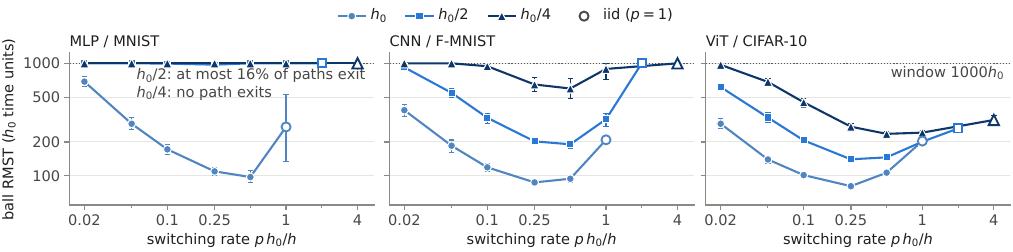}
\caption{Study 3: ball RMST in time ($h_0$ updates) against the
switching rate $p\,h_0/h$ at three learning rates (descriptive bootstrap
ranges; open markers: iid; dotted: the window). Under the
switching limit of Theorem~\ref{thm:telegraph-bridge} the curves would
coincide.}
\label{fig:nn-scaling}
\end{figure*}

\paragraph{Fresh training-seed replication.}
Both networks give 32/32 successes in each persistence comparison
(Appendix~\ref{app:replication-followup}). At calibrated normalized lengths,
refresh-reset RMST exceeds no-reset RMST in 32/32 seeds at both networks,
and random-reset RMST in 32/32 MLP but 0/32 CNN seeds
(Appendix~\ref{app:mechanism-followup}). With 16 new starts per network,
slowing across lengths $0.5,1,2$ times calibration holds in 16/16 seeds
versus no reset but 0/16 versus random reset: the latter ordering reverses
with length. These corrected joint stability/direction tests do not concern
population means or causally decompose raw SGD.

\paragraph{Surrogate errors and their attribution.} One-dimensional surrogates
along $e_{\rm noise}$, built at $\theta_0$ from the pool's projected
gradients and curvature, show the gap. A Gaussian surrogate matched to
the one-step moments assigns one exit law to every sampler; feeding the
pool values through the actual sampler reduced the error in 20 of 24
cells, as declared, and ordered 85\% of clearly separated Study 2
sampler pairs correctly. Neither is quantitatively adequate (RMST errors
above 100\% in every Study 1 cell). These comparisons do not isolate
projection from local approximation. A small-network factorial control
separates these interventions but does not attribute deep-network error
(Appendix~\ref{app:priority-controls}).

\raggedbottom
\section{Related work}
\paragraph{Continuum models of SGD.}
Stochastic modified equations \citep{li2017stochastic}, escape analyses
\citep{zhu2019anisotropic,xie2021diffusion}, heavy-tailed first-exit
comparisons \citep{nguyen2019first}, Hessian-aware SDEs
\citep{li2025hasme} and models of momentum
\citep{liu2021finite,simsekli2020fractional} are calibrated from one-step
noise statistics; by Section~\ref{sec:information}, no calibration from
the pointwise law, Gaussian or heavy-tailed, determines reshuffled exits.

\paragraph{Switching and Markovian sampling.}
Switching models represent subsampling by a Markov index, for SGD
\citep{latz2021analysis} and for momentum \citep{jin2025momentum} with
finite-horizon distance bounds; Markov-chain SGD is analyzed through
mixing times \citep{sun2018markov,even2023stochastic}; data echoing and
minibatch persistency reuse batches for speed
\citep{choi2019faster,fischetti2018faster}; telegraph first-passage
formulas are classical \citep{angelani2014first,kus1993}. We add a
mean-exit bracket in the switched regime and measure how reuse changes
exits.

\paragraph{Reshuffling and identifiability.}
Reshuffling lower bounds \citep{safran2020shuffling} and variance analyses
\citep{mishchenko2020reshuffling} show that order matters for
optimization error; our target is an exit event. Unlabeled recovery links
our information results to homomorphic sensing and phase retrieval
\citep{peng2021homomorphic,alaifari2017stable}.

\paragraph{Uncertainty-guided acquisition.}
OUGS \citep{li2026ougs} propagates Gaussian-primitive parameter covariance
through rendering Jacobians to guide object-aware view selection.
This connects local uncertainty propagation to data acquisition;
our exit comparisons additionally track sampler history and boundary
crossings. We do not evaluate active view selection here.

\section{Conclusion}
Pointwise laws do not identify reshuffled exits. The one-dimensional
persistent regime admits a mean-exit bracket; certified brackets meet
targets in 29/32 logistic pools. Neural direction findings replicate under
the specified finite-path protocols, while reset rankings depend on step
length and window. Original-radius ViT mean guarantees and a general
causal interpretation remain open.

\paragraph{Limitations.}
The bridge assumes sign-definite one-dimensional velocities, absent from
our networks. Only fresh-pool exit brackets are rounding-certified;
regional smoothness does not cover general ReLU networks. Most
architectures retain 3--4 seeds with descriptive intervals. Pilot radii,
frozen ResNet-18 BatchNorm and reset/step-length confounding restrict
generalization. An earlier telegraph diagnostic failed prospective validation (Appendix~\ref{app:extended-validation}).

\clearpage
\section*{AI Use Statement}
Generative AI tools assisted with mathematical derivations, experimental
design, implementation and audits, and manuscript revision. The authors are responsible
for verifying the arguments, results and final text.

\bibliographystyle{plainnat}
\bibliography{references}
\section*{Checklist}
\begin{enumerate}

  \item For all models and algorithms presented, check if you include:
  \begin{enumerate}
    \item A clear description of the mathematical setting, assumptions, algorithm, and/or model. [Yes. Sections 2--6 and the mathematical appendices.]
    \item An analysis of the properties and complexity (time, space, sample size) of any algorithm. [No. Dense propagation time complexity is analyzed, but a complete time/space/sample-complexity treatment of every algorithm is not supplied.]
    \item (Optional) Anonymized source code, with specification of all dependencies, including external libraries. [No. The manuscript source, figure assets and executable exact-arithmetic head/cone/two-stage certificate audits with rational inputs, and the 64-path finite-tail statistical replay are provided. A single package reproducing every reported experiment is not supplied.]
  \end{enumerate}

  \item For any theoretical claim, check if you include:
  \begin{enumerate}
    \item Statements of the full set of assumptions of all theoretical results. [Yes. Hypotheses accompany the theorem statements.]
    \item Complete proofs of all theoretical results. [Yes. Complete proofs are given in the appendices.]
    \item Clear explanations of any assumptions. [Yes. Initialization, clocks, conditional moments and sampling restrictions are explicit.]
  \end{enumerate}

  \item For all figures and tables that present empirical results, check if you include:
  \begin{enumerate}
    \item The code, data, and instructions needed to reproduce the main experimental results (either in the supplemental material or as a URL). [No. The source package provides displayed results and exact certificate inputs with executable audits. Complete saved trajectories, training data and full network checkpoints are omitted; the package alone does not reproduce all experiments.]
    \item All the training details (e.g., data splits, hyperparameters, how they were chosen). [Yes. The prospective-design appendices record the settings, how learning rates and radii were chosen on separate pilot seeds, and all amendments.]
    \item A clear definition of the specific measure or statistics and error bars (e.g., with respect to the random seed after running experiments multiple times). [Yes. Run and seed resampling (hierarchical bootstrap), RMST, censoring and sensitivity intervals are distinguished.]
    \item A description of the computing infrastructure used. (e.g., type of GPUs, internal cluster, or cloud provider). [Yes. The historical neural experiments used one NVIDIA GeForce RTX 3060 Laptop GPU (6 GB; CUDA 12.8, PyTorch 2.10). The eight new network studies used NVIDIA GeForce RTX 3090 GPUs (24 GB) of a shared server with PyTorch 2.10.0+cu128; their continuations recorded 6.3, 10.9, 17.0, 4.7, 19.5, 12.7, 1.6 and 1.0 device-hours, the Study 4 pilot 0.4 and the Study 7 start states 0.3 device-hours, the precision check 2.5 device-hours, and training the 51 networks took 3.1 GPU-hours. The ridge-logistic study ran sequentially on CPU (18.98 minutes of run computation), its certified recomputation took 69 CPU-minutes, and the telegraph-bridge simulations used 20 CPU threads. The follow-up 64-training-run replication reserved one RTX 3090 for 1.016 GPU wall-hours; smooth-network validation, its 401.7-second longer-window follow-up, and the 3.21-second integer interval check ran on CPU. The completed six-arm mechanism phase used less than 78 minutes between calibration completion and its own completion, an upper bound on reserved wall time including launch delay; the completed independent robustness follow-up reserved a further 0.642 GPU wall-hours. The initial replication, calibration, mechanism and robustness GPU stages fit within one 3.069-hour allocation window (including idle gaps), below their frozen four-hour cap. Later ViT continuations are accounted for separately in Appendices~\ref{app:vit-fresh-tail} and~\ref{app:vit-defensive-tail}. Reserved wall time is distinct from measured GPU utilization.]
  \end{enumerate}

  \item If you are using existing assets (e.g., code, data, models) or curating/releasing new assets, check if you include:
  \begin{enumerate}
    \item Citations of the creator If your work uses existing assets. [Yes. MNIST \citep{lecun1998}, Fashion-MNIST \citep{xiao2017}, CIFAR-10 and CIFAR-100 \citep{krizhevsky2009learning} and SVHN \citep{netzer2011reading} are cited, as are the ResNet \citep{he2016deep} and ViT \citep{dosovitskiy2021image} architectures.]
    \item The license information of the assets, if applicable. [Yes. MNIST is distributed under CC BY-SA 3.0 and Fashion-MNIST under the MIT license. The CIFAR web page states no license for CIFAR-10 or CIFAR-100; the UCI Machine Learning Repository lists CIFAR-10 under CC BY 4.0. The SVHN web page marks the cropped-digit files for non-commercial use only; this work is non-commercial research.]
    \item New assets either in the supplemental material or as a URL, if applicable. [No. The manuscript source, displayed tables, figure assets, exact head/cone/two-stage certificate materials and the 64-path finite-tail statistical replay are provided. Full trajectories, large-network checkpoints and complete end-to-end retraining are not included; no public release is claimed.]
    \item Information about consent from data providers/curators. [Not Applicable. Only public benchmark datasets are used.]
    \item Discussion of sensible content if applicable, e.g., personally identifiable information or offensive content. [Not Applicable. No new human-subject or sensitive-content collection is used.]
  \end{enumerate}

  \item If you used crowdsourcing or conducted research with human subjects, check if you include:
  \begin{enumerate}
    \item The full text of instructions given to participants and screenshots. [Not Applicable. No crowdsourcing or human-subject study was conducted.]
    \item Descriptions of potential participant risks, with links to Institutional Review Board (IRB) approvals if applicable. [Not Applicable. No crowdsourcing or human-subject study was conducted.]
    \item The estimated hourly wage paid to participants and the total amount spent on participant compensation. [Not Applicable. No crowdsourcing or human-subject study was conducted.]
  \end{enumerate}

\end{enumerate}

\clearpage
\flushbottom
\onecolumn
\appendix
\renewcommand{\thesection}{\AlphAlph{\value{section}}}
\paragraph{Guide to the appendices.} The appendices follow the main
text. Appendix~\ref{app:sampling-information} proves the information
results of Section~\ref{sec:information}.
Appendices~\ref{app:obstructions} and~\ref{app:telegraph-bridge} contain
the obstructions and the persistent-sampling bridge of
Section~\ref{sec:bridge}.
Appendices~\ref{sec:heterogeneous}--\ref{app:reshuffling} support the
computable comparisons of Section~\ref{sec:curvature}; among them,
Appendix~\ref{sec:fresh-pool} gives the fresh-pool protocol and its
results, including a recomputation of every bracket with rigorous
rounding-error bounds.
Appendix~\ref{app:nn-sampler} documents the network studies of
Section~\ref{sec:networks}: training, pilot, declarations and full
results of Studies 1--7, the post-hoc mechanism analysis and the
intervention that tests it (Study 8), a float64 precision check, and
censoring and seed-level checks.
Appendix~\ref{app:replication-followup} reports the fresh 32-seed replication;
Appendix~\ref{app:mechanism-followup} reports equal-length reset controls;
Appendix~\ref{app:fullparam-followup} gives whole-parameter smooth-network
bounds, interval-verified tail conditions, and the failed feasibility
criterion and horizon follow-up. Appendices~\ref{app:large-head-tail}--\ref{app:two-stage-head}
give the classification-head certificates and their scope;
Appendix~\ref{app:vit-adaptive-energy} reports the ViT continuation
study and states what is still needed for complete means. Appendices~\ref{app:extended-validation}
and~\ref{sec:bounded-precision} report supplementary computational
benchmarks and an earlier neural diagnostic that the main results do not
rely on.

\section{Sampling information and exit laws}\label{app:sampling-information}
\subsection{Pointwise gradient laws identify iid dynamics but not reshuffling}
\label{sec:pointwise-law-rr}
Let $q=1/4$, $d(t)=t^3/[4(1+t^2)]$, and define two component-gradient
pools, indexed by $s\in\{-1,+1\}$, by
\[
 g_s^A(x)=x-1+s d(x-q),\qquad
 g_s^B(x)=x-1+s|d(x-q)|.
\]
Define each component loss as $\ell_s^M(x)=\int_0^x g_s^M(z)\,dz$,
where $M\in\{A,B\}$. Both pools consist of globally strongly convex
$C^3$ losses, with strong-convexity constant $23/32$ and
gradient-Lipschitz constant $41/32$. Their mean objective is exactly
$f(x)=x^2/2-x$. At every fixed $x$, the entire fresh-uniform gradient
law is identical between the two pools:
\[
 \tfrac12\delta_{g_{-1}^A(x)}+\tfrac12\delta_{g_{+1}^A(x)}
 =\tfrac12\delta_{g_{-1}^B(x)}+\tfrac12\delta_{g_{+1}^B(x)}.
\]
Consequently iid uniform sampling produces the same full momentum-state
path law in both pools, for any common deterministic initialization and
common algorithm parameters. This conclusion follows from equality of
the transition kernels and does not require a small-step limit.

In contrast, consider independent uniform permutations of the two labels
at successive epochs, initialize $x_0=u_0=0$, and set
\[
 u_n=\tfrac12u_{n-1}-g_{s_n}^M(x_{n-1}),\qquad
 x_n=x_{n-1}+\tfrac12u_n,\qquad
 \tau=\inf\{n\geq1:|x_n|\geq R\}.
\]
For every $1\leq R\leq513/512$, the first-epoch exit probabilities are
$\Pr_A(\tau\leq2)=1/2$ and $\Pr_B(\tau\leq2)=0$.
Thus any deterministic scalar predictor based only on the complete
family of pointwise fresh-uniform gradient laws, the common mean
objective, algorithm, initialization, radius and specified reshuffling
rule has worst-case absolute error at least $1/4$ for this CDF value.
The restriction excludes access to the cross-position identity of the
component functions or to their history-conditioned remaining-pool laws.

\paragraph{Proof.}
Writing $r=t^2\geq0$ gives
$d'(t)=r(3+r)/[4(1+r)^2]$. Since
$9(1+r)^2-8r(3+r)=(r-3)^2$, we have $0\leq d'(t)\leq9/32$.
The derivative of $|d|$ has the same absolute bound, including its
continuous value at zero. Both $d$ and $|d|$ are $C^2$; their first
two derivatives vanish at zero. Hence all component Hessians lie in
$[23/32,41/32]$, proving the regularity and convexity claims (and
boundedness below). The unordered pair of gradient values is unchanged
by replacing $d$ by $|d|$, proving equality of the iid kernels.

For reshuffling the first two labels are either $(-1,+1)$ or
$(+1,-1)$, each with probability $1/2$. Direct rational substitution
gives the following integer-time positions:
\[
\begin{array}{c|c|c|c}
 M&(s_1,s_2)&x_1&x_2\\\hline
 A&(-1,+1)&271/544&1362255009/1367228672\\
 A&(+1,-1)&273/544&1374685153/1369596160\\
 B&(-1,+1)&273/544&1369542447/1369596160\\
 B&(+1,-1)&271/544&1367175759/1367228672
\end{array}
\]
All first positions and both second positions in B are below $1$;
the second A position in the first row is below $1$, while that in
the second row exceeds $513/512$. All are positive. This proves the
exit counts, including boundary equality. The predictor returns the
same value on both pools, so its maximum error is at least half their
CDF separation.

Here the omitted information is the association of sample labels across
positions, not higher one-point gradient moments: the full pointwise
laws already agree. Under reshuffling the remaining label depends on
history, so a fresh-uniform kernel does not specify the next update.
This is an information counterexample for a finite-window event, not
a claim of differing asymptotic convergence or a classification of
complete exit means. It concerns a ball around the initialization,
not a validated basin boundary. The construction and radii were chosen
exploratorily. Exact two- and four-step enumerations are retained;
at step four both pools have exited for all tested paths and radii,
so the separation does not persist at every observation horizon.

The example also limits a local affine reference. For $x<q$,
$g_s^B(x)=g_{-s}^A(x)$, so the complete local component information at
the initial point agrees up to a fixed relabeling. In particular, the
joint gradient--Hessian pairs are identical as empirical distributions:
in A they are $(-1-s/272,\,1+49s/1156)$, with the signs reversed in B.
Thus the corresponding local affine references also have identical
reshuffling laws. Retaining local sample curvature alone does not
resolve this example; a valid nonlinear residual bound must account
for the discrepancy. Epoch-level or augmented-state descriptions of
reshuffling already occur in the literature, including the perturbed
algorithm studied in \cite{emmanouilidis2026shuffle}. Our example concerns an
explicit finite-step information restriction, not a new Markov-state
augmentation principle.

\subsection{Matching the remaining pools preserves the sampler}
\label{sec:remaining-matching}
The distinction above can be expressed as a standard adapted coupling.
Fix two pools of $N$ gradients $g_i,\widetilde g_j:\mathbb R^d\to
\mathbb R^d$ and common parameters $h>0$, $0\leq\beta<1$.
Assume every $g_i$ is globally $L$-Lipschitz. Use batch size one,
matched deterministic full initial states, and uniform random
reshuffling with synchronized epoch boundaries in both pools.
At step $n$, let $A_n,B_n$ be the two remaining index sets before
the draw. Relative to the joint past, choose a measurable bijection
$\sigma_n:A_n\to B_n$ before drawing $I_n$ uniformly from $A_n$,
and put $J_n=\sigma_n(I_n)$. Delete both selected indices; at each
epoch boundary reset both sets to all $N$ indices.

This generates the prescribed reshuffling law in each marginal.
Indeed, conditional on the joint past, each remaining index in either
pool has probability $1/|A_n|$. Conditioning further only on that
pool's own label-augmented history preserves this probability, because
its remaining set is determined by that history; positions alone need
not reveal the used labels. The same argument over a complete
epoch gives a uniform permutation independent of its prior history.
The bijection may change at each step; a single fixed cross-pool
permutation is sufficient but is not necessary.

Run momentum updates with $g_{I_n}$ for $(x,u)$ and
$\widetilde g_{J_n}$ for $(y,v)$. Define the predictable residual bound
\[
 \delta_n=\max_{i\in A_n}
 \|g_i(y_{n-1})-\widetilde g_{\sigma_n(i)}(y_{n-1})\|_2,
 \quad Q_n=\beta Q_{n-1}+LE_{n-1}+\delta_n,
 \quad E_n=E_{n-1}+hQ_n,
\]
with $E_0=Q_0=0$. Splitting the gradient difference at $y_{n-1}$
and applying induction proves
$\|u_n-v_n\|_2\leq Q_n$ and $\|x_n-y_n\|_2\leq E_n$.
For a common open ball containing $x_0=y_0$ and integer exit time
$\tau=\inf\{n\geq1:\|x_n\|_2\geq R\}$, the usual times
$T_- =\inf\{n\geq1:\|y_n\|_2+E_n\geq R\}$ and
$T_+ =\inf\{n\geq1:\|y_n\|_2-E_n\geq R\}$ satisfy
$T_-\leq\tau\leq T_+$. This gives finite-window CDF and restricted-mean
comparisons, without implying a complete-mean bound.

For iid sampling one instead keeps both full index sets at every
step, so location-dependent bijections also preserve the iid marginals.
For reshuffling, resetting the matching domain to all indices at every
step is generally invalid. On the pair in
Appendix~\ref{sec:pointwise-law-rr}, zero-residual pointwise matching
first reverses the labels and then preserves them. It therefore maps
the two legitimate A permutations to repeated B labels $(+1,+1)$ and
$(-1,-1)$, not to B permutations. It returns exit probability $1/2$
for a process whose intended B reshuffling probability is zero.
Respecting the remaining pools restores the correct marginal law,
but the norm-based comparison gives only $[0,1]$ for A's two-step
CDF at $R=513/512$ in this example. Correct sampling does not by
itself guarantee an informative bound.

Choosing a bijection minimizing $\delta_n$ is a finite bottleneck
assignment problem, not a claim of optimality for future exit error.
It requires gradient values at the reference state or validated
surrogate bounds for those values. This is an application of the
classical conditional-marginal coupling principle; causal transport
and its dynamic programming formulations predate this work
\cite{backhoff2016causal,moulos2020bicausal}. It is not counted as a
new coupling theorem, and no computational advantage is inferred.

\subsection{Analytic identification does not imply stable reshuffled exits}
\label{sec:analytic-identification}
Analyticity changes the exact identification question. Suppose two
equally sized finite pools of real-analytic vector fields $g_i,h_j$ on
a connected open set $U\subseteq\mathbb R^d$ have the same empirical
distribution of values at every $x\in U$. Then there is one fixed
permutation $\pi$ with $g_i=h_{\pi(i)}$ throughout $U$.
To see this, fix $g_1$. The analytic scalar function
$\prod_{j=1}^N\|g_1-h_j\|_2^2$ vanishes throughout $U$.
Real-analytic functions on a connected open set have no zero divisors:
if a product vanishes identically and one factor is nonzero somewhere,
the other vanishes on an open neighborhood and hence everywhere by
the identity theorem. Iteration shows that one factor above vanishes
identically. Cancel the matched components, including their
multiplicities, and repeat. A fixed relabeling preserves uniform
reshuffling, so the corresponding exit laws from any ball contained
in $U$ agree for matched full initial states whose position starts
inside the ball. This is an exact
information statement, not a stable reconstruction algorithm.
In particular, the exact label-switching example above cannot be
realized by two distinct analytic finite pools such as ridge-logistic
gradient pools on all of $\mathbb R^d$.

Nevertheless, pointwise-law closeness provides no uniform modulus of
continuity for the reshuffled exit CDF, even within globally strongly
convex analytic two-component losses with fixed smoothness constants.
Keep pool A from Appendix~\ref{sec:pointwise-law-rr} and replace B by
\[
 g_s^{B,\varepsilon}(x)=x-1+s\sqrt{d(x-1/4)^2+\varepsilon^2},
 \qquad 0<\varepsilon\leq1/2048.
\]
Define losses by integration from zero. They are real analytic and
their Hessians lie in $[23/32,41/32]$, since the derivative of the
square-root term has magnitude at most $|d'|\leq9/32$.
The mean objective remains $x^2/2-x$. Let $K_M(x)$ denote the
fresh-uniform gradient law at $x$. For the usual $W_\infty$ transport
distance on scalar gradient values,
\[
 \sup_{x\in\mathbb R}W_\infty(K_A(x),K_{B,\varepsilon}(x))
 =\varepsilon.
\]
Indeed, sorting the two atoms couples them at distance
$\sqrt{d^2+\varepsilon^2}-|d|\leq\varepsilon$, with equality at $x=1/4$.
This concerns transport distance, not total variation.

Use the same $h=\beta=1/2$, $x_0=u_0=0$, integer clock and
$R=513/512$ as before. The A reshuffling CDF at time two is $1/2$.
For B and $B,\varepsilon$, the same-label gradient discrepancy is at
most $\varepsilon$ globally. The elementary Lipschitz error recurrence
with $L=41/32$ yields
\[
 E_1=\varepsilon/2,\qquad E_2=201\varepsilon/128.
\]
Both second-step B positions have absolute value below $1$, and its
first-step positions are at most $273/544$. Since
$201/(128\cdot2048)<1/512$, neither step of either perturbed B
permutation reaches $R$. Consequently
\[
 \Pr_A(\tau\leq2)=1/2,\qquad
 \Pr_{B,\varepsilon}(\tau\leq2)=0
 \quad(0<\varepsilon\leq1/2048).
\]
Thus no bound on their CDF difference by
$\omega(\sup_x W_\infty(K_A(x),K_{B,\varepsilon}(x)))$, with
$\omega(r)\to0$ as $r\downarrow0$, holds uniformly over this analytic
class. The radius, initialization, momentum, step size, pool size,
strong convexity and gradient smoothness remain fixed.

\paragraph{The information obstruction persists away from the boundary.}
The preceding discrepancy is not caused solely by placing the boundary
on a path atom. Write $R_0=513/512$ and $\gamma=1/1024$.
For every $0<\varepsilon\leq1/2048$ and every radius
$R\in[R_0-\gamma,R_0+\gamma]$, the same two CDFs remain $1/2$ and zero.
Indeed, the larger A second position exceeds $R_0$ by
$4828001/2739192320>\gamma$, while its other second position is below
$1<R_0-\gamma$. Both first positions are below $R_0-\gamma$.
The perturbed B positions at time two are bounded above in absolute
value by $1+201/262144<R_0-\gamma$; their first positions are smaller
than $273/544+1/4096<R_0-\gamma$. These inequalities follow from the
same-label comparison to B and hold throughout the epsilon interval.
At the central radius $R_0$, every first- and second-step position in
both pools is at distance greater than $\gamma$ from the boundary.
Thus the failure of a uniform CDF modulus for the pointwise-law metric
persists with a fixed positive boundary margin in this example.
It is distinct from the boundary-atom obstruction that can occur even
when component labels are matched consistently.

\paragraph{A quantitative neighborhood of algorithm parameters.}
The same information separation holds throughout a parameter box,
not only at a single learning rate and momentum. Let
\[
 \delta=1/16384,\qquad \rho=1/4096,\qquad R_0=513/512.
\]
For every common choice of parameters satisfying
\[
 |h-1/2|,\ |\beta-1/2|,\ |x_0|,\ |u_0|\leq\delta,
 \qquad |R-R_0|\leq\rho,
\]
and every $0<\varepsilon\leq1/2048$, the A and analytic
$B,\varepsilon$ two-step RR exit probabilities remain $1/2$ and zero.
The full initial states agree between the pools, and the exit interval
remains centered at zero; we perturb the initial position, not its center.

To verify this, compare each path with the same-pool, same-label path
at $h=\beta=1/2$, $x_0=u_0=0$. The central first and second velocities
have magnitude less than two, uniformly in epsilon. For A and B this
follows directly from their rational positions; for the perturbed B
velocities the additional errors are at most $\varepsilon$ and
$137\varepsilon/64$, respectively. The central positions have distance
strictly greater than $m=311/262144$ from the boundary $R_0$:
the B second-step lower margin is $1/512-201/262144=m$,
and the explicit A and first-step margins exceed it.

Every component has Lipschitz constant $41/32<3/2$, while
$h,\beta<3/4$ in the box. For velocity and position errors $Q_n,E_n$
relative to the central path, start from $Q_0,E_0\leq\delta$.
Using the central zero initial velocity gives
\[
 Q_1\leq\tfrac94\delta,\quad E_1\leq\tfrac{75}{16}\delta,
 \quad Q_2\leq\tfrac{343}{32}\delta,\quad
 E_2\leq\tfrac{1885}{128}\delta.
\]
Indeed, $Q_n\leq\tfrac34Q_{n-1}+\tfrac32E_{n-1}
+\delta|\bar u_{n-1}|$ and
$E_n\leq E_{n-1}+\tfrac34Q_n+\delta|\bar u_n|$,
where bars denote central velocities. Since
\[
 m-\rho-\tfrac{1885}{128}\delta=\frac{91}{2097152}>0,
\]
no first- or second-step exit classification can change anywhere in
the box. This is a conservative explicit neighborhood, not a broad
hyperparameter-robustness claim or a new perturbation theorem.

The identification argument is a direct application of the analytic
identity theorem. The stability obstruction concerns the information
metric for fresh-uniform oracles, which does not preserve cross-position
sample identities. It does not contradict perturbation estimates that
control matched component functions, nor prove instability for a fixed
parametric logistic family. No common quantitative analytic-extension
bound is imposed on this family, and no complete-mean conclusion is
drawn from the displayed CDF separation.

\subsection{A sufficient condition for consistent approximate labels}
\label{sec:separated-matching}
Let $D$ be a connected subset of Euclidean space and let
$g_i,h_i:D\to\mathbb R^d$, $1\leq i\leq N$, be continuous.
Suppose $\inf_{x\in D}\min_{i\ne j}\|g_i(x)-g_j(x)\|>2\epsilon$
and, at every $x$, a permutation matches the two pools with maximum
distance at most $\epsilon$. Then one fixed permutation achieves this
bound throughout $D$.
Indeed, each $h_j(x)$ has exactly one possible matching $g_i(x)$:
two candidates would violate separation by the triangle inequality.
At any point, all incorrect candidate distances exceed $\epsilon$
strictly. Continuity and finiteness preserve these strict inequalities
on a relative neighborhood. Existence of a matching at every point
then forces the same assignment on that neighborhood. The assignment
is locally constant and hence constant on connected $D$.
For equal empirical masses, the matching hypothesis is equivalent to
pointwise $W_\infty\leq\epsilon$: a feasible transport matrix on the
allowed edges has a perfect matching (the finite matching theorem).
This is a standard separation argument, not a new matching principle.

For a ball of radius $r$ about $a$, a computable sufficient separation
test is $\Delta_a-2Lr>2\epsilon$, where
$\Delta_a=\min_{i\ne j}\|g_i(a)-g_j(a)\|$ and each component is
$L$-Lipschitz on the ball. This test alone does not establish the
uniform pointwise transport bound; that needs a separate analytic or
validated numerical argument. Observed finite-grid distances alone
are insufficient without interpolation control.

Assume additionally that each $g_i$ is $L$-Lipschitz on $D$.
Couple uniform reshuffling by this fixed permutation, with identical
epoch boundaries, batch size one, full initial state, $h>0$, and
$0\leq\beta<1$. For $u_n=\beta u_{n-1}-g_{I_n}(x_{n-1})$,
$x_n=x_{n-1}+hu_n$ and the corresponding reference $(y,v)$,
put $Q_0=E_0=0$ and
\[
 Q_n=\beta Q_{n-1}+LE_{n-1}+\epsilon,\qquad
 E_n=E_{n-1}+hQ_n.
\]
Assume the actual exit ball is contained in $D$, its initial position
is inside, and every reference input $y_0,\ldots,y_{J-1}$ belongs to
$D$ for every coupled sequence. Then $\|x_n-y_n\|\leq E_n$ through
the actual exit-causing update, up to $J$. The proof splits the
gradient difference at $y_{n-1}$ and uses induction; it requires the
stated Lipschitz inequality between these two inputs in $D$.
Let $T_-$ and $T_+$ be the first integer times up to $J$ with
$\|y_n\|+E_n\geq R$ and $\|y_n\|-E_n\geq R$, respectively,
using infinity if no such time exists. For $n\leq J$,
\[
 \Pr(T_+\leq n)\leq\Pr(\tau_x\leq n)
 \leq\Pr(T_-\leq n).
\]
The same bracket contains the reference exit CDF. Their difference
is therefore at most $\Pr(T_-\leq n<T_+)$. This separates label
consistency, propagated position error, and boundary ambiguity.
Summing survival brackets gives restricted means, not complete MFPT.
No claim is made that the minimum separation remains useful for large
data pools or that this norm recurrence exploits matrix cancellations.

For a finite exact example, take $g_1(x)=x-1$, $g_2(x)=2x+2$,
$h_1=g_1+\epsilon$, $h_2=g_2-\epsilon$, $D=[-3/5,3/5]$,
$h=1/4$, $\beta=1/2$, zero initial state, and $J=2$.
The center separation test gives $3/5>2\epsilon$.
At $\epsilon=1/100$ and $R=3/8$, exact enumeration yields the
nontrivial CDF bracket $[1/2,1/2]$. At $R=1/2$, the actual and
reference CDFs are $1/2$ and zero: the actual path has an atom on
the boundary. Label stability therefore does not remove boundary
sensitivity.

This sufficient condition does not explain the existing training-pool
results. A post-hoc individual-example gradient audit of the 32 saved
training pools, both initializations and three radii, found
0 of 192 nested settings passing the center test even at
$\epsilon=0$. The implied radius limits were
0.00305--0.01872, below the smallest tested radius
0.05771. These are floating-point checks of a conservative
sufficient test, not impossibility results or validated numerical
certificates. This individual-example audit does not establish a
batch-size-64 coupling theorem. We do not count this standard separation
argument as a primary contribution.

\subsection{An event-aware matching oracle as a diagnostic}
\label{sec:remaining-oracle}
A small-pool calculation tests whether rematching itself improves an
exit comparison. With full label histories $p,q$ of equal length $k$,
let $A(p),B(q)$ be the remaining indices. For a fixed horizon $J$ and
radius $R$, define
\[
 V_J(p,q)=|\mathbf1_{\{\tau_x(p)\leq J\}}
                -\mathbf1_{\{\tau_y(q)\leq J\}}|,
\]
and, recursively for $k<J\leq N$,
\[
 V_k(p,q)=\min_{\sigma:A(p)\simeq B(q)}
 \frac1{N-k}\sum_{i\in A(p)}V_{k+1}(pi,q\sigma(i)).
\]
Histories determine full states and whether an earlier exit occurred;
the recursion does not replace exit by final-position exceedance.
Each bijection is selected before the next uniform draw. Thus both
marginals retain uniform reshuffling. Conversely, a conditional coupling
with uniform remaining-pool marginals is a scaled doubly stochastic
matrix; minimizing a linear continuation cost admits a permutation
extreme point. Backward induction therefore gives the smallest terminal
exit-indicator disagreement within this joint-history conditional-marginal
coupling class. It need not equal the unrestricted Bernoulli coupling
minimum $|F_x(J)-F_y(J)|$, although it does in the cases below.
This is a finite adapted-transport dynamic program, not a new optimal
transport method; such conditional coupling formulations are classical
\cite{moulos2020bicausal,backhoff2016causal}.
Its policy depends on the complete processes, $R$,
and $J$; it is an oracle diagnostic rather than a cheap surrogate.

We compare it with the original fixed labels and a predictable greedy
bijection minimizing the maximum gradient discrepancy at the current
reference position, breaking ties lexicographically. The diagnostic
uses the four-sample heterogeneous-curvature pools of
Appendix~\ref{sec:full-covariance-pairing}, with common quartic coefficient
$\nu=1/4$, $h=1/4$, $\beta=1/2$, and zero full initial state. We compare
nonlinear A against its own affine reference and against nonlinear B,
for $J\in\{2,4\}$ and $R\in\{5/8,3/4,7/8,1\}$.
All first-epoch permutations and rational states are retained. The
endpoint is disagreement of exit indicators at the specified horizon,
not complete exit-time agreement, a path-law metric, or MFPT.

In all 8 nonlinear-A versus affine-A settings, fixed labels
already attain the lower bound $|F_x(J)-F_y(J)|$: rematching cannot
improve this endpoint. For the nonlinear A/B comparison, the oracle
improves on fixed labels in 5 of 8 settings. At $J=2,R=3/4$,
fixed and greedy labels give disagreement $1/6$, whereas the oracle
gives zero; both exit CDFs are equal. Across all 16 settings the
greedy policy has exactly the same terminal disagreement as fixed
labels. Thus a local gradient-matching objective does not automatically
capture the available event-coupling improvement. These finite examples
support neither a universal optimality claim for fixed labels nor a
scalable improvement from the oracle.

The implementation retains all policy decisions and mapped paths.
At $J=4$ the combined cache for all three policies contains 1443
history-pair/mode entries, compared with 24 possible marginal paths.
Direct marginal enumeration already gives the exact CDFs in these
small pools; the dynamic program is used to diagnose coupling slack,
not to accelerate that computation. No additional neural training or
GPU computation is involved.

\subsection{Fixed polynomial complexity permits stable matching}
\label{sec:polynomial-matching}
There is a classical sufficient condition that allows component
crossings. Let $g_1,g_2,h_1,h_2$ be scalar polynomials of degree at
most $d$ on a compact interval $D$ of positive length. Suppose their
equally weighted two-atom gradient laws satisfy
$\sup_{x\in D}W_\infty(K_g(x),K_h(x))\leq\epsilon$.
Then some fixed permutation $\pi$ satisfies
\[
 \max_{i=1,2}\sup_{x\in D}|g_i(x)-h_{\pi(i)}(x)|
 \leq T_d(3)\epsilon,
\]
where $T_d$ is the Chebyshev polynomial of the first kind.
No lower bound on component separation is required.

For a proof, let $S_\pi$ be the closed set where permutation $\pi$
matches both atoms to within $\epsilon$. The two sets cover $D$,
so at least one has Lebesgue measure at least $|D|/2$.
On that same set both polynomial differences are bounded by
$\epsilon$. Affinely rescale $D$ to $[-1,1]$ and apply the classical
Remez inequality to each difference: a subset of measure at least one
gives the factor $T_d(3)$. See Brudnyi and Yomdin,
Theorem 1.1~\cite{brudnyi2013norming}. This is a direct application of
Remez, not a new stability theorem. Writing the pair as its center
plus/minus a difference also connects the problem to real phase
retrieval, whose finite-dimensional stability is established in
\cite{bandeira2014saving}.

For a finite observation variant, take any $2d+1$ distinct nodes in
$D$. Pointwise matching error at most $\epsilon$ at those nodes
alone implies a fixed permutation with uniform error at most
$\Lambda\epsilon$, where
\[
 \Lambda=\max_{|S|=d+1}\sup_{x\in D}
       \sum_{j\in S}|\ell_{j,S}(x)|.
\]
Here $\ell_{j,S}$ are the Lagrange basis polynomials for that node
subset. One of the two permutations works on at least $d+1$ nodes;
interpolation bounds both differences for this permutation. The
constant is finite and calculable from the nodes but can be large.
For $d=1$ and endpoints plus midpoint it equals three. This statement
requires the polynomial degree restriction; it does not turn arbitrary
finite gradient probes into whole-region validation.

With the matched-state RR coupling and domain conditions from
Section~\ref{sec:separated-matching}, use
$\delta=T_d(3)\epsilon$ (or $\Lambda\epsilon$) in
$Q_n=\beta Q_{n-1}+LE_{n-1}+\delta$,
$E_n=E_{n-1}+hQ_n$. Each $g_i$ must be $L$-Lipschitz on $D$,
the actual exit interval must lie in $D$, and all reference inputs
through the horizon must lie in $D$. The previous stopped argument
then yields CDF and RMST brackets. Boundary sensitivity is still
required to turn small position error into a small CDF gap.
The degree dependence $T_d(3)$ grows exponentially; this is a
sufficient constant, not a claim of optimal dependence. This result
explains a regularity boundary of the analytic example, without
establishing a new general-purpose SGD comparison method.

\paragraph{The information obstruction is not specific to momentum.}
The same pools also separate first-exit probabilities with $\beta=0$.
Keep $h=1/2$, zero initial position, and change the radius to $385/512$.
The two second-step A positions are
$1021704485/1367228672$ and $1031027293/1369596160$;
the B positions are $1025884587/1369596160$ and
$1026625235/1367228672$. Hence the two-step RR CDFs are again
$1/2$ and zero. For the analytic $B,\varepsilon$ family, the
same-label comparison now gives $E_1=\varepsilon/2$ and
$E_2=169\varepsilon/128$. For all $0<\varepsilon\leq1/2048$,
the perturbed B paths still remain inside: even the larger B second
position plus $169/(128\cdot2048)$ is less than $385/512-1/4096$.
The A positions and all first positions have boundary margin greater
than $1/4096$ as well. The iid identity for A and unperturbed B is
unchanged. Thus the obstruction concerns cross-position sample identity
under reshuffling; momentum is part of the comparison framework but
is not necessary for this mechanism. Since the boundary differs from
the positive-momentum example, this is not a fixed-event ablation or
a causal estimate of the effect of momentum.

\subsection{Full gradient covariance does not identify the exit law}
\label{sec:full-covariance-pairing}
The information restriction can be strengthened beyond center-based
summaries. Let $c=(-2,-1,1,2)$ and consider
\[
 \ell_i(x)=c_i x+\tfrac12 H_i x^2+\tfrac\nu4 x^4,
 \qquad H^A=(1,2,1/2,3),\quad H^B=(2,1/2,1,3).
\]
Fix $0\leq\nu<1/2$. All components are strongly convex, and both
pools have the same entire mean objective
$f(x)=13x^2/16+\nu x^4/4$. For a fresh uniform index $I$, both also have
the same entire gradient variance function
\[
 \operatorname{Var}(\ell'_I(x))=\tfrac52+\tfrac54x+\tfrac{59}{64}x^2.
\]
Their empirical laws of $c_i$ and $H_i$ agree as well. Nevertheless,
with $u_0=x_0=0$, $u_n=u_{n-1}/2-\ell'_{i_n}(x_{n-1})$,
$x_n=x_{n-1}+u_n/4$, and
$\tau=\inf\{n\geq1:|x_n|\geq55/64\}$, their two-step exit
probabilities are respectively $3/16$ and $1/4$ under iid uniform
sampling, and $1/12$ and $1/6$ under a uniformly shuffled first epoch.
Thus a deterministic scalar predictor determined only by these shared functions and
marginals, together with the stated algorithm, initialization, boundary
and sampling law, has worst-case absolute CDF error at least $1/32$
for iid sampling and $1/24$ for shuffling.

\paragraph{Proof.}
In both pools $\mathbb E c_I=0$, $\mathbb E H_I=13/8$,
$\mathbb E c_I^2=5/2$, $\mathbb E(c_I H_I)=5/8$ and
$\operatorname{Var}(H_I)=59/64$. The common cubic gradient cancels
from the variance, proving the function identities. At $\nu=0$,
$x_1=-c_i/4$ and
$x_2=-[(3/2-H_j/4)c_i+c_j]/4$.
Using indices $1,2,3,4$, the exiting ordered pairs are
$\{(1,1),(4,3),(4,4)\}$ in A and
$\{(1,1),(1,2),(4,3),(4,4)\}$ in B.
Deleting repeated indices gives the reshuffling counts.
Every first-step position has absolute value at most $1/2$.
Among nonexiting second-step positions in either pool the largest
absolute value is $13/16$; every exiting one has absolute value
at least $7/8$. Adding the quartic term changes $x_2$ by
$-\nu x_1^3/4$, of magnitude at most $\nu/32<1/64$.
Hence every exit indicator at radius $55/64$ is unchanged.
The error lower bounds follow from half the separation of the two
CDF values, since the restricted predictor returns the same value.

For $\nu=0$ the same counts hold throughout $13/16<R\leq7/8$;
more generally they hold on the nonempty interval
$13/16+\nu/32<R<7/8-\nu/32$ for the stated range of $\nu$.
This is an exploratory finite-step information example, not a claim
that covariance models fail for all stepsizes or noise classes.
The variance identity concerns fresh uniform sampling, not equality
of history-conditioned remaining-pool kernels under shuffling.
It neither implies unequal complete mean exit times nor contradicts
Gaussian-noise approximation theorems under their own assumptions.
All 24 quadratic assignments and five inspected radii are retained
in the accompanying exact enumeration, together with two- and
four-step results for $\nu=0$ and $\nu=1/4$; these are not independent
confirmation experiments.

\subsection{A finite table retaining cross-position sample identity}
\label{sec:labelled-grid}
The preceding information example admits a constructive comparison
once sample labels are retained across positions. This construction
uses standard interpolation and perturbation bounds; it supplies a
positive counterpart to the example, not a new interpolation theorem.

For each component, store its gradient values on a grid in $[-2,2]$
and linearly interpolate consecutive values with the same sample label,
obtaining $\widetilde g_s$. The gradients in
Appendix~\ref{sec:pointwise-law-rr} have Lipschitz constant $L=41/32$
and satisfy $|g_s''|\leq M=3/4$. Indeed,
\[
 d''(t)=\frac{t(3-t^2)}{2(1+t^2)^3},\qquad
 |d''(t)|\leq
 \frac{|t|}{1+t^2}\frac{3+t^2}{2(1+t^2)^2}\leq\frac34,
\]
and the same absolute second-derivative bound holds for $|d|$.
On a grid cell $[a,b]$, the interpolation residual at $y$ obeys
$|g_s(y)-\widetilde g_s(y)|\leq\epsilon(y):=M(y-a)(b-y)/2$.

Drive the reference by the same labels, with matched full initial state:
$v_n=\beta v_{n-1}-\widetilde g_{s_n}(y_{n-1})$ and
$y_n=y_{n-1}+hv_n$. Starting at $E_0=Q_0=0$, define
\[
 Q_n=\beta Q_{n-1}+LE_{n-1}+\epsilon(y_{n-1}),\qquad
 E_n=E_{n-1}+hQ_n.
\]
As long as the evaluated reference positions lie on the grid domain,
induction gives $|u_n-v_n|\leq Q_n$ and $|x_n-y_n|\leq E_n$:
split the gradient difference into
$g_{s_n}(x_{n-1})-g_{s_n}(y_{n-1})$ and the interpolation residual,
then use the global Lipschitz bound. Thus the same exit comparisons
based on $|y_n|\pm E_n$ apply, including the exit-causing update.
For the two-step construction all evaluated reference positions lie
in $[-2,2]$. No actual-path residual enters the reference or its bound.

With mesh width $1/2$, there are nine nodes and 18 offline component
gradient evaluations per pool. At $R=513/512$, exact rational evaluation
of the two equally likely first-epoch permutations gives exit-CDF
brackets $[1/2,1/2]$ for A and $[0,0]$ for B. The online reference
uses table lookups and interpolation, with no further true-gradient
queries. In contrast, sorting the two gradient values independently
at every node produces identical tables for A and B and discards the
information responsible for this distinction.

The complete exploratory mesh grid $1,1/2,1/4,1/8,1/16,1/32,1/64$
and radii $1,513/512$ are retained. Mesh width $1$ yields uninformative
$[0,1]$ brackets for both pools at both radii. At mesh $1/2$, the B
bracket at $R=1$ remains $[0,1]$, despite being exact at $513/512$;
boundary proximity still matters. All arithmetic for this finite
example is rational. The two-sample scalar enumeration is inexpensive
by direct simulation as well, so it establishes an information and
error-control connection, not a computational advantage or a scalable
neural-network method. A tensor grid in high dimension would be
prohibitively large without additional structure.

\section{Proofs of the bridge obstructions}\label{app:obstructions}
\subsection{A direct discrete counterexample at fixed diffusivity}
Fix $D,R>0$. For $0<\beta<1$, set
\[
 a_\beta=\sqrt{\frac{2D(1-\beta)}{1+\beta}},\qquad
 V_{n+1}=\beta V_n+\sqrt{1-\beta^2}\,Z_{n+1}.
\]
Here $V_0\sim N(0,1)$ and the $Z_n$ are iid standard normal independent
of $V_0$. Independently define a stationary symmetric sign chain $J_n$,
flipping with probability $(1-\beta)/2$ per update. Let
\[
 X_n=a_\beta\sum_{k=1}^n V_k,\quad
 Y_n=a_\beta\sum_{k=1}^n J_k,\quad
 \tau_G=\inf\{n\ge1:|X_n|\ge R\},\quad
 \tau_S=\inf\{n\ge1:|Y_n|\ge R\}.
\]
Both exits use integer observation times, position zero, and the same
interval. Only the initial velocity \emph{moments}, not its full law, match.

\begin{theorem}[Failure of a uniform mean bridge]\label{obs:thm:discrete}
Both complete means are finite for each $\beta$. All two-time position
covariances agree, and both diffusivities per update equal $D$. However,
as $\beta\uparrow1$,
\[
 \frac{a_\beta}{R}\E\tau_S\longrightarrow1,\qquad
 \frac{a_\beta}{R}\E\tau_G\longrightarrow\infty,\qquad
 \frac{\E\tau_G}{\E\tau_S}\longrightarrow\infty.
\]
\end{theorem}

\paragraph{Matching and finiteness.}
Both stationary velocities have covariance $\beta^{|j-k|}$.
Summing over the same position increments proves the covariance identity.
The Green--Kubo sum gives
$a_\beta^2(1+\beta)/(2(1-\beta))=D$.
Conditioned on any Gaussian-model state $(x,v)$, the next position has
variance $a_\beta^2(1-\beta^2)>0$. For a fixed variance, a normal law's
probability of leaving $(-R,R)$ is minimized at mean zero. Thus a positive
state-uniform one-step exit probability gives a geometric tail.
For the sign model, write $\delta=a_\beta/R$, $r=(1+\beta)/2$ and
$M=\lceil2/\delta\rceil$. No sign changes over $M$ updates force exit from
any interior position. Consequently
\begin{equation}\label{obs:eq:block}
 \sup_{\text{interior states}}\E\tau_S\le M/r^M<\infty .
\end{equation}
These bounds establish integrability, not numerical sharpness.

\subsection{Proof of the divergent ratio}
\paragraph{A coupling that preserves the Gaussian driver.}
Put $\epsilon=-\log(\beta)/\delta$. Both $\delta$ and $\epsilon$ tend to
zero. On a common probability space take $Z\sim N(0,1)$ independent of
Brownian motion $W$ and define stationary OU velocity
\begin{equation}\label{obs:eq:ou}
 U_\epsilon(t)=e^{-\epsilon t}Z+
 \sqrt{2\epsilon}\left[W_t-\epsilon\int_0^t
 e^{-\epsilon(t-s)}W_s\,ds\right].
\end{equation}
Its samples at $k\delta$ have exactly the stationary AR(1) law above.
For every fixed $H<\infty$,
\[
 \sup_{t\le H}|U_\epsilon(t)-Z|
 \le (1-e^{-\epsilon H})|Z|
 +\sqrt{2\epsilon}(1+\epsilon H)\sup_{t\le H}|W_t|
 \longrightarrow0\quad\text{a.s.}
\]
The linearly interpolated sampled positions
$\delta\sum_{k=1}^nU_\epsilon(k\delta)$ therefore converge uniformly on
compact intervals to $tZ$. For $Z\ne0$, this limiting line crosses the
boundary strictly at $1/|Z|$. Uniform convergence prevents an exit before
$1/|Z|-\eta$ and forces one before $1/|Z|+\eta$, for every $\eta>0$.
Within a linearly interpolated segment there is no boundary crossing if
both endpoints are inside; the discrete exit is within $\delta$ of the
interpolated exit. Hence
\[
 \delta\tau_G\longrightarrow |Z|^{-1}\quad\text{a.s.}
\]
The density of $Z$ is positive at zero, so $\E|Z|^{-1}=\infty$.
Fatou's lemma along any sequence $\beta\uparrow1$ proves
$\delta\E\tau_G\to\infty$. We do \emph{not} exchange a path limit and
expectation; failure of uniform integrability is central to this argument.

\paragraph{The sign model needs an explicit tail argument.}
Let $L=\lceil1/\delta\rceil$. Finite speed gives
$\delta\tau_S\ge L\delta$. With probability $r^L$ no sign changes occur
in the first $L$ updates and exit occurs by then. Applying
\eqref{obs:eq:block} conditionally at update $L$ yields
\begin{equation}\label{obs:eq:sign}
 L\delta\le\delta\E\tau_S
 \le L\delta+(1-r^L)\frac{M\delta}{r^M}.
\end{equation}
Now $L\delta\to1$, $M\delta\to2$ and
$(1-\beta)/\delta\to0$, whence $r^L,r^M\to1$.
The two bounds converge to one. This proves Theorem~\ref{obs:thm:discrete}.

\paragraph{What the theorem excludes.}
No finite constant can uniformly bound $\E\tau_G/\E\tau_S$ over this
family using their common second-order statistics as a justification.
This is a parameter-uniform counterexample, not a claim that every
fixed-parameter mean comparison fails. The sign chain is the sampled
velocity of a telegraph process; its position here sums endpoint
velocities. It is not the exact continuous integral of that telegraph.
The next section treats continuous observation separately.

\paragraph{What initialization contributes.}
The Gaussian velocity admits arbitrarily slow initial motion with
positive density near zero; the sign velocity has fixed magnitude.
Both are stationary with variance one. These facts explain how matched
correlations conceal the ballistic exit mechanism. Neither this stationary
construction nor its proof may be applied unchanged to the original
zero-velocity neural continuations.

\subsection{Continuous means and an explicit lower bound}
Let $Q_\epsilon(t)=\int_0^tU_\epsilon(s)\,ds$ with
\eqref{obs:eq:ou}, and let $P_\epsilon$ integrate a stationary symmetric
sign switching at rate $\epsilon/2$. Both velocity covariances are
$e^{-\epsilon|t-s|}$. Denote continuously observed exits from $(-1,1)$
by $\sigma_G,\sigma_T$.

\begin{proposition}\label{obs:prop:continuous}
For every $\epsilon>0$, both means are finite and
$\E\sigma_T=1+\epsilon/2$. As $\epsilon\downarrow0$,
$\E\sigma_G\to\infty$. Moreover, for $0<\epsilon\le1/128$,
\begin{equation}\label{obs:eq:log}
 \E\sigma_G\ge C\log\frac1{128\epsilon},\qquad
 C=\frac{\phi(1/2)[1-4\Phi(-2)]}{3}>0,
\end{equation}
where $\phi,\Phi$ are the standard normal density and CDF.
The same lower bound holds for $\delta\E\tau_G$ in
Theorem~\ref{obs:thm:discrete}, using its $\epsilon=-\log(\beta)/\delta$.
\end{proposition}

\begin{proof}
At any block start $(q,u)$, the unrestricted Gaussian endpoint at time
$M>0$ has variance
\[
 q_\epsilon(M)=\frac2\epsilon\int_0^M(1-e^{-\epsilon s})^2\,ds>0 .
\]
Thus the exit probability is at least
$2\Phi(-1/\sqrt{q_\epsilon(M)})>0$, uniformly in $(q,u)$,
and a geometric block tail proves finite mean. For telegraph, no switch
in the next two time units forces exit and has probability $e^{-\epsilon}$.
The bounded functions
$m_\pm(x)=1+\epsilon(1-x^2)/2\mp x$ satisfy
$\mathcal Lm=-1$ and outgoing conditions $m_+(1)=m_-(-1)=0$.
Dynkin's formula stopped at $\sigma_T\wedge H$, followed by bounded and
monotone convergence, gives the stated mean at zero.
The compact convergence in \eqref{obs:eq:ou} and Fatou prove divergence
of the Gaussian mean by the same coupling and Fatou argument.

For a quantitative bound, take $t\ge1$ and impose the independent events
\[
 |Z|<\frac1{2t},\qquad
 \sup_{s\le t}|W_s|<
 \frac1{2t\sqrt{2\epsilon}(1+\epsilon t)}.
\]
They imply $\sup_{s\le t}|U_\epsilon(s)|<1/t$ and hence survival to $t$.
This also keeps every sampled position with $n\delta\le t$ inside.
Set $T=(128\epsilon)^{-1/3}$. For $1\le t\le T$, $\epsilon t\le1$
and the Brownian threshold divided by $\sqrt t$ is at least 2.
The reflection principle and a union bound lower-bound its probability
by $1-4\Phi(-2)$. Also
$\Pp(|Z|<1/(2t))\ge\phi(1/2)/t$.
Multiplying and integrating survival over $[1,T]$ proves \eqref{obs:eq:log}.
\end{proof}

\paragraph{Keeping physical diffusivity fixed.}
For fixed $D,R>0$, choose speed $c=D\epsilon/R$, OU rate
$\gamma=D\epsilon^2/R^2$, and telegraph rate $\lambda=\gamma/2$.
Both physical diffusivities are $c^2/\gamma=D$.
Position and time units are $R$ and $R/c$; the physical mean ratio is
unchanged. Within each pair, all position covariances agree. Across pairs
the correlation time varies. The logarithmic inequality is a conservative
lower bound, not an asymptotic equivalent or an observed mean.

\subsection{Process and exit-distribution obstructions}
\begin{proposition}[Gaussian closure]
A deterministic linear momentum recursion with jointly Gaussian inputs
and initial state cannot converge, under deterministic affine spatial
rescalings, deterministic time changes, and deterministic linear
interpolation, to a nondegenerate finite-rate telegraph position.
\end{proposition}
\begin{proof}
Every fixed-time marginal is Gaussian. A weakly convergent sequence of
Gaussian laws is tight. Its standard deviations must be bounded, since
unbounded standard deviation sends the probability of every bounded
interval to zero along a subsequence. With bounded deviations the means
must also be bounded. Convergent parameter subsequences and characteristic
functions show that the limit is Gaussian, possibly a point mass.
For telegraph speed $c>0$ and $0<\lambda<\infty$, the no-switch event
at $t>0$ creates mass $e^{-\lambda t}$ at $ct$ or $-ct$ (or both for
a random initial sign), while switches create nonconstant interior
positions. This is neither a nondegenerate Gaussian nor a point mass.
Thus even marginal convergence is impossible. Continuous-path or
Skorokhod $J_1$ convergence would imply these marginals, since the
telegraph position is continuous.
\end{proof}
This includes deterministic zero initialization and parameter sequences
with $\beta\to1$. It does not cover nonlinear transformations, random
clocks, arbitrary nonlinear SGD, or a prescribed persistent sign driver.
A fast-switching Brownian limit is compatible with the result.
Appendix~\ref{app:telegraph-bridge} treats the persistent-sampler case
and proves a finite-rate telegraph bridge there.

\begin{proposition}[A ballistic exit atom]
Let telegraph start at zero in $(-R,R)$ with either initial sign and
write $t_*=R/c$. For a continuous comparison process starting inside,
assume $\Pp(|X(t_*)|=R)=0$. Its exit law and the telegraph exit law satisfy
\[
 d_K\ge\tfrac12e^{-\lambda R/c},\qquad
 d_{\rm TV}\ge e^{-\lambda R/c},
\]
where $d_K$ is the supremum CDF distance and
$d_{\rm TV}=\sup_A|P(A)-Q(A)|$.
\end{proposition}
\begin{proof}
Telegraph cannot exit before $t_*$ and exits exactly then if and only
if it never switches, up to null events. Its exit CDF has a jump
$p=e^{-\lambda R/c}$. Continuity of $X$ implies
$\{\tau_X=t_*\}\subseteq\{|X(t_*)|=R\}$, so the other CDF is continuous
there. Comparison at the jump and its left limit gives $p/2$; the
singleton $\{t_*\}$ gives the total-variation bound.
\end{proof}
No complete-mean assumption is needed, and a possible mass at infinity
is allowed. These inequalities alone do not imply mean separation or
rule out weak convergence to an atomic law.

\paragraph{Matching all covariance, not just diffusivity.}
Stationary OU velocity with variance $c^2$ and rate $\gamma=2\lambda$
and stationary telegraph velocity have covariance
$c^2e^{-\gamma|t-s|}$. Their integrated positions share every two-time
covariance, with variance
$2c^2[t/\gamma-(1-e^{-\gamma t})/\gamma^2]$.
For $\chi=c/(\lambda R)$, the atom bound is $e^{-1/\chi}/2$.
At $t=\theta R/c$, $0<\theta<1$, the Gaussian endpoint also gives
\[
 d_K\ge2\Phi\!\left(
 -[\chi\theta-\tfrac12\chi^2(1-e^{-2\theta/\chi})]^{-1/2}\right).
\]
This is a probability lower bound, not the actual Gaussian exit CDF.

\section{A telegraph bridge under persistent sampling}\label{app:telegraph-bridge}
\providecommand{\Tt}{\mathsf T}
\providecommand{\Ii}{\mathcal I}
This appendix states and proves the bridge summarized in
Theorem~\ref{thm:telegraph-bridge}. The component tools are standard
and cited as such: the averaged-iterate form of heavy ball, backward
equations of piecewise-deterministic Markov processes (PDMPs), and
optional stopping. The contribution is their assembly into an explicit
mean-exit bracket between discrete momentum SGD and a finite-rate
switching limit, together with the identification of the sampler, not the
momentum, as the source of persistence.

\subsection{Setting}
Let $\Ii=(\ell,r)$, $\rho>0$ and $\Ii_\rho=(\ell-\rho,r+\rho)$. For
$\Ii'=(\ell',r')$ write $\Ii'_{\pm e}=(\ell'\mp e,r'\pm e)$.

\paragraph{Sampler states and velocities.}
A finite set $\mathcal S$ of sampler states (single samples or
minibatches) carries component-gradient fields
$g_s\in C^1(\bar\Ii_\rho)$, written in the displacement coordinate of
Section~\ref{sec:process}. They define velocities
$v_s=-g_s/(1-\beta)$, with
\[
V=\max_s\sup_{\bar\Ii_\rho}|v_s|,\qquad
L=\max_s\sup_{\bar\Ii_\rho}|v_s'|.
\]
We assume \emph{non-degeneracy}: every $v_s$ keeps a constant sign on
$\bar\Ii_\rho$ with $|v_s|\ge c_0>0$. Equivalently, no component gradient
vanishes on the region. Write $\mathcal S_+$ and $\mathcal S_-$ for the
right- and left-moving states.

\paragraph{Persistent sampler.}
A \emph{persistent sampler} is a Markov chain $S_n$ on $\mathcal S$,
independent of the iterates given its past, with transition matrix
$P=I+hQ$ for an irreducible generator $Q$ (a stochastic matrix when
$h\max_s|Q_{ss}|\le1$, which we assume throughout). Put
$\bar q=\max_s\sum_{s'\ne s}Q_{ss'}$.
\begin{itemize}
\item Lazy resampling keeps the active state with probability $1-p$ and
otherwise draws uniformly. Then $Q=(p/h)(\Pi-I)$, where every entry of
$\Pi$ is $1/|\mathcal S|$, and $\bar q=(p/h)(1-1/|\mathcal S|)$. Data
echoing with geometric echo counts is an example.
\item For two states $\mathcal S=\{\pm1\}$ with flip probability $p$, the
limit is the scalar telegraph with rate $\lambda=p/h=\bar q$, drift
$(v_{+1}+v_{-1})/2$ and speed $c=(v_{+1}-v_{-1})/2$.
\end{itemize}

\paragraph{Heavy ball and the backward problem.}
Heavy ball uses the active state:
\[
u_{n+1}=\beta u_n-g_{S_n}(x_n),\qquad x_{n+1}=x_n+hu_{n+1},\qquad
\tau=\inf\{n\ge1:x_n\notin\Ii\}.
\]
This is the recursion of Section~\ref{sec:process}, with the active state
$S_n$ in place of the minibatch $B_{n+1}$. The reference is the PDMP
$dY_t=v_{S(t)}(Y_t)\,dt$, where $S$ is the continuous-time chain with
generator $Q$.

For $[\ell',r']\subset\bar\Ii_\rho$, let
$\Tt^{\Ii'}=(\Tt_s)_{s\in\mathcal S}$ solve, on $\bar\Ii_\rho$,
\begin{equation}\label{br:bvp}
v_s\Tt_s'+\sum_{s'}Q_{ss'}(\Tt_{s'}-\Tt_s)=-1,\qquad
\Tt_s(r')=0\ (s\in\mathcal S_+),\quad \Tt_s(\ell')=0\ (s\in\mathcal S_-).
\end{equation}
The boundary conditions are imposed on outgoing states only. Let
$\Tt_0,\Tt_1,\Tt_2$ be the suprema over $\bar\Ii_\rho$ and $s$ of
$|\Tt_s|$, $|\Tt_s'|$ and $|\Tt_s''|$. Put
\[
U=\max(|u_0|,V),\qquad \epsilon=\frac{h\beta U}{1-\beta},\qquad
z_0=x_0+\frac{h\beta u_0}{1-\beta}.
\]

\subsection{Backward problem and one-step consistency}
\begin{proposition}[Backward problem]\label{br:lem:bvp}
Problem \eqref{br:bvp} has a unique solution. It is $C^2$ on
$\bar\Ii_\rho$, and on $\Ii'$ it equals the PDMP mean exit time
from $\Ii'$; endpoint values are one-sided interior traces. In
particular it is nonnegative on $[\ell',r']$. If $W$ solves the
homogeneous system on $[a,b]$, then every $W_s(y)$ lies between the
minimum and the maximum of the outgoing boundary data $\{W_s(b)\}_{s\in\mathcal
S_+}\cup\{W_s(a)\}_{s\in\mathcal S_-}$. Moreover
\[
\Tt_s'=-\frac{1+\sum_{s'}Q_{ss'}(\Tt_{s'}-\Tt_s)}{v_s},\qquad
\Tt_s''=-\frac{\sum_{s'}Q_{ss'}(\Tt_{s'}'-\Tt_s')+\Tt_s'v_s'}{v_s}.
\]
\end{proposition}
\begin{proof}
\emph{Exit is fast.} Let $(Y,S)$ be the PDMP started in $(a,b)$ and
$\sigma$ its exit time. Suppose $\mathcal S_+\neq\emptyset$ (otherwise use
$\mathcal S_-$). By irreducibility, the chain enters a right-moving state
within unit time with probability at least $q_1>0$. It then stays there
for time $(b-a)/c_0$, which forces exit, with probability at least
$e^{-\bar q(b-a)/c_0}$. These bounds are uniform in the starting point,
so $\sigma$ has a geometric tail and $\E\sigma<\infty$.

\emph{Uniqueness and the bounds.} Let $W$ be a $C^1$ solution of the
homogeneous system, extended in $C^1$ beyond $[a,b]$. Dynkin's formula for
PDMPs \citep{davis1993markov} makes $W_{S_t}(Y_t)$, stopped at $\sigma$, a
bounded martingale. Hence $W_s(y)=\E_{y,s}[W_{S_\sigma}(Y_\sigma)]$.

At an exit through $b$ the last motion is rightward, so
$S_\sigma\in\mathcal S_+$ almost surely (a jump at the exit instant has
probability zero). Symmetrically, exit through $a$ occurs in
$\mathcal S_-$. So $W_s(y)$ is an average of outgoing boundary data, which
gives the bounds. Continuity extends these bounds to the endpoints.
With zero data it gives uniqueness.

\emph{Existence.} Solutions of the inhomogeneous system form an affine
family of dimension $|\mathcal S|$. There are exactly $|\mathcal S|$
boundary conditions, and the affine map from initial data at $\ell'$ to
the boundary values is injective by uniqueness, hence onto. Regularity
follows from $v_s\in C^1$ and $v_s\neq0$.

\emph{Mean exit time.} The same Dynkin argument applied to the solution
gives $\Tt_s(y)=\E_{y,s}\sigma'$, the mean exit time from $\Ii'$. The
derivative identities follow from the equation.
\end{proof}

\begin{proposition}[Averaged iterate and one-step consistency]\label{br:lem:step}
Let $z_n=x_n+h\beta u_n/(1-\beta)$. Then, exactly,
$z_{n+1}=z_n+hv_{S_n}(x_n)$. If $|u_n|\le U$ then $|x_n-z_n|\le\epsilon$,
and if in addition $x_n\in\bar\Ii_\rho$ then $|u_{n+1}|\le U$.

Let $\Phi_n=\Tt_{S_n}(z_n)$ with $\Tt=\Tt^{\Ii'}$, and suppose
$z_n\in\Ii'$, $x_n\in\Ii_{+2\epsilon}$, $|u_n|\le U$ and
$[z_n,z_{n+1}]\subset\bar\Ii_\rho$. Then
\[
\E[\Phi_{n+1}\mid S_0,\dots,S_n]=\Phi_n-h+r_n,\qquad |r_n|\le Kh^2,
\qquad
K=\frac{L\beta U\Tt_1}{1-\beta}+\frac{V^2\Tt_2}{2}+2\bar qV\Tt_1.
\]
\end{proposition}
\begin{proof}
The identity is the iterate-moving-average form of stochastic heavy ball
\citep{sebbouh2021almost}:
$z_{n+1}=x_n+hu_{n+1}/(1-\beta)=z_n-hg_{S_n}(x_n)/(1-\beta)$. The velocity
bound uses $|g_s|\le(1-\beta)V$.

Write $s=S_n$. Since $z_{n+1}$ is determined by $S_0,\dots,S_n$ and
$P=I+hQ$,
\[
\E[\Phi_{n+1}\mid\cdot]=\Tt_s(z_{n+1})+h\sum_{s'}Q_{ss'}(\Tt_{s'}-\Tt_s)(z_{n+1}).
\]
Expand $\Tt_s(z_{n+1})$ by Taylor's formula, with increment
$hv_s(x_n)$, and subtract $h[v_s\Tt_s'+\sum_{s'}Q_{ss'}(\Tt_{s'}-\Tt_s)+1](z_n)=0$.
Three remainders are left:
\begin{itemize}
\item $h(v_s(x_n)-v_s(z_n))\Tt_s'(z_n)$, at most $hL\epsilon\Tt_1$;
\item the Taylor remainder, at most $h^2V^2\Tt_2/2$;
\item $h\sum_{s'}Q_{ss'}[(\Tt_{s'}-\Tt_s)(z_{n+1})-(\Tt_{s'}-\Tt_s)(z_n)]$,
at most $h\bar q\cdot hV\cdot2\Tt_1$.
\end{itemize}
\end{proof}

\subsection{The bridge}
\begin{theorem}[Momentum SGD to finite-rate switching limit]\label{br:thm:A}
Assume non-degeneracy, $x_0\in\Ii$, $2\epsilon<r-\ell$, and
$\max(2\epsilon,\epsilon+hV)\le\rho$. Let
$\Tt^{\pm}=\Tt^{\Ii_{\pm\epsilon}}$, and let $K_\pm$ and
$B_\pm=V\Tt^\pm_1+\bar q\,\Tt^\pm_0$ be computed from $\Tt^\pm$. Assume
$K_\pm h<1$ and define
\[
E_\pm=B_\pm+\frac{K_\pm\bigl(\Tt^\pm_{S_0}(z_0)+hB_\pm\bigr)}{1-K_\pm h}.
\]
Then $\tau<\infty$ almost surely and
\[
\Tt^{-}_{S_0}(z_0)-hE_-\;\le\;h\,\E\tau\;\le\;\Tt^{+}_{S_0}(z_0)+hE_+,
\]
where the lower bound is replaced by $0$ if $z_0\notin\Ii_{-\epsilon}$.
If moreover $z_0\in\Ii_{-\epsilon}$, then
\[
|h\E\tau-\Tt^{\Ii}_{S_0}(x_0)|\le C_Ah,\qquad
C_A=\max_\pm E_\pm+\frac{2\beta U\max(\Tt^{\Ii}_1,\Tt^{+}_1)}{1-\beta}.
\]
\end{theorem}
\begin{proof}
Let $\tau_\pm$ be the exit times of $z$ from $\Ii_{\pm\epsilon}$.

\emph{Hypotheses before $\tau_+$.} By Proposition~\ref{br:lem:step} and
induction, for $n<\tau_+$ we have $|u_n|\le U$,
$x_n\in\Ii_{+2\epsilon}\subset\Ii_\rho$ and
$z_{n+1}\in\Ii_{+\epsilon+hV}\subset\Ii_\rho$. Moreover $z$ moves in the
direction of $v_{S_n}$.

\emph{Sandwich.} Since $|x_n-z_n|\le\epsilon$, an exit of $x$ forces
$z_\tau\notin\Ii_{-\epsilon}$, so $\tau_-\le\tau$. If $\tau>\tau_+$, then
$x_{\tau_+}\in\Ii$ would contradict $z_{\tau_+}\notin\Ii_{+\epsilon}$, so
$\tau\le\tau_+$.

\emph{Mean of $\tau'$.} Fix $\Ii'=\Ii_{\pm\epsilon}$ with exit time
$\tau'$. The process
$M_n=\Phi_{n\wedge\tau'}+h(n\wedge\tau')-\sum_{k<n\wedge\tau'}r_k$ is a
martingale, and $\Phi\ge-\Tt_0$. Hence
$h\E[N\wedge\tau'](1-Kh)\le\Tt(z_0)+\Tt_0$, so $\E\tau'<\infty$ without an
a priori assumption. Letting $N\to\infty$ gives
$h\E\tau'=\Tt(z_0)-\E\Phi_{\tau'}+\E\sum_{k<\tau'}r_k$.

\emph{Exit value.} The last move before exit is taken in an outgoing
state, so $|\Phi_{\tau'}|\le hV\Tt_1$ unless the sampler moves at the exit
step to a state that is not outgoing. The exit step is determined by the
earlier sampler states, so that move has probability at most $h\bar q$ and
contributes at most $h\bar q\Tt_0$. Combining the three steps proves the
bracket.

\emph{Rate.} On $\Ii$, the difference $\Tt^+-\Tt^{\Ii}$ solves the
homogeneous system with outgoing boundary data in $[0,\epsilon\Tt^+_1]$. On
$\Ii_{-\epsilon}$, the difference $\Tt^{\Ii}-\Tt^-$ solves it with data in
$[0,\epsilon\Tt^{\Ii}_1]$. Proposition~\ref{br:lem:bvp} bounds both
differences, and $|z_0-x_0|\le\epsilon$.
\end{proof}

\paragraph{Interpretation.}
The map is: time $t=nh$; velocities $-g_s/(1-\beta)$; switching generator
$Q=(P-I)/h$ from the sampler, for example rate $p/h$ for lazy resampling.
The switching state is the active sample, not a momentum sign. Momentum
contributes only the speed factor $1/(1-\beta)$ and the lag $\epsilon$,
which is the distance traveled on residual momentum after a switch.

For the flat two-sample pool $g_{\pm1}=\mp q$ on $(-R,R)$, starting
from $x_0=u_0=0$, the theorem gives
$h\E\tau=R^2/(2D)+R/c+O(h)$ with $c=q/(1-\beta)$ and
$D=c^2/(2\lambda)$. The ballistic boundary correction $R/c$ of the
telegraph is therefore a property of momentum SGD with this sampler.

With learning rate $e^2$, $1-\beta=\mu e$ and switching probability of
order $e$, the lag tends to $\max_s\sup_{\Ii_\rho}|g_s|/\mu^2$ as $e\to0$
while the step terms vanish. At fixed $\mu$, the discrete limit is the
second-order process of Theorem~\ref{br:thm:B}, not the first-order
limit.

\subsection{Second-order dynamics, fast switching, and iid sampling}
\begin{theorem}[Second-order to first-order switching limit]\label{br:thm:B}
Let $X'=V$ and $V'=-\mu V-g_{S(t)}(X)$, where $S$ has generator $Q/\mu$ in
physical time, and let $\theta=t/\mu$. Assume $X_0\in\Ii$ and the
same component non-degeneracy as above. Put $v_s=-g_s$,
$L=\max_s\sup_{\bar\Ii_\rho}|g_s'|$,
$G=\max_s\sup_{\Ii_\rho}|g_s|$ and $\epsilon_\mu=\max(|V_0|,G/\mu)/\mu$.
Let $\Tt^\pm$ solve \eqref{br:bvp} with generator $Q$ on
$\Ii_{\pm\epsilon_\mu}$, and set $\eta_\pm=L\epsilon_\mu\Tt^\pm_1<1$ with
$2\epsilon_\mu<\min(\rho,r-\ell)$. Then the slow-time exit $\theta_X$ of $X$ from $\Ii$
satisfies
\[
\Tt^-_{S_0}(Z_0)\frac{1-2\eta_-}{1-\eta_-}\le\E\theta_X
\le\frac{\Tt^+_{S_0}(Z_0)}{1-\eta_+},\qquad Z_0=X_0+V_0/\mu.
\]
The lower bound is replaced by zero if $Z_0\notin\Ii_{-\epsilon_\mu}$.
For fixed interior $X_0$,
$|\E\theta_X-\Tt^{\Ii}_{S_0}(X_0)|=O(\mu^{-2})$ when
$V_0=O(1/\mu)$.
\end{theorem}
\begin{proof}
The process $Z=X+V/\mu$ satisfies $dZ/d\theta=v_{S(\mu\theta)}(X)$
exactly. Solving the velocity equation gives $|V|\le\max(|V_0|,G/\mu)$
while $X\in\Ii_\rho$, so $|X-Z|\le\epsilon_\mu$.

Apply Dynkin's formula for PDMPs \citep{davis1993markov} to $\Tt_S(Z)$.
The backward equation at $Z$ leaves an integrand bounded by $\eta$. For either adjusted interval, stop first at its exit or a deterministic
time $t$. Nonnegativity of the backward solution inside the interval and
$\eta<1$ give $(1-\eta)\E(t\wedge\theta_Z)\le\Tt_{S_0}(Z_0)$,
so the adjusted exit is integrable. Its exit value is zero: $Z$ exits
continuously in an outgoing state, and a switch at the exit instant has
probability zero. Letting $t\to\infty$ gives
$\Tt_{S_0}(Z_0)/(1+\eta)\le\E\theta_Z
\le\Tt_{S_0}(Z_0)/(1-\eta)$.
The lower endpoint here is at least
$\Tt_{S_0}(Z_0)(1-2\eta)/(1-\eta)$.
The exit sandwich and maximum-principle comparison of adjusted intervals,
as in Theorem~\ref{br:thm:A}, prove the stated bracket and rate.
\end{proof}
For convex pools, \citet[Theorem~6]{jin2025momentum} bound the
finite-horizon distance between the same stochastic gradient-momentum
process and the first-order switching process.
Theorem~\ref{br:thm:B} instead controls the mean exit time, without
convexity.

\begin{proposition}[Fast switching]\label{br:prop:kac}
Let $a_\delta^2/\lambda_\delta\to1$ and let $S^\delta$ flip at rate
$\lambda_\delta$. Then $\int_0^ta_\delta S^\delta_u\,du$ converges to
Brownian motion in $C[0,T]$ if and only if $\lambda_\delta\to\infty$.
\end{proposition}
\begin{proof}
Sufficiency is classical \citep{kac1974telegrapher,pavliotis2008multiscale};
we give a complete argument. Write $a=a_\delta$, $\lambda=\lambda_\delta$,
$S=S^\delta\in\{-1,1\}$ and $X_t=\int_0^taS_u\,du$. The generator
$Lf(s)=\lambda[f(-s)-f(s)]$ gives $Ls=-2\lambda s$, so
$N_t=S_t-S_0+2\lambda\int_0^tS_u\,du$ is a martingale whose jumps are
$\pm2$ at the flip times; hence $[N]_t=4F_t$ with $F$ the Poisson
process of flips, of rate $\lambda$. Then
$X_t=\tfrac{a}{2\lambda}(N_t-S_t+S_0)$. The term
$\tfrac{a}{2\lambda}|S_0-S_t|\le a/\lambda=(a^2/\lambda)^{1/2}\lambda^{-1/2}$
vanishes uniformly when $\lambda\to\infty$. The martingale
$M=\tfrac{a}{2\lambda}N$ has jumps of size $a/\lambda\to0$ and
$[M]_t=(a^2/\lambda)\,F_t/\lambda\to t$ in probability for each $t$,
because $F_t/\lambda\to t$. The martingale functional central limit
theorem \citep[Theorem~7.1.4]{ethier1986markov} gives $M\Rightarrow W$
in $D[0,T]$ for a standard Brownian motion $W$; hence $X\Rightarrow W$,
and since $X$ and $W$ are continuous the convergence holds in $C[0,T]$.

For necessity, suppose $\lambda_\delta\not\to\infty$ and pass to a
subsequence with $\lambda_\delta\to\lambda<\infty$. If $\lambda>0$, the
limit is a telegraph with speed $\sqrt\lambda$, whose paths are
$\sqrt\lambda$-Lipschitz, so the limit is not Brownian. If $\lambda=0$,
the limit is $0$.
\end{proof}
A Brownian (Kac) limit of telegraph forcing therefore describes the
fast-switching regime. The finite-rate limit of Theorem~\ref{br:thm:A}
is a different limit.

\begin{proposition}[iid sampling averages]\label{br:prop:iid}
Assume iid sample indices from a finite pool, fixed $\beta$, component
fields bounded by $G$ on a set containing $\Ii_\rho$ and the ODE path
$y([0,T])$, and an $L_g$-Lipschitz mean field $\bar g$, where
$y'=-\bar g(y)/(1-\beta)$ with $y(0)=x_0$. Then, with an explicit $C_T$,
\[
\E\Bigl[\max_{n\le T/h,\,n<\tau}|x_n-y(nh)|^2\Bigr]^{1/2}
\le e^{L_gT/(1-\beta)}\Bigl(\frac{4G\sqrt{hT}}{1-\beta}+C_Th\Bigr).
\]
\end{proposition}
\begin{proof}
By Proposition~\ref{br:lem:step},
$z_{n+1}=z_n-\tfrac{h}{1-\beta}\bar g(x_n)-\tfrac{h}{1-\beta}\xi_n$, where
$\xi_n=g_{S_n}(x_n)-\bar g(x_n)$ is a martingale difference bounded by
$2G$. Since $|g_s|\le G$, $|u_n|\le U_0:=\max(|u_0|,G/(1-\beta))$ and hence
$|x_n-z_n|\le hA$ with $A=\beta U_0/(1-\beta)$. Let $y_n=y(nh)$ and
$e_n=z_n-y_n$. Because $|y''|\le L_gG/(1-\beta)^2$, Taylor's formula gives
$y_{n+1}=y_n-\tfrac{h}{1-\beta}\bar g(y_n)+\varrho_n$ with
$|\varrho_n|\le h^2L_gG/[2(1-\beta)^2]$. Subtracting, using
$|\bar g(x_n)-\bar g(y_n)|\le L_g(|e_n|+hA)$ and summing,
\[
\max_{k\le n}|e_k|\le |e_0|+\tfrac{h}{1-\beta}\max_{k\le n}|M_k|
+h\,TL_g\Bigl(\tfrac{A}{1-\beta}+\tfrac{G}{2(1-\beta)^2}\Bigr)
+\tfrac{hL_g}{1-\beta}\sum_{k<n}\max_{j\le k}|e_j|
\]
for $n\le T/h$ and $n<\tau$, where $M_k=\sum_{j<k\wedge\tau}\xi_j$ is the
martingale stopped at exit and $|e_0|\le h\beta|u_0|/(1-\beta)$. The
discrete Gronwall inequality multiplies the first three terms by
$e^{L_gT/(1-\beta)}$. Doob's $L^2$ inequality gives
$\|\max_{k\le T/h}|M_k|\|_2\le2\|M_{T/h}\|_2\le4G\sqrt{T/h}$.
Taking $L^2$ norms and adding $|x_n-z_n|\le hA$ proves the bound with
\[
C_T=A+\frac{\beta|u_0|}{1-\beta}
+TL_g\Bigl(\frac{A}{1-\beta}+\frac{G}{2(1-\beta)^2}\Bigr).
\]
\end{proof}
Finite-window trajectory convergence alone does not imply convergence
of exit times. If the ODE crosses the boundary transversally at a finite
exit time, the same estimate localized to a neighborhood of the crossing
implies convergence of $h\tau$ in probability to that time. If the ODE
stays a positive distance inside the interval on every finite window,
then $h\tau\to\infty$ in probability. Tangential boundary contacts require
separate analysis. In these two cases a non-degenerate finite-rate
switching exit law cannot be the limit; no convergence of complete means
is asserted without uniform integrability.

With switching probability $p_h$, the regimes $p_h/h\to\infty$,
$p_h/h\to\lambda$ and $p_h/h\to0$ give, respectively, the averaged ODE,
the switching limit and the frozen sample. This is the standard trichotomy
for Markov-modulated stochastic approximation
\citep{yin2004regime,kushner2003stochastic}. Iid sampling has $p_h=O(1)$
and lies in the first regime.

These results complement the Gaussian-closure obstruction of
Appendix~\ref{app:obstructions}, which excludes finite-rate telegraph
limits for linear Gaussian momentum recursions but leaves persistent
drivers open.

\subsection{Numerical checks}\label{br:sec:num}
\paragraph{Design.}
All settings were fixed before simulation. There are four two-sample
pools on $\Ii=(-1,1)$ in displacement coordinates:
\begin{itemize}
\item flat, $g_\sigma=-\sigma$;
\item well, $g_\sigma(x)=x/2-\sigma$;
\item saddle, $g_\sigma(x)=-x/2-\sigma$;
\item a heterogeneous ridge-logistic pair with features $(2,1.5)$, labels
$(1,0)$ and ridge $0.01$, centered at the minimizer of its mean loss.
\end{itemize}
We use $\beta\in\{0,0.5,0.9\}$, steps $h=\tilde h(1-\beta)$ with
$\tilde h\in\{0.02,0.01,0.005,0.0025,0.00125\}$, rates
$\lambda=\tilde\lambda/(1-\beta)$ with $\tilde\lambda\in\{0.5,2\}$,
$x_0=u_0=0$, $S_0=+1$, and $10^6$ paths per cell.

Telegraph means solved from \eqref{br:bvp} agree with exact
continuous-time telegraph simulation in 12 of 12 checks, all within
$3.29$ standard errors. The declaration fixed $\rho=0.3$: 68 of the
120 cells satisfy its conditions, and all 68 brackets contain their
99.9\% Monte Carlo intervals. A subsequent deterministic analysis uses
the smallest admissible margin $\rho=\max(2\epsilon,\epsilon+hV)$,
giving 82 applicable cells and 82 contained intervals. Table~\ref{br:tab:A}
reports this latter analysis, not a newly preregistered criterion;
the original declaration and both analyses are retained in the supplement.
Constants are float64 evaluations of
the backward solution on a grid, not outward-rounded certificates.

\begin{table}[ht]
\centering
\caption{Theorem~\ref{br:thm:A} for two-sample pools in 120 cells (30 per
pool): cells that satisfy the hypotheses, and cells whose 99.9\% Monte
Carlo interval lies entirely inside the bracket. Slopes are least-squares
fits of $\log|h\E\tau-\Tt|$ against $\log h$ within each (momentum, rate)
group where the error exceeds three standard errors.}
\label{br:tab:A}
\small
\begin{tabular}{@{}lrrcr@{}}
\toprule
 & \multicolumn{2}{c}{Cells of 30} & & \\
\cmidrule(lr){2-3}
Pool & hypotheses & MC inside & Slope range & Max.\ rel.\ error, $\tilde h=0.00125$\\
\midrule
Flat & 30 & 30 & 0.53--0.99$^{*}$ & 1.7\%\\
Well & 18 & 18 & 0.93--1.18 & 0.65\%\\
Saddle & 28 & 28 & 0.67--1.10 & 2.3\%\\
Logistic pair & 6 & 6 & 0.93--1.05 & 0.64\%\\
\bottomrule
\end{tabular}

\smallskip
{\footnotesize $^{*}$Slopes below one in the flat pool reflect the
floating-point boundary effect described below.}
\end{table}

\paragraph{Findings.}
No applicable bracket fails; resolved errors are broadly consistent
with the theoretical $O(h)$ rate. Over the 63 hypothesis
cells with resolved error, $C_A$ exceeds the observed $|\text{error}|/h$
by factors between 2.5 and 368, with median 8.2.

The error coefficient grows with $\beta$, as the lag term predicts. Within
each flat, well or saddle group, the median $|\text{error}|/h$ increases
monotonically in $\beta$.

For the logistic pair with $\tilde\lambda=2$, the telegraph exit is nearly
rare ($\Tt\approx28$ in $\beta=0$ units, $c_0=0.088$). The hypotheses fail
at every tested step there, although the observed error still decays at
rate $h$.

\paragraph{Eight sampler states.}
A one-dimensional ridge-logistic pool with eight samples uses:
\begin{itemize}
\item features drawn from $U[1,2.5]$, alternating labels, ridge $0.01$;
\item lazy resampling with $p=\tilde\lambda\tilde h$;
\item $\beta\in\{0,0.9\}$, $\tilde\lambda\in\{0.5,2\}$, four steps, and
$4\times10^5$ paths per cell.
\end{itemize}
Every component velocity keeps its sign on the working interval
($c_0=0.072$). The mean exit times converge to the eight-state backward
solution. In the three groups with resolved error, the fitted slopes are
1.07, 1.11 and 0.70; the last is affected by Monte Carlo noise at the
smallest step. In the fourth group all errors are within two standard
errors. With eight states the explicit constants are larger, and the
hypotheses hold in 2 of the 16 cells; both brackets contain the Monte
Carlo interval.

\paragraph{A floating-point boundary effect.}
In the flat pool, $1/\tilde h$ is an integer and floating-point positions
settle on the lattice $\tilde h\mathbb Z$, so exits occur exactly at the
boundary.

For $\beta=0$, the exact discrete mean, computed from a linear system,
gives error$/h$ approaching $-0.25$ and $2.00$ at the two rates
as $h$ decreases. Float64 Monte Carlo
deviates from it by up to 5.4 standard errors at the two smallest steps.

Rerunning every cell with common random numbers and boundaries
$\pm(1\mp10^{-9})$ changes flat-pool means by up to 45 standard errors.
Other pools change by at most $5\times10^{-4}$ standard errors, and both
perturbed runs stay inside the brackets. The flat-pool error constant is
therefore not resolved by float64 simulation; containment and rate are
unaffected.

\paragraph{Second-order dynamics.}
For the flat and well pools, with $\mu\in\{2,4,\dots,64\}$ and
$\Lambda\in\{0.5,2\}$, we simulate the second-order dynamics exactly,
piece by piece between switches. The slopes of $\log|\text{error}|$
against $\log\mu$ lie between $-2.35$ and $-1.94$.

All 21 hypothesis cells intersect the bracket of
Theorem~\ref{br:thm:B}, and 14 lie inside it. The others have brackets
narrower than the Monte Carlo interval. In the flat pool, the observed
error is about 90\% of the upper bracket.

\paragraph{iid diagnostic.}
For the flat pool with iid sampling, $h\E\tau$ exceeds the diffusion
value $R^2/(2D)$ by a nearly step-independent amount: 1.03--1.07 time
units at $\beta=0.5$ and 0.66--0.69 at $\beta=0.9$.

A telegraph calibrated from the momentum autocorrelation,
$\lambda=-\log(\beta)/(2h)$ and $c=\sqrt{2\lambda D}$, predicts 0.85 and
0.44. This has the right order but a coefficient 18\% and 36\% too
small. Applying the map of Theorem~\ref{br:thm:A} outside its regime
predicts about half of the observed mean.

These are diagnostics of calibration choices. Beyond
Proposition~\ref{br:prop:iid}, no theorem is claimed for iid sampling.

\paragraph{Scope.}
The results cover one dimension, finitely many sampler states whose
velocities keep their signs on the region, and mean exit times. The
following remain open:
\begin{itemize}
\item stagnation points inside the domain. These are typical for large
minibatches near a minimizer, where batch-mean gradients change sign, so
the minibatch case is covered only when every batch gradient is
sign-definite on the region;
\item $d\ge2$, where exits from a ball have grazing sets;
\item deterministic echo lengths.
\end{itemize}
The frozen studies of Appendix~\ref{sec:fresh-pool} use iid and
reshuffled sampling, so this appendix does not apply to them.

\section{Computable comparison with heterogeneous curvature}
\label{sec:heterogeneous}
Fix a finite gradient pool $g_i:\mathbb R^d\to\mathbb R^d$, a center
$w_*$, a step size $h>0$, and momentum $\beta\in[0,1)$.
For the realized batches $B_n$ (with multiplicities), let $g_B$ denote
their average. The sampling rule may depend on the actual past; all
comparisons below use exactly its realized batches. Conditional on any
random calibration, define $c_B=g_B(w_*)$ and
$H_B=|B|^{-1}\sum_{i\in B}\nabla g_i(w_*)$. Consider
\[
 u_n=\beta u_{n-1}-g_{B_n}(w_*+x_{n-1}),\quad x_n=x_{n-1}+hu_n,
 \qquad v_n=\beta v_{n-1}-c_{B_n}-H_{B_n}y_{n-1},\quad
 y_n=y_{n-1}+hv_n,
\]
with identical initial position and velocity, $x_0=y_0$ inside the
open ball of radius $R$. The clock is the integer update number and
$\tau=\inf\{n\geq1:\|x_n\|_2\geq R\}$.
Assume known nondecreasing functions $\eta_B:[0,R]\to[0,\infty)$ satisfy
\[
 \|g_B(w_*+z)-c_B-H_Bz\|_2\leq\eta_B(r)
 \quad(\|z\|_2\leq r\leq R).
\]
These are bounds over a whole region, not observed trajectory maxima.
Set
\[
 A_n=\begin{pmatrix}I-hH_{B_n}&h\beta I\\-H_{B_n}&\beta I\end{pmatrix},
 \quad D=\binom{hI}{I},\quad P=(I\ \ 0),\qquad
 K_{n,j}=P A_n\cdots A_{j+1}D\quad(1\leq j\leq n),
\]
where an empty product is the identity. For any computable upper bounds
$q_{n,j}\geq\|K_{n,j}\|_2$, recursively define $E_0=0$ and
\[
 E_n=\sum_{j=1}^n q_{n,j}
 \eta_{B_j}\!\left(\min\{R,\|y_{j-1}\|_2+E_{j-1}\}\right).
\]
\paragraph{Pathwise finite-window comparison.}
For every realized batch sequence and $n\leq\tau$,
$\|x_n-y_n\|_2\leq E_n$. Define
\[
 T_- =\inf\{n\geq1:\|y_n\|_2+E_n\geq R\},\qquad
 T_+ =\inf\{n\geq1:\|y_n\|_2-E_n\geq R\}.
\]
Then $T_-\leq\tau\leq T_+$, with infinity allowed. Consequently,
$\Pr(T_+\leq t)\leq\Pr(\tau\leq t)\leq\Pr(T_-\leq t)$ and
$\mathbb E[T_-\wedge J]\leq\mathbb E[\tau\wedge J]
\leq\mathbb E[T_+\wedge J]$ for every finite $J$.
No stability or independent sampling assumption is needed for these
finite-window statements; useful widths are a separate issue.

\paragraph{Proof.}
Subtracting the recursions gives the exact variation-of-constants sum
$x_n-y_n=-\sum_{j=1}^n K_{n,j}r_{B_j}(x_{j-1})$.
For $n\leq\tau$, each input $x_{j-1}$ is inside the ball; induction bounds
its norm by $\min\{R,\|y_{j-1}\|_2+E_{j-1}\}$ and proves the tube,
including the update causing exit. If $\tau<T_-$, the triangle inequality
contradicts exit at $\tau$. If $T_+<\tau$, the reverse triangle inequality
contradicts survival at $T_+$. This proves the ordering and its bounded
expectation consequences. The products depend on later batches, so this
proof does not treat them as predictable martingale weights.

\paragraph{Computable logistic constants and the common-curvature control.}
For ridge logistic gradients
$g_i(w)=(\sigma(a_i^Tw)-b_i)a_i+\lambda w$, Taylor's theorem and
$\sup_z|\sigma''(z)|=1/(6\sqrt3)$ give
\[
 \eta_B(r)=\frac{r^2}{12\sqrt3|B|}\sum_{i\in B}\|a_i\|_2^3.
\]
If the reference instead uses the common matrix $\bar H$, the same
comparison applies with
$\eta_B^{\rm common}(r)=\|H_B-\bar H\|_2r+\eta_B(r)$.
Thus retaining batch curvature removes the first-order heterogeneity
term from this particular remainder bound. It does not imply that every
heterogeneous-reference interval is narrower: the propagators also change.
The matrices need not commute. For history-dependent sampling the reference
is a coupled comparison on actual realized batches, not an autonomous
simulator with independently generated histories. IID sampling and fresh
epoch permutations can both be generated in advance and replayed exactly.

\paragraph{Scope.}
The event here is exit from a full Euclidean ball in the controlled
finite-pool model, distinct from the historical projected probe event.
The construction uses standard perturbation analysis; its research value
must be established by useful computable widths beyond common-curvature
comparisons. Monte Carlo summaries of these pathwise brackets are not
population confidence bounds. Full MFPT inference still requires separately
justified tail control; numerical matrix norms in ordinary floating point
are not outward-rounded certificates (the fresh-pool brackets are
certified separately in Appendix~\ref{sec:fresh-pool}).

\paragraph{Initialization scope of the curvature experiments.}
The historical logistic pools used for the curvature and short-memory
comparisons were recentered by a ridge-logistic Newton solve and started
with zero momentum. They are not the unrecentered stochastic-training
checkpoints from which those pools originally arose. New continuation
streams on these pools establish conditional replication, not validation
across newly trained runs or nonzero training velocities. Appendix~\ref{sec:fresh-pool} separately evaluates newly trained complete
states on fresh pools.

\paragraph{Finite-pool numerical validation.}
After a separate 64-step pilot, we fixed eight other historical logistic
pools, two samplers (iid and fresh epoch reshuffling), 16 driver sequences
per pool and sampler, three previously specified radii, and a 200-update
window. There are 256 driver sequences and 48 dependent radius cells,
not 48 independent training runs. Both references use identical batches,
initial states, Euclidean-ball events, and adaptive remainder bounds.
The common-Hessian control uses the extra curvature-mismatch term above.
No observed trajectory residual is used to construct either bound.
Across all cells the median absolute RMST-width reduction is 21.781 updates. Among the 45 cells with a nonzero common-reference width, the median relative reduction is 84.5\%. There are 45 narrower, 3 tied, and 0 wider cells. Mean widths across cells are 35.306 updates for the common reference and 7.810 for the batch reference. Mean paired absolute errors of the capped exit-time point predictions are respectively 0.474 and 0.001 updates. These summaries concern numerical enclosure sharpness on fixed
pools, not confidence coverage across training runs or a complete exit mean.
The saved per-cell Monte Carlo standard errors quantify the attained
conditional precision. No stopped tube check failed at the stated
floating-point tolerance. The comparison establishes neither a universal
advantage over other probabilistic bounds nor a high-dimensional neural
implementation.

\subsection{Reading a precision guarantee from an exit bracket}
\label{sec:reference-precision-gate}
The existing bracket has a direct reference-only precision interpretation.
Let $\phi(t)=\min(t,J+1)$, including $\phi(\infty)=J+1$, and put
$\underline T=\phi(T_-)$, $\overline T=\phi(T_+)$ and
$\hat\tau=\phi(\tau_y)$, where $\tau_y$ is the integer exit time of the
reference. Since $\underline T\leq\hat\tau\leq\overline T$ and
$\underline T\leq\phi(\tau)\leq\overline T$,
\[
 |\phi(\tau)-\hat\tau|\leq D:=\max\{\hat\tau-\underline T,\overline T-\hat\tau\},\qquad
 (\overline T-\underline T)/2\leq D\leq\overline T-\underline T.
\]
Thus $D\leq\delta$ is equivalent to
$[\underline T,\overline T]\subseteq[\hat\tau-\delta,\hat\tau+\delta]$. This is a restatement of the
bracket, not an additional approximation theorem or a cheaper screening
method. It uses the reference and regional envelope without evaluating
the actual path, but still requires their full construction and the
realized batches. Under history-dependent sampling, obtaining those
batches can itself require the actual trajectory.
For $\delta=0$, $\hat\tau=J+1$ guarantees survival through $J$ only; for
positive tolerance near $J$, exit status need not be resolved.

An exploratory application to the same saved trained-state runs yields
zero-tolerance acceptance fractions 0.8132 and 0.4518 for the batch and
common references. The corresponding observed exact-code agreement
fractions are 0.9935 and 0.9635. Fractions weight runs equally.
Every accepted saved case satisfies its error requirement, as expected
from bracket containment; this is not an independent coverage experiment.
Unresolved cases need not have inaccurate predictions. All tolerances
$0,1,5,20$ and both initializations are retained in the saved results.
Ordinary floating-point envelope calculations do not constitute
outward-rounded numerical certificates.

\subsection{A reference-only radius--time ambiguity map}
\label{sec:radius-window-map}
\paragraph{A radius--time view of ambiguity.}
The same tube can expose sensitivity to the chosen event without
rerunning the actual process. Suppose it is valid through exit from a
maximal ball of radius $R_*$, with both positions initially at its
center. For $n\le J$, define
\begin{align*}
 a_n&=\max\{0,\max_{k\le n}(\|y_k\|-E_k)\},\\
 b_n&=\max_{k\le n}(\|y_k\|+E_k).
\end{align*}
For every $0<R\le R_*$ the smaller-ball exit satisfies
\[
 \mathbf1_{\{a_n\ge R\}}\le
 \mathbf1_{\{\tau(R)\le n\}}\le
 \mathbf1_{\{b_n\ge R\}}.
\]
For the lower inequality, survival in the smaller ball would keep the
tube valid at all earlier updates, contradicting $a_n\ge R$.
For the upper inequality, use the tube at the actual exit-causing
update. No control of the actual path after exit is needed. The
uncertain radii lie in $a_n<R\le b_n$; over a specified range
$[r_0,R_*]$ their total length is
$[\min(b_n,R_*)-\max(a_n,r_0)]_+$. This separates a geometric source of
exit uncertainty from the magnitude of the process error. It also
explains why changing a radius after observing outcomes would change
the question rather than validate the original prediction. We keep
radii fixed for the confirmatory comparisons.

This elementary interval identity describes boundary sensitivity; it
is not a new exit comparison theorem. Averaging the indicators over
saved reference paths gives the empirical CDF bracket gap $q_n(R)$.
Sweeping their endpoints computes the radius fraction for which
$\max_{1\leq n\leq J}q_n(R)$ does not exceed a chosen descriptive threshold.
Zero empirical ambiguity does not establish zero population error.
The use of a maximal-radius envelope may be conservative for smaller
balls, and the map still requires the full reference/envelope cost.
It changes neither the originally specified event nor its clock.

This post-hoc diagnostic is applied to every saved run, initialization
and sampling law. The analyzed radius interval is the original smallest
to largest radius for that momentum setting; the envelope is always
the one computed for the largest radius. With the trained full initial
state, the mean radius fraction having zero empirical ambiguity at
every one of the 200 observation times is
0.3370 for batch curvature and 0.0374 for common curvature.
The corresponding fractions allowing at most one of eight paths
to remain ambiguous are
0.7376 and 0.2482.
These are equal-run, equal-law descriptive averages, not population
precision guarantees. The batch zero-ambiguity fraction improves in
28 runs and ties in 4; its run means range from 0.0000 to 0.8193.
In particular, useful simultaneous regions are absent in some runs.
All stationary-initialization results are retained in the reproducibility
output. This reuses the existing experiment and brackets; it is neither
an independent confirmation nor an additional methodological contribution.

\paragraph{Where the enclosure loses precision.}
A post-hoc audit of the same 256 historical-pool drivers separates
propagated curvature and nonlinear terms using the identical adaptive
radii in the common-curvature bound. At the actual exit capped at 200,
the median curvature share of that bound is 0.7465. This is an additive
accounting identity for the bound, not a causal attribution of actual
trajectory error; removing a term changes the adaptive radii.
For diagnosis only, replace the envelope at each update by the observed
norm of the actual-reference error along the saved full trajectory.
The resulting average capped exit-bracket widths are 3.4753 updates
for the common-curvature reference and 0.1914 for the sample-curvature
reference, compared with computable widths 35.3060 and 7.8099.
These oracle brackets use the target trajectories and are not prediction
methods or uniform residual estimates. Each mean averages three radii
within each driver, then the 256 drivers; the radii do not constitute
independent observations. The audit motivates sharper region-wide
nonlinear constants rather than treating more simulation as a remedy
for loose deterministic envelopes.

\paragraph{A region-dependent logistic remainder.}\label{sec:local-logistic}
For $z_i=a_i^Tw_*$ and $s_i=\|a_i\|_2$, define
\[
 m_i(r)=\sup_{z\in[z_i-s_ir,z_i+s_ir]}|\sigma''(z)|,
 \qquad \eta_B^{\rm loc}(r)=\frac{r^2}{2|B|}\sum_{i\in B}s_i^3m_i(r).
\]
This is a computable nondecreasing whole-ball remainder bound, and
$\eta_B^{\rm loc}(r)\leq r^2\sum_{i\in B}s_i^3/(12\sqrt3|B|)$.
To compute $m_i$, evaluate $|\sigma''|$ at the interval endpoints and
at each of $\pm\log(2+\sqrt3)$ lying in the interval; the latter values
are $1/(6\sqrt3)$. These are the only nonzero interior extrema:
writing $p=\sigma(z)$ gives $\sigma''(z)=p(1-p)(1-2p)$ and stationary
values at $p=(3\pm\sqrt3)/6$. The remaining cusp of the absolute value
at $z=0$ is a minimum. For every $\|x\|\leq r$, Taylor's integral
remainder along $w_*+t x$ has gradient norm at most
$s_i^3m_i(r)r^2/2$. Averaging proves the bound; the ridge term has no
Taylor remainder. Interval inclusion proves monotonicity of $m_i$.
This is a refinement of a standard Taylor bound, not a new
concentration or approximation theorem.

Replacing the global remainder in the same recursive envelope by
$\eta_B^{\rm loc}$ cannot increase that envelope: induction uses its
monotonicity, its pointwise domination by the global bound, and the
nonnegative impulse-response norms. This applies also to the
common-curvature reference when its unchanged linear mismatch term is
added to both remainder functions. The corresponding capped exit
brackets are therefore nested in exact arithmetic. The reference,
initialization, sampling law, clock and boundary are unchanged.
\paragraph{Saved-driver evaluation.}
On all 256 drivers and 48 settings of the historical-pool validation,
the region-dependent bound reduces the mean sample-curvature bracket
width from 7.8099 to 4.2786 updates (45.2\%); 31 settings improve and 17 tie.
For the common-curvature reference, the corresponding widths are
35.3060 and 30.4010, with 35 improvements and 13 ties. No setting worsens.
This is a post-hoc paired analysis, not an independent confirmation.
Means average the three radii within each driver before averaging drivers.
The two envelope loops took 14.947 seconds (local) and 0.943 seconds
(global), a 15.85-fold cost ratio; these single-pass timings alternate
execution order and exclude the shared reference and propagator computation.
They are not an end-to-end benchmark or a speedup over direct simulation.
The maximum observed stopped-tube discrepancy was below $1.4\times10^{-17}$.
All calculations use ordinary floating point rather than outward rounding.

\paragraph{Precomputing the ball constant.}
For each prescribed exit radius $R$, define
\[
 \eta_B^{\rm pre}(r)=\frac{r^2}{2|B|}\sum_{i\in B}\|a_i\|^3m_i(R),
 \qquad 0\leq r\leq R.
\]
Here $m_i$ is the interval maximum defined above. Interval inclusion yields
$\eta_B^{\rm loc}(r)\leq\eta_B^{\rm pre}(r)\leq\eta_B^{\rm global}(r)$.
All three are nondecreasing, so the same induction nests their
adaptive envelopes and their capped exit brackets. Precomputation
therefore trades width for fewer sigmoid evaluations without changing
the reference process or event. This is an implementation refinement
of the existing Taylor bound.
On the same saved-driver grid, the precomputed sample-curvature
bracket has mean width 4.7487 updates, retaining 86.7\% of the width
reduction from the adaptive local bound. For the common-curvature
reference these quantities are 30.5794 and 96.4\%. Fixed-coefficient
construction and envelope propagation together took 3.723 seconds,
versus 0.687 for the global-envelope loop in this run. These timings
exclude reference and propagator construction, and the adaptive timing
from the preceding run is not a matched repeated benchmark. All results
remain post-hoc analyses of the same historical drivers. We retain these
two implementations as width--cost choices, without treating successive
constant refinements as separate scientific contributions.

\section{Computable exit exclusion from the sampling schedule}
\label{sec:sampler-reach}
Random reshuffling constrains the support of a whole epoch, not merely
the conditional variance of its next update. We exploit this constraint
to compute the largest reachable position of a scalar affine momentum
reference, including observations inside epochs. This is a support
calculation using standard linear reachability and the rearrangement
inequality, rather than a new concentration inequality.

Let $a_1,\ldots,a_N\in\mathbb R$, $\sum_i a_i=0$, and consider
$g_i(x)=\lambda x-a_i$, $\lambda>0$, with batch size one. From
$x_0=u_0=0$, use $u_{n+1}=\beta u_n-g_{i_{n+1}}(x_n)$ and
$x_{n+1}=x_n+hu_{n+1}$. The clock is the integer update count,
$0\le\beta<1$, $0<h\lambda<2(1+\beta)$, and every epoch uses an
independent uniform permutation of all $N$ indices, with no dropped
samples. Let $e=(1,0)$ and
\[
 A=\begin{pmatrix}1-h\lambda&h\beta\\-\lambda&\beta\end{pmatrix},
 \quad D=\binom h1,\quad C=A^N,\quad
 b_\pi=\sum_{j=1}^N A^{N-j}D a_{\pi_j},\quad
 d_{\pi,l}=\sum_{j=1}^l A^{l-j}D a_{\pi_j}.
\]
The displayed stability range is the Schur range of $A$.
For a row $r$, define $H(r)=\max_\pi rb_\pi$ and, for
$s\in\{-1,1\}$, $P_{s,l}=\max_\pi se d_{\pi,l}$.
For $l=0$, take $d_{\pi,0}=0$. All maxima are over permutations,
including repeated-valued atoms with their multiplicities.

\begin{theorem}[Phase-dependent reachable radius]\label{thm:reach}
The supremum of $|x_n|$ over every finite time and every allowed
permutation sequence equals
\begin{equation}\label{eq:reach-radius}
 R_{\rm RR}=\max_{\substack{s\in\{-1,1\}\\0\le l<N}}
 \left\{P_{s,l}+\sum_{k=0}^\infty H(seA^lC^k)\right\}.
\end{equation}
Each summand is nonnegative. Each permutation maximum is computable
by sorting: if $w_j=rA^{N-j}D$, then
$H(r)=\sum_j w_{(j)}a_{(j)}$, with both sequences sorted increasingly.
For $P_{s,l}$, sort the $N$ weights consisting of
$seA^{l-j}D$, $1\le j\le l$, and $N-l$ zeros.
If $R>R_{\rm RR}$, the integer exit time from $(-R,R)$ is infinite
on every path. If $0<R<R_{\rm RR}$, its complete mean is finite.
Thus the same computable radius separates finite and infinite complete
means away from equality. No conclusion at equality is asserted.
\end{theorem}
\begin{proof}
At phase $l$ of epoch $m$, unrolling gives
\[
 x_{mN+l}=eA^l\sum_{k=0}^{m-1}C^k b_{\pi_{m-1-k}}
                 +ed_{\pi_m,l}.
\]
For fixed $m,l,s$, permutations in separate epochs can be chosen
independently, so the maximum of $sx_{mN+l}$ is exactly
$P_{s,l}+\sum_{k<m}H(seA^lC^k)$.
The average of each epoch vector and each prefix vector over uniform
permutations is zero because $\sum_i a_i=0$. Their support maxima
are therefore nonnegative. Stability implies absolute summability,
so taking the supremum over $m$ proves \eqref{eq:reach-radius}.
An exchange of two inverted atom assignments increases their weighted
sum by $(w_j-w_i)(a_{\pi_i}-a_{\pi_j})\ge0$, proving the sorting
formula. To prove the finite-mean assertion, fix a phase and sign with
limiting support larger than $R$, and choose $\epsilon>0$ smaller than
one third of this gap. At every surviving epoch boundary,
$|x|<R$ and $|u|<2R/h$; the latter follows from two consecutive
surviving positions, with $u_0=0$ handled separately. Stability and
monotone convergence of the support sums allow an integer $m$ for which
the zero-state maximum at $mN+l$ exceeds $R+2\epsilon$, while
\[
 |eA^l C^m (x,u)^\top|\le\epsilon
 \quad (|x|\le R,\ |u|\le2R/h).
\]
Choose permutations realizing that maximum. From every surviving
epoch-boundary state their endpoint exceeds $R$ in the selected
direction, unless exit has occurred earlier. This word has conditional
probability at least $p=(N!)^{-(m+1)}$. Repeating blocks of
$L=(m+1)N$ updates, aligned to epochs, gives
$\Pp(\tau>kL)\le(1-p)^k$ and $\E\tau\le L/p$.
The bound is finite, although it can be extremely loose.
The opposite strict inequality excludes every exit, including
boundary equality at any finite observation.
\end{proof}

\paragraph{A finite, verifiable truncation.}
Set $B_* = \max_\pi\|b_\pi\|_\infty$. This requires only four
sorting calculations, one for each signed coordinate.
Choose an integer $q$ such that $\theta=\|C^q\|_\infty<1$.
The omitted sum after $K$ terms is at most
\begin{equation}\label{eq:reach-tail}
 B_*\|seA^lC^K\|_1
 \frac{\sum_{j=0}^{q-1}\|C^j\|_\infty}{1-\theta}.
\end{equation}
Indeed, write $k=K+tq+j$ and apply the induced-norm inequality.
Schur stability guarantees existence of $q$. Partial sums give lower
bounds. Thus the support threshold can be bracketed to arbitrary
accuracy without enumerating $N!$ permutations. Rational inputs allow
all operations in this bracket to use rational arithmetic.

\begin{theorem}[Nonlinear exit exclusion and finite-window comparison]
\label{thm:reach-nonlinear}
Replace the gradients by $g_i(x)=\lambda x-a_i+r_i(x)$, where
$|r_i(x)|\le\eta_R$ for every $i$ and $|x|<R$. Use the identical
initial state and index sequence in the affine reference $y_n$.
For $\tau_x=\inf\{n\ge1:|x_n|\ge R\}$, put
\[
 E_n=\eta_R\sum_{j=0}^{n-1}|eA^jD|,
 \qquad G=\sum_{j=0}^\infty|eA^jD|<\infty.
\]
Then $|x_n-y_n|\le E_n$ for all $n\le\tau_x$.
For any finite $J$, if
$\tau_\pm=\inf\{n\ge1:|y_n|\ge R\pm E_n\}$, then
\[
 \tau_-\wedge J\le\tau_x\wedge J\le\tau_+\wedge J.
\]
Consequently the corresponding CDF and RMST bounds hold with no
exceptional-event penalty. In particular, under reshuffling,
$R_{\rm RR}+\eta_R G<R$ implies $\tau_x=\infty$ on every path.
The finite-window statement holds for any shared sampling sequence;
only the support radius uses the reshuffling law.
\end{theorem}
\begin{proof}
Subtract the two recursions and stop the forcing before exit:
$z_n=-\sum_{j=0}^{n-1}A^{n-1-j}D
\mathbf1_{\{j<\tau_x\}}r_{i_{j+1}}(x_j)$.
Induction gives $ez_n=x_n-y_n$ through and including $n=\tau_x$,
because the exit-causing update evaluates its gradient at a surviving
position. The convolution bound gives $E_n$. At the first actual
exit the reference has reached the eroded boundary; at an expanded
reference exit before actual exit the actual position must already
be outside, a contradiction. This proves the capped ordering even
if the eroded boundary is nonpositive. If the all-time strict
inequality holds, a first actual exit would satisfy
$|x_{\tau_x}|\le R_{\rm RR}+\eta_R G<R$, again a contradiction.
\end{proof}

\paragraph{A nonlinear finite-data separation.}
Take two nonnegative strongly convex losses
\[
 \ell_\pm(x)=\tfrac12(x\mp1)^2+\tfrac\kappa4x^4,
 \quad h=\tfrac1{10},\quad\beta=\tfrac12,
 \quad R=\tfrac12,\quad 0\le\kappa\le1.
\]
The exact-arithmetic calculation accompanying this section establishes
$R_{\rm RR}<0.178$ and $G<1.002$. Since
$\eta_R=\kappa/8$, the reshuffled iterates never exit for this entire
nonlinear family. The model has a restoring force and a nonconstant
strongly convex empirical objective; the preceding flat-loss example
does not provide this conclusion.

For iid sampling from the very same two losses, the complete mean
is finite. On a surviving history, $|u_n|<2R/h=10$ (and $u_0=0$).
If the next $M$ atoms all have sign $+$, their negative gradient is
at least $\delta=1-R-\kappa R^3\ge3/8$ until exit. Unless exit
has already occurred, the displacement after $M$ steps is at least
\[
 \frac{h\delta M}{1-\beta}
 -\frac{h\beta(1-\beta^M)}{1-\beta}
       \left(10+\frac\delta{1-\beta}\right).
\]
For $M=28$ and $\delta=3/8$ this exceeds $2R=1$.
Thus every surviving history has conditional exit probability at
least $2^{-28}$ in the next 28 updates, and
$\E\tau_{\rm iid}\le28\,2^{28}<\infty$.
This loose upper bound establishes integrability, not an accurate
mean prediction. Both algorithms start with the same complete state,
use the same integer exit event, and have the same finite pool.
Their conditional gradient laws differ; no covariance matching is claimed.

\paragraph{What the computation adds.}
The support calculation preserves epoch phase and every within-epoch
observation; observing only epoch endpoints can miss the largest
excursion. The residual term then gives both finite-window comparisons
and a test for when a proposed complete-mean model is inappropriate.
The algorithm is restricted to a common scalar affine curvature, with
uniformly bounded local residuals. General high-dimensional Hessian
heterogeneity and neural-network constants require additional work.
Linear invariant-set sums are classical \cite{rakovic2005}; shuffling
effects on momentum variance are also established \cite{domingo2022}.
The research target here is computable first-exit exclusion and its
nonlinear robustness, rather than the discovery that shuffling changes
optimization dynamics.

\paragraph{Exploratory implementation check.}
The frozen correctness pilot uses 256 paths per sampler and nonlinearity,
with the same driver reused across nonlinearities and radii. At $R=1/2$
and horizon 200, the following are empirical capped means and pathwise
reference-bracket averages, not confidence intervals or complete means
(Table~\ref{tab:reach-pilot}). All 18 sampler--nonlinearity--radius cells
and raw paths are retained.
\begin{table}[h]
\centering
\caption{Exploratory implementation check at $R=1/2$ and horizon 200:
empirical capped means and average pathwise reference brackets (256
paths per sampler and nonlinearity).}
\label{tab:reach-pilot}
\small
\begin{tabular}{@{}lrrrrr@{}}
\toprule
 & & & & \multicolumn{2}{c}{Mean bracket}\\
\cmidrule(l){5-6}
Sampler & $\kappa$ & Exits & RMST & lower & upper\\
\midrule
iid & 0 & 256 & 22.656 & 22.656 & 22.656\\
 & 0.1 & 256 & 22.848 & 20.469 & 24.273\\
 & 1 & 256 & 24.316 & 13.480 & 51.172\\
\addlinespace[2pt]
RR & 0 & 0 & 200.000 & 200.000 & 200.000\\
 & 0.1 & 0 & 200.000 & 200.000 & 200.000\\
 & 1 & 0 & 200.000 & 200.000 & 200.000\\
\bottomrule
\end{tabular}
\end{table}
The stronger nonlinear iid case has a wide bracket, illustrating that
valid coupling need not give precise prediction. The reshuffling
infinite-mean conclusion follows from the proof, not from zero observed
exits. Exact rational support calculations and a separate independent
AI proof review accompany the source; neither is formal verification.

\section{Fresh-pool validation from complete training states}
\label{sec:fresh-pool}
We froze a separate study of 32 new ridge-logistic training runs,
with eight runs per dataset--momentum stratum (MNIST/Fashion-MNIST,
$\beta=0.5/0.9$). Each run independently samples a 2,048-example
pool from the fixed empirical training dataset and trains for two
reshuffled epochs. Features are fixed normalized 17-dimensional pooled
images with an intercept. We use $h=0.2$, batch size 64, ridge 0.1,
and the three previously fixed radii for each momentum. Dataset examples
can be reused across runs and historical studies; the independence is in
the pool and algorithm random streams, conditional on the datasets.

Continuations use the original training pool and match the full trained
position and velocity. A second initialization uses the same pool's
Newton stationary point with zero velocity. Each initialization uses
eight paths per sampling law (iid or fresh-epoch reshuffling), sharing
batch sequences across the two initializations: 1,024 actual trajectories
but 512 distinct driving streams. The event is the first integer-time
exit from a full parameter ball, with equality counted as exit, and
all times are capped at 200. Both references use the precomputed
whole-ball remainder, and neither is selected after seeing outcomes.

The primary endpoint is run-level success from the trained state:
after averaging paths, samplers and radii within a run, the sample-curvature
bracket width is at most 20 updates and improves on the common-curvature
width by at least five updates. This evaluates the probability of success
of the fixed finite-path protocol, rather than the proportion of pools
whose exact conditional expected widths meet the thresholds.
Of 32 runs, 29 succeed. The one-sided 95\% Hoeffding lower bound is
0.6899, exceeding the prespecified one-half threshold. No runs were added
or removed. Two unsuccessful runs miss the five-update gain threshold;
one misses the 20-update width threshold.

Table~\ref{tab:trained-comparison} in the main text reports the width and point-error comparison.

These are means of run-level averages in updates. The stationary
initialization has 27 successful runs as a descriptive secondary result.
Both point predictors are already accurate; the main observed benefit
is sharper model-comparison brackets. The two starts change both position
and velocity, so their contrast does not isolate an initial-velocity effect.
The primary result is not a complete-MFPT claim or a guarantee for ReLU
networks. The archived brackets were evaluated in float64; the
recomputation below certifies them.
The sequential CPU run computations took 18.98 minutes, with maximum
per-process peak working set 188.2 MiB; these measurements exclude final
independent audit and aggregation. No GPU was used.

\paragraph{Certified recomputation.}
After review we recomputed every archived bracket (32 runs, 6,144
brackets) from the stored inputs (pool, center and velocity, realized
batch indices, radii, $h$, $\beta$, ridge) with bounds on every rounding
error. The sigmoid terms and the regional constants $m_i(R)$ are enclosed by
interval arithmetic at 80 bits. Each floating-point sum and matrix product
carries an a priori error bound. Propagator norms are certified upper
bounds: for a computed matrix $M$, $\|M\|_2\le\sqrt t$ once
$V^\top(tI-M^\top M)V$, with $V$ the computed eigenvectors, is shown
positive definite by Gershgorin's theorem on its interval enclosure.
The error of the floating-point state transitions, which enter every
propagator, is bounded by a discrete Gronwall argument, and the exact
affine reference is enclosed around its floating-point evaluation by the
same propagators. The certified envelopes exceed the archived ones by at
most a relative $3.1\times10^{-9}$, the transition and reference errors
are below $3\times10^{-11}$, and every integer bracket endpoint is
unchanged. The brackets therefore hold for the exact-arithmetic algorithm
with the stored inputs, and so does the primary result of 29 successes
(mean widths 5.112 and 30.857 updates). The recomputation is post hoc;
its code and outputs are in the supplement.

\paragraph{Where the brackets remain informative.}
Post-hoc stratification retains all runs and averages nested observations
within each run. The widest trained-state stratum is Fashion-MNIST with
$\beta=0.9$: mean sample- and common-curvature widths are 13.065 and 70.130
updates. Two primary non-successes have narrow common-curvature brackets
already and miss the absolute improvement target; the third improves
substantially but misses the absolute width target. Thus failure of the
joint practical criterion has two distinct interpretations. Small capped
widths also require care when most paths remain inside at the horizon:
they describe restricted means, not the unobserved exit tail. These
descriptive comparisons neither isolate a momentum effect (radii differ
across momentum settings) nor decompose the causes of enclosure width.

\paragraph{Retaining signed reference mismatch.}
A post-hoc control tests the role of the enclosure construction. Write
$y_n$ for the batch reference and $z_n$ for the common reference, and
set $F_n=\|y_n-z_n\|_2+E_n$, where $E_n$ is the batch-reference
envelope. Triangle inequality makes $F_n$ a valid
error envelope around $z_n$ through actual exit. This uses two reference
paths and sampled curvature, but no observed actual-trajectory residual.
The resulting bracket contains the batch-reference bracket by construction;
their ordering is not an empirical finding. On the same 32 trained-state
runs, mean widths are 30.857 for the old common bound, 7.736 for this
control, and 5.112 for the batch bound. The control closes 89.8\% of
the original mean-width gap. Thus much of that gap can be recovered
by retaining the propagated mismatch vector before taking its norm.
This descriptive calculation distinguishes bound construction from point
prediction accuracy; it is not a causal decomposition, a cheaper
low-information competitor, or a new confirmatory test.

\paragraph{Prediction error versus enclosure width.}
A further post-hoc diagnostic reconstructs integer exit times from all
saved paths, distinguishing exit at $J$ from censoring beyond $J$.
For trained initial states, the batch and common references have mean
paired capped-time errors of 0.094 and 0.544 updates, respectively.
The batch error is smaller in 24 runs, equal in 7, and larger in 1.
Mean cellwise empirical CDF sup differences are 0.0065 and 0.0312;
each cell has only eight paired paths, so these describe empirical
measures rather than precisely estimated population CDF errors.
All averages weight runs equally. Both references are already accurate
on these saved paths: the substantial enclosure-width gap should not
be read as a comparably large point-prediction gain. Observed residual
and boundary-margin diagnostics are retrospective and are not uniform
regional certificates. They do not change the frozen primary endpoint.

\section{Short-memory propagation with an explicit old-error bound}
\label{sec:shortmemory}
Use the batch-dependent matrices $A_n,D,P$ and uniform residual
functions of Section~\ref{sec:heterogeneous}. Fix an invertible matrix
$Q$ and a memory length $L\geq1$, before evaluating the comparison.
Write $\Phi_{n,j}=A_n\cdots A_{j+1}$, with $\Phi_{n,n}=I$, and set
\[
 \rho_n=\|Q A_n Q^{-1}\|_2,\qquad \gamma=\|P Q^{-1}\|_2.
\]
Initialize $Z_0=0$, $\widetilde E_0=0$. Before update $j$ store
\[
 b_j=\eta_{B_j}(\min\{R,\|y_{j-1}\|_2+\widetilde E_{j-1}\}).
\]
Recursively define
\[
 Z_n=\rho_n Z_{n-1}
 +\mathbf1_{\{n>L\}}\|Q\Phi_{n,n-L}D\|_2 b_{n-L},
 \qquad
 \widetilde E_n=\gamma Z_n+
 \sum_{j=\max(1,n-L+1)}^n\|P\Phi_{n,j}D\|_2 b_j.
\]
Then $\|x_n-y_n\|_2\leq\widetilde E_n$ through the exit-causing
update. Hence this envelope can replace $E_n$ in the finite-window
exit comparison. In particular, no discarded impulse is assigned zero
error and no contraction assumption is needed for validity.

\paragraph{Proof.}
Split the exact error sum into inputs $j\leq n-L$ and the last $L$
inputs. In the $Q$ norm, the former sum at time $n$ is bounded by
$\rho_n$ times its preceding bound plus the single impulse with index
$n-L$, which has just left the retained window. Triangle inequality
gives $Z_n$. Multiplying by $P Q^{-1}$ bounds its position component;
the recent terms remain explicit. Induction with the stopped residual
bound proves the assertion, exactly as in the full-memory comparison.
The same argument using nonnegative impulse magnitudes shows
$\widetilde E_n\geq E_n$ when both use exact spectral norms and the
same nondecreasing residual functions. This ordering is a statement
about envelopes, not a guarantee of useful inflation factors.

\paragraph{Cost and metric choice.}
Retaining only $L$ matrix impulses reduces the dense arithmetic from
$O(J^2d^3)$ to $O(JLd^3)$, including a per-step norm of the $2d$-state
matrix. Working impulse storage becomes $O(Ld^2)$; saved output has its
own cost. The old bound can grow when $\rho_n>1$. In the pilot we fix
$Q$ from the common-reference discrete Lyapunov equation
$M-\bar A^T M\bar A=I$, $M=Q^TQ$, whenever $\bar A$ is Schur stable.
This metric choice does not assert common contraction of the realized
batch matrices. Norms and the Lyapunov solve are evaluated in ordinary
floating point, not certified outward rounding.

\section{Stronger comparisons and computational limits}
\label{sec:stronger-comparisons}
\paragraph{A common-curvature probability baseline.}
On the saved iid paths we also evaluate the coordinate Bennett tube,
using five fixed risks $\delta\in\{0.001,0.01,0.05,0.1,0.2\}$.
Each RMST endpoint pays $(J-1)\delta$ before clipping to $[1,J]$.
We intersect the resulting expectation bounds with the deterministic
common-curvature bounds. This is a comparison against this particular
computable probability bound, not all possible concentration methods.
The iid constants are not substituted into reshuffled sampling.
On the 24 iid cells, the probability bounds improve none of the common-reference nominal intervals: average widths remain 39.531 updates, versus 8.172 for batch curvature. Accounting simultaneously for all 336 estimated endpoints gives a Hoeffding margin of 108.462 updates; the mean corrected widths become 167.376 and 153.477, respectively. Thus the small saved sample supports a comparison of numerical
sharpness, not a precise population expectation enclosure. Risk-grid
intersections are taken only after simultaneous MC correction when
reporting confidence bounds.

\paragraph{Separating curvature and nonlinearity.}
We use a synthetic two-dimensional finite pool with gradients
\[
 g_i^{s,\nu}(x)=c_i+[\bar H+s(H_i-\bar H)]x+\nu r_i(x),
\]
where $c_i$ are centered gradients and $r_i$ are logistic Taylor
remainders. The latter need not have zero pool mean. The grid fixes
$s\in\{0,0.5,1,2\}$ and $\nu\in\{0,1,2\}$, two momentum values,
two samplers, three ball radii, and a 200-update horizon. Sixty-four
driver streams are shared across the parameter grid. All 144 cells
are retained; the 48 quadratic cells are algebraic controls.
For the iid stress baseline, whole-ball coordinate mean, variance
and amplitude constants are derived from the Hessian dispersion and
the logistic Taylor bound; the companion proof note gives their formulas.
Among the 96 nonlinear cells, 29 have batch-reference width at most 20 updates, while 44 improve the common-reference intersection by at least 5 updates and 20\%. These are distinct criteria. At zero heterogeneity there is no such improvement. Doubling the
nonlinearity reduces the number of useful enclosures in the tested grid.
These are nominal conditional Monte Carlo summaries of synthetic losses;
they do not establish new training-run generalization. Individual
losses elsewhere in the parameter grid need not be convex.

\paragraph{Short-memory pilot.}
On one fixed saved driver per pool and sampler (16 paired drivers),
we time propagation-coefficient construction and retain all three memory
lengths. The full-memory mean width is 8.50 updates
(Table~\ref{tab:shortmemory-pilot}).
\begin{table}[h]
\centering
\caption{Short-memory pilot: median coefficient speedup and mean bracket
width (updates) for each memory length.}
\label{tab:shortmemory-pilot}
\small
\begin{tabular}{@{}rrr@{}}
\toprule
Memory $L$ & Median speedup & Mean width\\
\midrule
8 & 5.32 & 13.688\\
16 & 4.00 & 11.208\\
32 & 2.48 & 9.062\\
\bottomrule
\end{tabular}
\end{table}
Each timing is a single measurement, not an end-to-end benchmark.
Only 11 of the 48 radius cells have positive full-memory width;
the corresponding median inflation is zero, but this hides 4, 4 and 3
zero-width cells becoming positive at memory 8, 16 and 32. We therefore
report absolute mean widths as well. Repeated timings and a new-driver
validation are needed before selecting a default memory length.

\begin{figure*}[t]
\centering
\includegraphics[width=.95\textwidth]{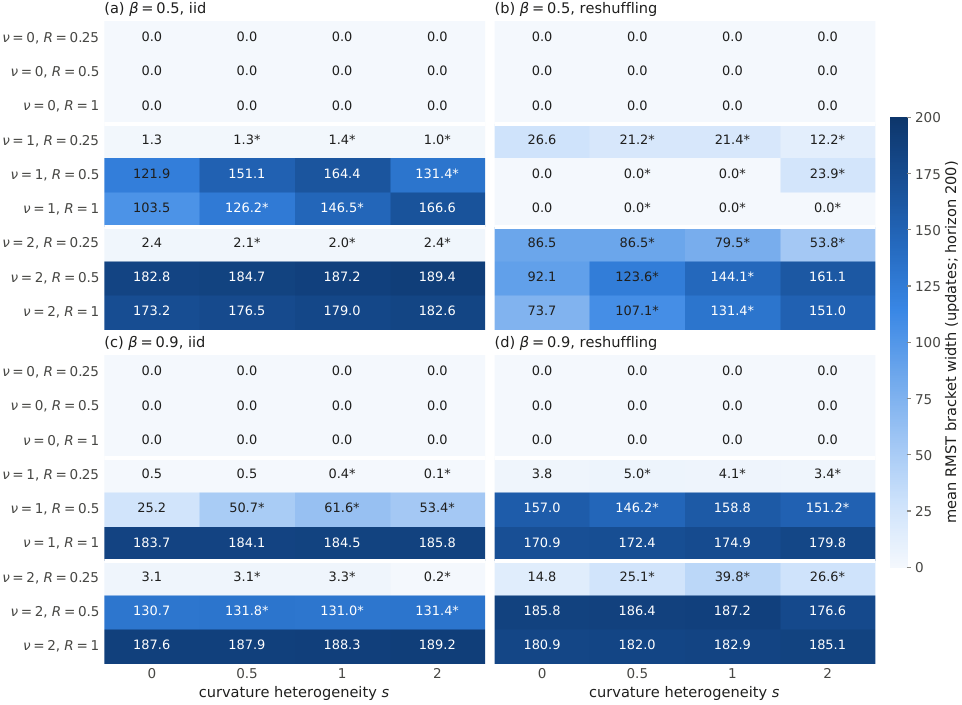}
\caption{All 144 frozen synthetic stress cells, including the quadratic
controls ($\nu=0$) and wide nonlinear enclosures. Entries are nominal
average RMST bracket widths with sampled curvature (updates, horizon 200)
over 16 drivers per cell, shared across settings. Stars mark the predefined
improvement criterion (a gain of at least 5 updates and 20\% over the
common-curvature intersection), not significance.}
\label{fig:curvature-applicability}
\end{figure*}

\section{Probabilistic reference transfer to complete means}\label{app:probtransfer}
\subsection{Stopped minibatch error}
Let
\[
u_{n+1}=\beta u_n-g_{\mathcal B_n}(w_0+X_n),\quad
X_{n+1}=X_n+h u_{n+1},\quad X_0=u_0=0.
\]
Here the batch $\mathcal B_n$ acts at $X_n$, so it is the minibatch
$B_{n+1}$ of Section~\ref{sec:process}; $B$ denotes the batch size. Every
batch contains $B$ fresh iid uniform with-replacement draws from a fixed
pool. Exit is $\tau_X=\inf\{n\geq1:\norm{X_n}\geq R\}$.
The affine reference uses the same batches, initial state and
integer clock, with gradient $g_{\mathcal B_n}(w_0)+HY_n$, where $H$ is
the deterministic symmetric mean Hessian at the center.

Fix an orthonormal eigenbasis $e_l$ of $H$ with eigenvalues $\lambda_l$,
and write
$d_i(x)=g_i(w_0+x)-g_i(w_0)$, $d(x)=\E_i d_i(x)$.
Uniformly for $\norm{x}<R$, assume
\[
 |e_l^\top[d(x)-Hx]|\leq a_l,\quad
 {\rm Var}_i(e_l^\top d_i(x))\leq v_l,\quad
 |e_l^\top[d_i(x)-d(x)]|\leq L_l.
\]
Let $p_0(\lambda)=1$, $p_1(\lambda)=1+\beta-h\lambda$ and
$p_j=(1+\beta-h\lambda)p_{j-1}-\beta p_{j-2}$.
Put $w_{j,l}=hp_j(\lambda_l)$. For $n\geq2$, define
\[
A_{n,l}=a_l\sum_{j=0}^{n-2}|w_{j,l}|,\quad
V_{n,l}=\frac{v_l}{B}\sum_{j=0}^{n-2}w_{j,l}^2,\quad
C_{n,l}=\frac{L_l}{B}\max_{j\leq n-2}|w_{j,l}|.
\]
The residual at the initial center is zero, which explains the
upper limit $n-2$ rather than $n-1$.

For $V,C>0$, define $q(V,C,\ell)$ by
\[
\frac{V}{C^2}\left[
\left(1+\frac{Cq}{V}\right)\log\left(1+\frac{Cq}{V}\right)
-\frac{Cq}{V}\right]=\ell.
\]
If $V=0$ or $C=0$, the centered sum is zero and we take $q=0$.
For window $J\geq2$, risk $\delta\in(0,1)$,
and dimension $d$, set $\ell=\log[2d(J-1)/\delta]$ and
\[
E_1=0,\qquad
E_n=\left\{\sum_{l=1}^d
[A_{n,l}+q(V_{n,l},C_{n,l},\ell)]^2\right\}^{1/2}.
\]
Here $\ell$ is the logarithmic risk factor; $l$ indexes directions.

\begin{theorem}\label{prob:thm:complete}
With probability at least $1-\delta$, the stopped auxiliary
position error is at most $E_n$ simultaneously for $n=1,\ldots,J$.
Let $\tau_\pm$ be reference exits at radii $R\pm E_n$ through $J$,
with a nonpositive eroded radius causing exit at that observation.
If the actual remaining exit mean is at most $U$ at every surviving
history, uniformly in velocity, then
\begin{align}
\E(\tau_-\wedge J)-(J-1)\delta
&\leq \E\tau_X,                                      \label{prob:eq:lower}\\
\E\tau_X
&\leq \E(\tau_+\wedge J)+(J-1)\delta
+U\min\{1,\Pp(\tau_+>J)+\delta\}.                    \label{prob:eq:upper}
\end{align}
\end{theorem}
\begin{proof}
Let $\mathcal F_k$ contain the history before batch $k$.
Define the stopped residual
\[
 r_k=\mathbf1_{\{k<\tau_X\}}
 \{d_{\mathcal B_k}(X_k)-HX_k\},
 \qquad
 Z_{n,l}=-\sum_{k=0}^{n-1}w_{n-1-k,l}e_l^\top r_k .
\]
Variation of constants gives
$e_l^\top(X_n-Y_n)=Z_{n,l}$ for every $n\leq\tau_X$,
including the exit-causing observation. After exit, $Z_{n,l}$
is an auxiliary error, not the difference of the unstopped paths.
The state and the indicator in $r_k$ are $\mathcal F_k$-measurable.
Its conditional mean is bounded by $a_l$ in direction $e_l$;
$r_0=0$. Thus the drift part of $Z_{n,l}$ is at most $A_{n,l}$.
Expose the $B$ indices individually,
holding the pre-batch state and stopping indicator fixed.
Each centered contribution is bounded by $C_{n,l}$ and the sum of
conditional variances is at most $V_{n,l}$.
For a centered increment $Z$ with $|Z|\leq C$,
\[
\E[e^{tZ}\mid\mathcal F]\leq
\exp\{(e^{tC}-1-tC)\E[Z^2\mid\mathcal F]/C^2\}.
\]
Iterating and optimizing in $t>0$ gives the Bennett exponential
bound defining $q$, for either sign. A union bound over the two
signs, $d$ coordinates and $J-1$ nontrivial times, followed by
orthonormality, proves the tube claim. Independence between error
coordinates is unnecessary; the weights must be deterministic.

On the tube event the first-exit argument gives
$\tau_-\wedge J\leq\tau_X\wedge J\leq\tau_+\wedge J$.
The exceptional event costs at most $(J-1)\delta$ for these capped
variables. The bound at observation $J$ also gives
$\Pp(\tau_X>J)\leq\Pp(\tau_+>J)+\delta$.
The residual-mean assumption is the full-history statement
$\E[\tau_X-J\mid\mathcal F_J]\leq U$ on $\{\tau_X>J\}$.
It follows, for example, from the uniform block-exit tail, which
first proves integrability from every surviving state.
Conditioning at $J$ yields
$\E[(\tau_X-J)_+]\leq U\Pp(\tau_X>J)$.
Combining the statements proves both inequalities.
\end{proof}
The conditional bound $U$ is essential. A finite-window coupling
failure probability alone cannot control the expectation of an
unbounded exit time. The martingale exponential method is standard
\citep{freedman}; we make no novelty claim for it.

\subsection{Logistic constants and the actual tail bound}
For features $a_i$, set $s_i=\norm{a_i}$ and $q_{il}=e_l^\top a_i$, as
in Appendix~\ref{sec:local-logistic}. Let $c_i$ bound the sigmoid
derivative and $m_i=m_i(R)$ its absolute second derivative on
$a_i^\top w_0+[-Rs_i,Rs_i]$; $\E_i$ averages over the pool. Valid
constants are
\begin{align*}
a_l&=\tfrac12R^2\lambda_{\max}
      \left(\E_i\,m_i|q_{il}|a_i a_i^\top\right),\\
v_l&=R^2\lambda_{\max}
      \left(\E_i\,c_i^2q_{il}^2a_i a_i^\top\right),\\
L_l&=2R\max_i c_i|q_{il}|s_i.
\end{align*}
Taylor's remainder proves the bias bound. Centering decreases the
second moment after removing the common deterministic ridge
contribution, giving the variance bound. Atom and mean increments
lie in the same symmetric directional range, giving the amplitude
bound. Thus the order-$R^2$ mean remainder is separated from the
order-$R$ random Hessian fluctuation.

The evaluated $U=M/p$ uses the block-exit construction proved in
Section~\ref{prob:sec:inputs}: at every surviving history, the
probability of exit in the next $M$ updates is at least $p>0$.
Iteration gives a geometric tail and the uniform remaining mean.
The pool, step size, momentum, radius and sampling law must match
the current setting. This is not a bound only at zero velocity.
Random reshuffling is outside the iid argument;
Appendix~\ref{app:reshuffling} shows that iid tail control does not
transfer to it.

\subsection{Inputs used by the numerical complete-mean comparison}
\label{prob:sec:inputs}
The following constructions supply the uniform upper bound and the
deterministic lower bound used below. They concern the iid
with-replacement pool of this appendix, not the neural studies.

\paragraph{A history-uniform block-exit bound.}
Fix a unit vector $v$, an integer $M\ge1$, and a deterministic matrix
$H$. Put
\[
 A=\begin{pmatrix}I-hH&h\beta I\\-H&\beta I\end{pmatrix},\quad
 D=\binom{hI}{I},\quad r_j^\top=(v^\top,0)A^{M-j}D,
 \qquad 1\le j\le M.
\]
Suppose, for every $\|x\|<R$,
\[
 |r_j^\top[d(x)-Hx]|\le D_j,\quad
 \operatorname{Var}_i(r_j^\top d_i(x))\le V_j,\quad
 |r_j^\top[d_i(x)-d(x)]|\le L_j.
\]
Write $\mathcal D=\sum_jD_j$, $\mathcal V=\sum_jV_j/B$,
$\mathcal L=\max_jL_j/B$, and, for each sign $s\in\{-1,1\}$, let
$S_s=-s\sum_jr_j^\top g_{\mathcal B_j}(w_0)$ be the unstopped frozen
forcing. For $a>0$, a valid stopped-remainder tail bound is
\[
 \Psi(a)=\min\left\{\frac{\mathcal V}{\mathcal V+a^2},
 \exp\left[-\frac{\mathcal V}{\mathcal L^2}
 h_B\left(\frac{\mathcal La}{\mathcal V}\right)\right]\right\},
 \qquad h_B(z)=(1+z)\log(1+z)-z,
\]
with $\Psi(a)=0$ when $\mathcal V=0$ or $\mathcal L=0$.
If
\[
 0<p\le\min_{s\in\{-1,1\}}
       \Pp(S_s\ge R+\mathcal D+a)-\Psi(a),
\]
then, from every interior position and every initial velocity,
$\Pp(\tau_X>kM)\le(1-p)^k$ and $\E\tau_X\le M/p$.

To prove the claim, condition on the full block-start state, set
$c=(v^\top,0)A^M(X_0^\top,u_0^\top)^\top$, and choose
$s=\operatorname{sign}(c)$, taking $s=1$ at zero. On survival through
the block, variation of constants gives
$sv^\top X_M=|c|+S_s-T-b$, where $|b|\le\mathcal D$ and $T$ is the
centered stopped remainder. Revealing the $B$ draws separately gives
conditional variance at most $\mathcal V$ and increment bound
$\mathcal L$. Conditional Cantelli and the same exponential argument
as in Theorem~\ref{prob:thm:complete} give
$\Pp(T>a\mid\mathcal F_0)\le\Psi(a)$.
The joint event $S_s\ge R+\mathcal D+a$, $T\le a$ contradicts survival.
A union bound therefore proves the exit probability $p$ without
assuming independence of $S_s$ and $T$. Repeating from each surviving
history gives the geometric tail, whose sum is at most $M/p$.

For logistic pools the constants are those in the preceding subsection,
with the unit direction replaced by $r_j$: specifically
\[
 D_j=\tfrac12R^2\lambda_{\max}(\E_i m_i|r_j^\top a_i|a_i a_i^\top),
 \quad V_j=R^2\lambda_{\max}(\E_i c_i^2(r_j^\top a_i)^2a_i a_i^\top),
 \quad L_j=2R\max_i c_i|r_j^\top a_i|\|a_i\|.
\]
These formulas use the exact mean Hessian. An approximate Hessian
$\widetilde H$ requires an additional bias allowance
$R\|r_j\|\|\widetilde H-\nabla g(w_0)\|$.

\paragraph{Computing the forcing probability.}
For fixed $s$, let $a_{ji}=-sr_j^\top g_i(w_0)$. Independence gives
the log moment-generating function
\[
 K_s(t)=B\sum_{j=1}^M\log\left(N^{-1}\sum_i e^{ta_{ji}/B}\right).
\]
Under exponential tilting, let $m_j,v_j,z_j$ be the mean, variance and
third absolute centered moment of the atom $a_{ji}$. Put
$\mu_t=\sum_jm_j$, $\sigma_t^2=\sum_jv_j/B$ and
$\epsilon_t=0.56\sum_jz_j/(B^2\sigma_t^3)$ when $\sigma_t>0$.
The independent, non-identically distributed Berry--Esseen inequality
\citep{pinelis} and change of measure on
$q<S_s\le q+z\sigma_t$ give, for $t\ge0$ and $z>0$,
\[
 \Pp(S_s\ge q)\ge e^{K_s(t)-t(q+z\sigma_t)}
 \left[\Phi\left(\frac{q-\mu_t}{\sigma_t}+z\right)
       -\Phi\left(\frac{q-\mu_t}{\sigma_t}\right)-2\epsilon_t\right]_+.
\]
Here $\Phi$ is the standard normal CDF and $q=R+\mathcal D+a$.
One may maximize these valid lower bounds separately for each sign
before taking their minimum. The finite analytical search is not
selected using continuation outcomes.

\paragraph{A convex-potential lower bound.}
For this part $X_0=u_0=0$. Let the mean loss $f$ have
$0\preceq\nabla^2 f(w_0+x)\preceq L_R I$ on the ball. Suppose the
gradient force is bounded by $K$ and the conditional gradient-noise
second moment by $v_+$. Set
\[
 b=\frac h{1-\beta},\quad c=\frac{h\beta}{1-\beta},\quad
 e=\|g(w_0)\|,\quad
 a_R=R+c\min\left(\frac K{1-\beta},\frac{2R}h\right),\quad bL_R\le2.
\]
Then a valid complete-mean lower bound is
\[
 L_{\rm deterministic}=\max\{1,R^2/C_R\},\qquad
 C_R=b^2v_++2ba_Re+2b^2L_RRe+b^2e^2,
\]
with $\E\tau_X=\infty$ if $C_R=0$.
Define $V(x)=f(w_0+x)-f(w_0)-g(w_0)^\top x$, $t(x)=\nabla V(x)$,
$Y_n=X_n+cu_n$, and
$\mathcal P_n=\|Y_n\|^2+(2bc/h)V(X_{n-1})$, with $X_{-1}=0$.
Before exit, convexity gives
$V(X_n)-V(X_{n-1})\le h u_n^\top t(X_n)$ and
$\|t(x)\|^2\le L_R x^\top t(x)$.
Using $Y_{n+1}-Y_n=-b g_{\mathcal B_n}(w_0+X_n)$, the potential
increment cancels the velocity--gradient cross term. The remaining
quadratic terms are nonpositive when $bL_R\le2$, and
$\E[\mathcal P_{n+1}-\mathcal P_n\mid\mathcal F_n]\le C_R$ before exit.
Bounded stopping gives
$\E\mathcal P_{\tau_X\wedge N}\le C_R\E(\tau_X\wedge N)$.
At exit, the outward update and $\beta\ge0$ imply $\|Y_{\tau_X}\|\ge R$;
the potential term is nonnegative. If the mean is finite, bounded
convergence therefore gives $R^2\le C_R\E\tau_X$; otherwise the lower
bound is automatic. The stopped potential is bounded because forces
are bounded before exit and its potential term uses the pre-exit position.

The 72-setting calculation uses these constructions with matching
pool, radius, step, momentum and iid sampling. Its constants, Hessians
and tail probabilities are evaluated in float64, without the rounding
certification of Appendix~\ref{sec:fresh-pool}; its reported enclosures
remain conditional numerical evaluations of the analytical bounds.

\subsection{Monte Carlo enclosures and all-setting results}
Each of 24 pool--momentum combinations receives 5,000 new affine
reference sequences through $J=200$, reused across three radii:
120,000 reference sequences and 72 settings, with no new independent
neural training. The deterministic process-error risk is $\delta=0.01$
per setting. The Monte Carlo risk is separately allocated as
$d_{\rm MC}=0.05/(3\cdot72)$.
If $m_\pm$ are sample capped means, $K$ counts survival of the
expanded reference strictly past $J$, and
\[
b=(J-1)\sqrt{\frac{\log(1/d_{\rm MC})}{2n}},\qquad
q_{\rm up}=
\begin{cases}
{\rm Beta}_{K+1,n-K}^{-1}(1-d_{\rm MC}),&K<n,\\
1,&K=n,
\end{cases}
\]
the complete mean lies in the model-conditional simultaneous
nominal 95\% enclosure
\begin{align*}
L&=\max\{1,L_{\rm deterministic},m_--b-(J-1)\delta\},\\
U_{\rm new}&=\min\{U,m_++b+(J-1)\delta
             +U\min(1,q_{\rm up}+\delta)\}.
\end{align*}
$L_{\rm deterministic}$ is the convex-potential bound in
Section~\ref{prob:sec:inputs}.
A hit exactly at $J$ is an event, not censoring. We do not
intersect this family with another 95\% family without a joint
confidence correction. The process risk and confidence risk have
different roles.

\begin{table}[ht]
\centering\small
\caption{All 72 retained iid settings, by pool and momentum. Width is
$(U_{\rm new}-L)/200$; a median above one means an interval wider
than the observation window. Counts are descriptive, not
independent training replications.}
\begin{tabular}{@{}lrrrrr@{}}
\toprule
 & & & \multicolumn{2}{c}{Upper endpoints} & \\
\cmidrule(lr){4-5}
Pool & $\beta$ & Settings & improved & below 200 & Median width\\
\midrule
Fashion-MNIST & 0.5 & 18 & 12 & 6 & 30.837\\
 & 0.9 & 18 & 12 & 0 & 557.967\\
\addlinespace[2pt]
MNIST & 0.5 & 18 & 18 & 12 & 0.693\\
 & 0.9 & 18 & 13 & 6 & 22.443\\
\bottomrule
\end{tabular}
\end{table}
All 72 erosions remain nonempty through the full window. The
final error-radius ratio $E_{200}/R$ ranges from 0.146 to 0.434,
whereas the preceding deterministic constant-envelope screen
had no nonempty erosions. All deterministic lower bounds and 55
upper bounds improve; 24 upper endpoints are below 200.
The complete upper endpoints range from 37.27 to
$1.252\times10^9$ updates. Many bounds remain dominated by the
tail term and are not practically useful.
A post-hoc algebraic sensitivity calculation attributes much of the
displayed width to the complete-tail allowance: deleting that allowance
while retaining the outer uniform-mean cap would reduce width by more
than half in 66 of 72 settings, with a median reduction of
97.7\%. This deletion is hypothetical and does not produce a valid
mean interval. The allowance combines the uniform remaining-mean bound,
expanded-reference survival uncertainty, and coupling failure; this
calculation does not isolate the effect of additional Monte Carlo samples.

All 72 intervals were recomputed from stored reference norms.
Checks include deterministic-matrix impulse reconstruction, 216
high-precision Bennett evaluations, 8,760 rational stopped-process
identity/moment checks and 20 exact complete-tail checks.
Floating-point constants and quantiles are not outward-rounded
certificates. These implementation checks do not replace the proof
or independent mathematical review.

\begin{figure}[ht]
\centering
\includegraphics[width=.64\textwidth]{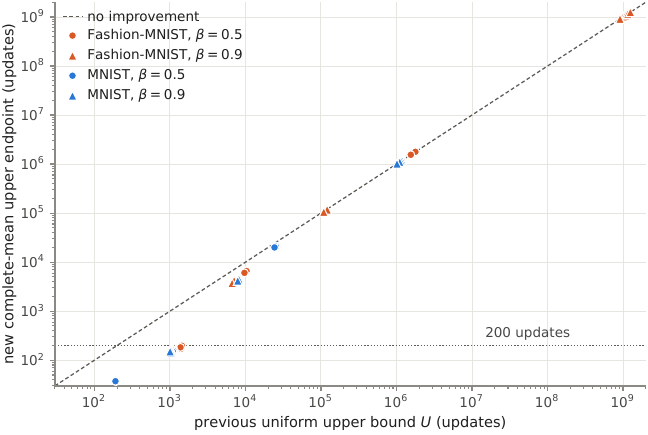}
\caption{All 72 model-conditional upper endpoints compared with the prior uniform
upper bound. The diagonal marks no improvement; the horizontal line marks 200
updates. The log scale retains the extremely loose cases rather than hiding them.}
\end{figure}

\section{Why iid tail control does not transfer to reshuffling}
\label{app:reshuffling}
This elementary example isolates the sampling law. It is distinct from
the covariance-matched obstruction in Appendix~\ref{app:obstructions}:
here the one-time gradient laws match, but the temporal laws and long-run
diffusivities do not. The broader distinction between random reshuffling
and iid SGD is established prior work \citep{mishchenko2020reshuffling}.

\begin{proposition}[Finite versus infinite complete exit means]
\label{prop:reshuffling}
Fix $h,\sigma>0$ and $0\leq\beta<1$. Let
$u_{n+1}=\beta u_n-G_{n+1}$ and $X_{n+1}=X_n+hu_{n+1}$,
with $X_0=u_0=0$. The two gradient atoms are $+\sigma$ and
$-\sigma$, corresponding to affine component losses with constant
average objective.

Under reshuffling, each epoch independently chooses a uniform permutation
of these two atoms. Under iid sampling, each update independently chooses
a uniform atom. At every deterministic update, both gradients are
uniform on $\{-\sigma,\sigma\}$. Nevertheless, if
\[
 R>\frac{h\sigma}{1-\beta^2},
\]
then the reshuffled process never exits $(-R,R)$, whereas the iid
process has finite complete exit mean.
\end{proposition}
\begin{proof}
Write $S_m\in\{-1,1\}$ for the first sign in reshuffling epoch $m$.
The updates in that epoch are $+\sigma S_m$ and $-\sigma S_m$.
The corrected position
\[
 Y_n=X_n+\frac{h\beta}{1-\beta}u_n
\]
has increment $-hG_{n+1}/(1-\beta)$, so $Y_{2m}=0$ and
$Y_{2m+1}=-h\sigma S_m/(1-\beta)$. Meanwhile,
\[
 u_{2m+2}=\beta^2u_{2m}+(1-\beta)\sigma S_m,
 \qquad |u_{2m}|\leq\frac{\sigma}{1+\beta}.
\]
Consequently
\[
 |X_{2m}|\leq\frac{h\beta\sigma}{1-\beta^2},\qquad
 X_{2m+1}=-h\sigma S_m-\frac{h\beta^2}{1-\beta}u_{2m},
\]
and $|X_{2m+1}|\leq h\sigma/(1-\beta^2)$.
These bounds hold for every permutation sequence, so the stated
reshuffling exit time is infinite almost surely.

For iid sampling, $Y_n$ is a martingale with independent centered
increments of variance
$q=h^2\sigma^2/(1-\beta)^2$. Put
$r=h\sigma/(1-\beta)^2$. The velocity satisfies
$|u_n|\leq\sigma/(1-\beta)$ before exit and on its causing step.
Thus $|Y_{\tau\wedge N}|\leq R+r$.
Bounded-time telescoping gives
\[
 q\,\E(\tau\wedge N)=\E Y_{\tau\wedge N}^2\leq(R+r)^2.
\]
Monotone convergence proves
$\E\tau\leq(R+r)^2/q<\infty$ without assuming integrability in advance.
\end{proof}

No neural basin, state-dependent gradient field, or matched diffusivity
is asserted in this example. The iid process has diffusivity $q/2>0$;
the bounded reshuffled positions have zero long-run diffusivity.
The purpose is to rule out replacing an iid full-history conditional
assumption by a one-time marginal statement. In particular, an iid
uniform residual-mean bound cannot be applied to reshuffling merely
because a frozen one-time gradient histogram agrees.

\section{Sampler dependence in trained networks: design and full results}\label{app:nn-sampler}
This appendix documents the eight studies summarized in
Section~\ref{sec:networks}. Each was declared in a time-stamped,
hashed file before its continuation runs. Amendments were also recorded
before the outputs they concern existed, and all are stated below.
The original decision rules are retained as historical protocol records.
Their 3--4-seed bootstrap ranges and labels are descriptive only; the
revision makes no population significance claim from these studies.
Only the separately specified new joint endpoints carry corrected
confirmatory tests (Appendices~\ref{app:replication-followup}
and~\ref{app:mechanism-followup}).

\subsection{Networks, training and continuation}
\paragraph{Networks.} Four configurations vary architecture and data:
\begin{itemize}
\item an MLP (784--256--256--10, ReLU) on MNIST \citep{lecun1998},
trained for 10 epochs with learning rate 0.05;
\item a CNN (two $3\times3$ convolutions with 32 and 64 channels, each
followed by max-pooling, and a 128-unit hidden layer) on Fashion-MNIST
\citep{xiao2017}, 15 epochs, learning rate 0.05;
\item ResNet-18 \citep{he2016deep} with a $3\times3$ stem and no initial
max-pooling on CIFAR-10 \citep{krizhevsky2009learning}, 30 epochs,
learning rate 0.1, random crops and horizontal flips;
\item a ViT \citep{dosovitskiy2021image} with $4\times4$ patches, width
192, depth 6, three heads and MLP width 384 on CIFAR-10, 60 epochs,
learning rate 0.05 with a five-epoch warm-up and the same augmentation.
\end{itemize}
Training uses momentum SGD ($\beta=0.9$), batch size 128, weight decay
$5\times10^{-4}$, cosine decay and bfloat16 autocast. Seed 0 is the
pilot, seeds 1--4 form Study 1 and seeds 5--8 form Study 2. Test
accuracies across these eight training seeds range from 98.28--98.50\%
(MLP), 91.65--92.24\% (CNN), 92.56--93.61\% (ResNet-18), and
78.29--79.82\% (ViT).

\paragraph{Common start.} For each trained network, a pool of $N=4096$
training examples is drawn without replacement and without augmentation.
The network is put in evaluation mode, so BatchNorm statistics are frozen
and the minibatch loss is the mean of per-example losses. Every
sampler below has the same marginal batch law, a uniformly random
64-subset of the pool; equality of conditional laws along training
is not assumed. A burn-in of 100 iid momentum-SGD steps ($b=64$, $\beta=0.9$,
weight decay $5\times10^{-4}$, continuation learning rate) gives the full
start state $(\theta_0,u_0)$, shared by all samplers. The burn-in
changes the batch size and learning rate from those at the end of
training, so exits measure displacement under continued training from
this start, including any drift that follows the change; every sampler
shares the start and the drift. Studies 5 and 6 repeat the burn-in with
other momenta, pools and batch sizes, and Study 7 lengthens it tenfold.

\paragraph{Continuation.} From the common start, 32 paths per sampler run
for $J=1000$ updates in single precision, each with its own seed; every
sampler uses the same arithmetic, but its effect can depend on the
dynamics. Float64 reruns preserve the checked ratio point estimates
to the reported tolerance
(Appendix~\ref{app:nn-precision}). The
samplers are those of Section~\ref{sec:process}: iid, RR (a fresh
permutation of the pool split into 64 consecutive batches per epoch) and
persistent sampling with probability $p$ of drawing a fresh batch, that
is, mean echo count $1/p$. Study 1 uses $p\in\{0.25,0.05\}$ and Study 2
uses $p\in\{0.5,0.25,0.1,0.05,0.02\}$; iid is the case $p=1$. Study 3
reruns the Study 2 starts at smaller learning rates (below).

\paragraph{Directions and events.} The direction $e_{\rm noise}$ is an
approximate top eigenvector of the per-example gradient covariance on the
pool at $\theta_0$, from 30 power iterations with forward-mode
Jacobian--vector products. The direction $e_{\rm rand}$ is a uniformly
random unit vector. The exit time is the first $n\le J$ with
$\|\theta_n-\theta_0\|_2\ge R$ (ball) or
$|e^\top(\theta_n-\theta_0)|\ge r$ (projection), censored at $J$.
Non-finite parameters would count as an exit. None of the 1,920 Study 1
paths became non-finite. In Study 2, 4 of 3,360 paths (ResNet-18,
persistent sampling with $p=0.1$ or $0.05$) became non-finite, each
after all of its recorded exits.

\subsection{Pilot}
A pilot on seed 0 fixed two things only. The continuation learning rate
is the largest of two candidates whose burn-in is finite and whose 24
pilot paths do not diverge: 0.05 (MLP, CNN) and 0.01 (ResNet-18, ViT).
The radii $r_1$ and $r_2$ are the medians, over the iid pilot paths, of
the running maxima of each statistic at steps 200 and 600. Sampler
comparisons in the pilot were not examined. All later studies of these
networks use these radii and, apart from the reduced steps of Study 3,
these learning rates; Study 4, on new datasets, has its own pilot
(Appendix~\ref{app:nn-studies456}).

\subsection{Study 1: declaration and amendments}
A declaration recorded before any Study 1 continuation fixed the radii,
seeds, sampler set, endpoints and hypotheses. The primary endpoint is
the ratio of each sampler's restricted mean exit time (RMST over $J$
updates) to that of iid for the ball at $r_1$. The interval is a 95\%
hierarchical bootstrap interval, resampling seeds and then paths within
seeds (10{,}000 replicates). The directional hypotheses were
\begin{itemize}
\item H1a: RR exits later than iid;
\item H1b: persistent sampling with $p=0.05$ exits earlier than iid;
\item H1c: the RMST decreases with persistence, from iid to $p=0.25$ to
$p=0.05$.
\end{itemize}
The archived rule labeled a hypothesis as supported when its interval
excludes one in the predicted direction (for H1c, both pairwise
intervals). The other radii and the two projections are secondary
endpoints.

The ResNet-18 burn-in for seed 3 diverged at the declared learning rate.
An amendment recorded before any ResNet-18 continuation output existed
analyses seeds 1, 2 and 4 without a replacement seed. The burn-in
precedes and is shared by all sampler arms, so the exclusion removes
every arm of that seed alike: it conditions on a stable start but does
not act differently on the samplers. A separate analysis
note, recorded before any formal output was loaded, fixed how surrogate
errors are averaged over seeds.

\subsection{Surrogates along \texorpdfstring{$e_{\rm noise}$}{the noise direction}}
At $\theta_0$ we compute the projected per-example gradients
$a_i=e^\top\nabla\ell_i(\theta_0)+\lambda_{\rm wd}e^\top\theta_0$, the
curvature $\kappa=e^\top(H+\lambda_{\rm wd}I)e$ of the pool loss, and the
projected start momentum. The surrogates follow the scalar recursion
$u_{n+1}=\beta u_n-(G_n+\kappa x_n)$, $x_{n+1}=x_n+hu_{n+1}$, and differ
only in the forcing $G_n$:
\begin{itemize}
\item S1: iid Gaussian with the mean and variance of a uniform 64-subset
mean of the $a_i$, that is, variance $\sigma_a^2(N-64)/\{64(N-1)\}$ with
$\sigma_a^2$ the pool variance (sampling without replacement); it is the
same for every sampler;
\item S2: the mean of the $a_i$ over the batches of the actual sampler;
\item S4 (persistent samplers): Gaussian values with the persistent
sampler's switching;
\item S3 (persistent samplers): the switching limit of
Theorem~\ref{thm:telegraph-bridge},
$\dot x=-(G+\kappa x)/(1-\beta)$ with $G$ refreshed at rate $p/h$, with
time reported in updates.
\end{itemize}
The surrogates separate sampler memory from the shape of the projected
forcing. S1 replaces the finite-pool forcing by iid Gaussian noise, S2
retains the pool and sampler, and S4 combines Gaussian values with
persistent refreshes. S3 also removes the explicit momentum state in
favor of the switching flow. S1, S2 and S4 initialize the projected
momentum; S3 has no independent velocity coordinate to match. Thus a
difference between S2 and S3 can reflect initialization and finite-step
dynamics as well as the continuum approximation. Study 3 places the
networks outside the switching regime of
Theorem~\ref{thm:telegraph-bridge}, so the large S3 errors (relative
error up to 79) mark the limit of its applicability, not a failure
within its hypotheses. The comparison diagnoses which information is
omitted; it does not identify a causal share of prediction error for
each omission, and, being one-dimensional, it cannot separate the
projection from the freezing of local information.

Each surrogate uses 20{,}000 paths. The error measure is the relative
RMST error $|{\rm RMST}_{\rm surrogate}-{\rm RMST}_{\rm actual}|/{\rm
RMST}_{\rm actual}$, averaged over seeds. The declared secondary
hypotheses were
\begin{itemize}
\item H2a: S1 has larger error under RR and persistent sampling than
under iid;
\item H2b: S2 has smaller error than S1 in at least 13 of the 24
(network, non-iid sampler, radius) cells.
\end{itemize}

\subsection{Study 1 results}
Table~\ref{tab:nn-all} reports all six declared endpoints;
Figure~\ref{fig:nn-study1} in Section~\ref{sec:networks} shows the
ratios at the first radius and the survival curves. The historical
interval rule selected H1a only for the ViT, H1b for the CNN and ViT,
and H1c for none; these are not population-level decisions.
For H2a, the Gaussian surrogate has larger relative error than under iid
in 20 of 24 non-iid cells. For H2b, the sampler-driven surrogate improves
on the iid Gaussian surrogate in 20 of 24 cells, exceeding the declared
13-cell threshold. These are cell counts, not independent replications;
the minimum seed-averaged relative error of S2 is still 1.05.
The ResNet-18 exclusion is a burn-in failure, and all its conclusions
are conditional on the three retained starts.

\begin{table*}[t]
\centering
\caption{Study 1: RMST ratios to iid with 95\% hierarchical bootstrap
ranges for every declared event; retained for descriptive comparison,
without significance highlighting. RMST$_{\rm iid}$ is in updates (maximum $1000$).
$^\dagger$Three seeds (amendment above); its endpoints shown as 1.00 equal
one exactly (Appendix~\ref{app:nn-seedlevel}).}
\label{tab:nn-all}
\small
\begin{tabular}{@{}llrlll@{}}
\toprule
 & & & \multicolumn{3}{c}{RMST ratio to iid, with 95\% interval} \\
\cmidrule(l){4-6}
Network & Event & RMST$_{\rm iid}$ & RR & $p=0.25$ & $p=0.05$ \\
\midrule
MLP / MNIST & ball $r_1$ & 160 & 0.81\ci{0.61}{1.13} & 0.61\ci{0.50}{0.71} & 1.61\ci{1.33}{1.94} \\
 & ball $r_2$ & 332 & 0.88\ci{0.63}{1.21} & 0.34\ci{0.25}{0.49} & 0.90\ci{0.65}{1.30} \\
 & $e_{\rm noise}$ $r_1$ & 531 & 0.95\ci{0.47}{1.02} & 0.70\ci{0.36}{1.19} & 1.06\ci{0.75}{3.69} \\
 & $e_{\rm noise}$ $r_2$ & 555 & 0.94\ci{0.37}{1.04} & 0.69\ci{0.37}{1.02} & 1.05\ci{0.78}{2.51} \\
 & $e_{\rm rand}$ $r_1$ & 295 & 1.31\ci{0.81}{2.12} & 0.43\ci{0.32}{0.62} & 1.05\ci{0.75}{1.55} \\
 & $e_{\rm rand}$ $r_2$ & 510 & 1.15\ci{0.88}{1.47} & 0.37\ci{0.27}{0.53} & 0.81\ci{0.56}{1.34} \\
\addlinespace[3pt]
CNN / F-MNIST & ball $r_1$ & 199 & 0.97\ci{0.93}{1.01} & 0.43\ci{0.41}{0.45} & 0.86\ci{0.77}{0.97} \\
 & ball $r_2$ & 619 & 1.05\ci{0.95}{1.17} & 0.25\ci{0.22}{0.30} & 0.51\ci{0.43}{0.62} \\
 & $e_{\rm noise}$ $r_1$ & 250 & 0.94\ci{0.54}{1.36} & 0.40\ci{0.26}{0.65} & 0.87\ci{0.57}{1.53} \\
 & $e_{\rm noise}$ $r_2$ & 325 & 1.04\ci{0.65}{1.42} & 0.33\ci{0.21}{0.51} & 0.72\ci{0.49}{1.19} \\
 & $e_{\rm rand}$ $r_1$ & 397 & 1.03\ci{0.78}{1.47} & 0.41\ci{0.30}{0.71} & 0.72\ci{0.51}{1.29} \\
 & $e_{\rm rand}$ $r_2$ & 518 & 0.95\ci{0.74}{1.18} & 0.39\ci{0.29}{0.65} & 0.69\ci{0.51}{1.14} \\
\addlinespace[3pt]
ResNet-18 / CIFAR-10$^\dagger$ & ball $r_1$ & 474 & 0.98\ci{0.89}{1.00} & 0.91\ci{0.67}{1.00} & 1.21\ci{1.00}{1.81} \\
 & ball $r_2$ & 783 & 0.96\ci{0.85}{1.04} & 0.67\ci{0.40}{1.00} & 1.00\ci{0.92}{1.08} \\
 & $e_{\rm noise}$ $r_1$ & 278 & 1.21\ci{0.86}{1.38} & 0.19\ci{0.14}{1.65} & 0.08\ci{0.06}{3.04} \\
 & $e_{\rm noise}$ $r_2$ & 291 & 1.15\ci{0.83}{1.31} & 0.18\ci{0.14}{1.56} & 0.08\ci{0.06}{3.29} \\
 & $e_{\rm rand}$ $r_1$ & 495 & 0.99\ci{0.70}{1.34} & 0.90\ci{0.52}{1.02} & 1.18\ci{1.00}{2.16} \\
 & $e_{\rm rand}$ $r_2$ & 668 & 1.03\ci{0.85}{1.28} & 0.93\ci{0.69}{1.09} & 1.26\ci{1.00}{2.03} \\
\addlinespace[3pt]
ViT / CIFAR-10 & ball $r_1$ & 204 & 1.07\ci{1.06}{1.08} & 0.41\ci{0.40}{0.42} & 0.72\ci{0.66}{0.78} \\
 & ball $r_2$ & 591 & 0.96\ci{0.96}{0.97} & 0.50\ci{0.49}{0.51} & 0.73\ci{0.70}{0.77} \\
 & $e_{\rm noise}$ $r_1$ & 33 & 0.81\ci{0.53}{1.23} & 0.52\ci{0.30}{1.51} & 0.96\ci{0.53}{2.91} \\
 & $e_{\rm noise}$ $r_2$ & 41 & 0.94\ci{0.60}{1.47} & 0.49\ci{0.27}{1.68} & 0.95\ci{0.49}{3.34} \\
 & $e_{\rm rand}$ $r_1$ & 195 & 1.26\ci{0.90}{1.68} & 0.44\ci{0.33}{0.60} & 0.77\ci{0.58}{1.00} \\
 & $e_{\rm rand}$ $r_2$ & 276 & 1.31\ci{0.89}{1.88} & 0.40\ci{0.32}{0.54} & 0.71\ci{0.52}{0.96} \\
\bottomrule
\end{tabular}

\end{table*}

\begin{table}[t]
\centering
\caption{Study 1: seed-averaged relative RMST errors of the surrogates
along $e_{\rm noise}$. Bold S2 entries are smaller than S1, the comparison
of H2b (20 of 24 non-iid cells); S3 and S4 exist for persistent samplers
only.}
\label{tab:nn-surrogates}
\small
\begin{tabular}{@{}llrrrrrrrr@{}}
\toprule
 & & \multicolumn{4}{c}{radius $r_1$} & \multicolumn{4}{c}{radius $r_2$} \\
\cmidrule(lr){3-6}\cmidrule(l){7-10}
Network & Sampler & S1 & S2 & S4 & S3 & S1 & S2 & S4 & S3 \\
\midrule
MLP / MNIST & iid & 5.38 & 5.28 & -- & -- & 3.23 & 3.20 & -- & -- \\
 & RR & 11.86 & \textbf{11.71} & -- & -- & 10.00 & \textbf{9.92} & -- & -- \\
 & $p=0.25$ & 5.68 & \textbf{1.50} & 2.38 & 4.82 & 5.12 & \textbf{1.51} & 2.41 & 4.39 \\
 & $p=0.05$ & 1.76 & \textbf{1.09} & 1.44 & 1.64 & 1.60 & \textbf{1.05} & 1.37 & 1.50 \\
\addlinespace[3pt]
CNN / F-MNIST & iid & 2.60 & 2.45 & -- & -- & 1.87 & 1.75 & -- & -- \\
 & RR & 3.60 & \textbf{3.39} & -- & -- & 2.47 & \textbf{2.28} & -- & -- \\
 & $p=0.25$ & 7.62 & \textbf{2.79} & 3.66 & 11.10 & 7.47 & \textbf{3.06} & 3.97 & 10.51 \\
 & $p=0.05$ & 3.16 & \textbf{2.75} & 3.20 & 4.50 & 2.81 & \textbf{2.52} & 2.91 & 3.91 \\
\addlinespace[3pt]
ResNet-18 / CIFAR-10 & iid & 1.75 & 1.81 & -- & -- & 1.94 & 1.99 & -- & -- \\
 & RR & 1.80 & \textbf{1.67} & -- & -- & 2.02 & \textbf{1.89} & -- & -- \\
 & $p=0.25$ & 2.68 & \textbf{2.41} & 2.13 & 35.19 & 3.01 & \textbf{2.70} & 2.53 & 42.33 \\
 & $p=0.05$ & 5.56 & 9.82 & 8.89 & 78.96 & 5.50 & 8.36 & 8.04 & 61.78 \\
\addlinespace[3pt]
ViT / CIFAR-10 & iid & 2.75 & 2.71 & -- & -- & 3.35 & 3.26 & -- & -- \\
 & RR & 3.38 & \textbf{3.20} & -- & -- & 3.49 & \textbf{3.27} & -- & -- \\
 & $p=0.25$ & 4.47 & \textbf{1.34} & 1.31 & 21.82 & 5.05 & \textbf{1.33} & 1.34 & 21.66 \\
 & $p=0.05$ & 2.20 & 3.54 & 3.58 & 22.74 & 2.37 & 3.30 & 3.38 & 18.89 \\
\bottomrule
\end{tabular}

\end{table}

\subsection{Study 2: persistence dose-response}
\paragraph{Declaration.} After Study 1 had finished for three networks, a
second declaration fixed a dose-response study on the new seeds 5--8. It
keeps the Study 1 learning rates, radii and continuation protocol and
uses seven samplers: iid ($p=1$), persistent sampling with
$p\in\{0.5,0.25,0.1,0.05,0.02\}$, and RR. The primary hypothesis H3
states an interior minimum for the ball at $r_1$: the ratios
${\rm RMST}(p_c)/{\rm RMST}({\rm iid})$ and
${\rm RMST}(p_c)/{\rm RMST}(p{=}0.02)$ both have 95\% intervals below one.
Here $p_c$ is whichever of 0.25 and 0.05 had the smaller Study 1 RMST;
this gives $p_c=0.25$ for every network (for ResNet-18 the rule was
applied once its Study 1 run ended). Secondary endpoints are the same
contrasts for the other events, the RR contrast, and the agreement of the
S2 surrogate's ordering of the samplers with the observed ordering. The
ResNet-18 burn-in diverged for seed 7. By the declared rule this seed is
excluded without replacement, as recorded in an amendment before any
Study 2 output existed. Over both studies, the eval-mode ResNet-18
burn-in at learning rate 0.01 diverged for 2 of 8 seeds.

\paragraph{Results.} The historical H3 rule selected the ball at both radii in all four
networks, and for the random projection at both radii in all four
(Table~\ref{tab:nn-study2}). Along $e_{\rm noise}$ the rule selected the MLP
and CNN but not for ResNet-18, whose $e_{\rm noise}$ exits take only a
few updates and are fastest under iid, nor for the ViT, whose interval
against iid includes one. At the first radius the ball RMST is smallest
at $p=0.5$ (MLP, ResNet-18) or $p=0.25$ (CNN, ViT). The RR ratios there are
0.84 [0.69, 1.03] (MLP), 1.00 [0.97, 1.03] (CNN), 0.93 [0.89, 0.97]
(ResNet-18) and 1.08 [1.07, 1.09] (ViT). Over sampler pairs whose
observed RMSTs differ by at least 10\%, the S2 surrogate reproduces the
observed order in 483 of 565 pairs. It locates the observed minimizing
persistence level in 14 of 30 (network, seed, radius) cells.

\begin{table*}[t]
\centering
\caption{Study 2: RMST by sampler, pooled over seeds, and the declared H3 contrasts with
95\% hierarchical bootstrap intervals. ResNet-18 uses seeds 5, 6 and 8.}
\label{tab:nn-study2}
\footnotesize
\setlength{\tabcolsep}{4.5pt}
\begin{tabular}{@{}llrrrrrrrll@{}}
\toprule
 & & \multicolumn{7}{c}{RMST (updates), maximum 1000} & \multicolumn{2}{c}{H3 contrasts, $p_c=0.25$} \\
\cmidrule(lr){3-9}\cmidrule(l){10-11}
Network & Event & iid & $p=0.5$ & $0.25$ & $0.1$ & $0.05$ & $0.02$ & RR & $p_c$/iid & $p_c$/$(p=0.02)$ \\
\midrule
MLP / MNIST & ball $r_1$ & 272 & 97 & 109 & 172 & 290 & 687 & 228 & 0.40\ci{0.21}{0.80} & 0.16\ci{0.14}{0.18} \\
 & ball $r_2$ & 422 & 114 & 126 & 198 & 340 & 774 & 362 & 0.30\ci{0.18}{0.60} & 0.16\ci{0.15}{0.18} \\
 & $e_{\rm noise}$ $r_1$ & 576 & 278 & 278 & 393 & 553 & 799 & 466 & 0.48\ci{0.23}{0.72} & 0.35\ci{0.14}{0.68} \\
 & $e_{\rm rand}$ $r_1$ & 462 & 180 & 167 & 237 & 368 & 646 & 369 & 0.36\ci{0.25}{0.56} & 0.26\ci{0.20}{0.34} \\
\addlinespace[3pt]
CNN / F-MNIST & ball $r_1$ & 209 & 94 & 87 & 119 & 185 & 383 & 209 & 0.42\ci{0.40}{0.44} & 0.23\ci{0.21}{0.25} \\
 & ball $r_2$ & 601 & 191 & 160 & 218 & 350 & 722 & 641 & 0.27\ci{0.25}{0.29} & 0.22\ci{0.21}{0.24} \\
 & $e_{\rm noise}$ $r_1$ & 235 & 73 & 86 & 118 & 230 & 502 & 217 & 0.37\ci{0.23}{0.60} & 0.17\ci{0.12}{0.26} \\
 & $e_{\rm rand}$ $r_1$ & 568 & 221 & 221 & 248 & 393 & 601 & 586 & 0.39\ci{0.29}{0.52} & 0.37\ci{0.28}{0.47} \\
\addlinespace[3pt]
ResNet-18 / CIFAR-10 & ball $r_1$ & 223 & 143 & 153 & 217 & 347 & 765 & 208 & 0.69\ci{0.65}{0.73} & 0.20\ci{0.19}{0.21} \\
 & ball $r_2$ & 674 & 305 & 293 & 409 & 677 & 997 & 625 & 0.43\ci{0.38}{0.49} & 0.29\ci{0.28}{0.31} \\
 & $e_{\rm noise}$ $r_1$ & 4 & 6 & 8 & 12 & 20 & 31 & 4 & 1.89\ci{1.23}{2.62} & 0.26\ci{0.16}{0.89} \\
 & $e_{\rm rand}$ $r_1$ & 260 & 173 & 180 & 231 & 332 & 680 & 273 & 0.69\ci{0.50}{0.99} & 0.27\ci{0.20}{0.36} \\
\addlinespace[3pt]
ViT / CIFAR-10 & ball $r_1$ & 203 & 106 & 81 & 101 & 140 & 290 & 219 & 0.40\ci{0.39}{0.41} & 0.28\ci{0.25}{0.31} \\
 & ball $r_2$ & 593 & 425 & 296 & 298 & 421 & 835 & 572 & 0.50\ci{0.49}{0.51} & 0.35\ci{0.34}{0.37} \\
 & $e_{\rm noise}$ $r_1$ & 268 & 59 & 44 & 56 & 91 & 187 & 228 & 0.17\ci{0.11}{1.22} & 0.24\ci{0.18}{0.33} \\
 & $e_{\rm rand}$ $r_1$ & 210 & 119 & 99 & 119 & 148 & 275 & 267 & 0.47\ci{0.32}{0.72} & 0.36\ci{0.28}{0.48} \\
\bottomrule
\end{tabular}

\end{table*}

\subsection{Study 3: step-size scaling}
\paragraph{Declaration.} Theorem~\ref{thm:telegraph-bridge} implies that,
in time units $t=nh$, a lazy-resampling sampler enters the exit law only
through its switching rate $p/h$ (up to $O(h)$), whereas iid sampling,
whose rate diverges as $h\to0$, averages out. After Studies 1 and 2, a
third declaration tested this scaling. It reruns the Study 2 starts
(seeds 5--8; the same $\theta_0$ and momentum buffer) of the MLP, CNN and
ViT at the learning rates $h_0/2$ and $h_0/4$, where $h_0$ is the Study 2
rate; ResNet-18 is excluded for cost. The time window is fixed at
$1000h_0$, so $J=2000$ and $4000$ updates. The switching rates
$\tilde\lambda=p\,h_0/h\in\{0.02,0.05,0.1,0.25,0.5,1\}$ are those of
Study 2, so at $h_0/2$ and $h_0/4$ the persistent samplers use
$p=\tilde\lambda/2$ and $p=\tilde\lambda/4$; iid is also run at each step.
Each (seed, step, sampler) cell has 32 paths (16 for the ViT). The
estimand is the RMST in time units, $h\,\E\min(\tau,J)/h_0$, with the
hierarchical bootstrap of Study 1 applied to ratios with shared seeds.
The declared hypotheses were
\begin{itemize}
\item H5a (averaging): for the ball at $r_1$, the time-unit RMST of iid
at $h_0/4$ exceeds that at $h_0$ (interval above one);
\item H5b (switching-rate invariance): at $\tilde\lambda=0.25$, the
time-unit RMST at $h_0/4$ is equivalent to that at $h_0/2$, with the
95\% interval of their ratio inside $[0.8,1.25]$.
\end{itemize}
The analysis code was hashed before any Study 3 output was loaded. A
note records that the runs were moved from one to three GPUs when other
users' jobs ended; the jobs and analysis were unchanged.

\paragraph{Results.} Table~\ref{tab:nn-study3} and
Figure~\ref{fig:nn-scaling} report the outcome; no Study 3 path became
non-finite. The per-arm intervals in Figure~\ref{fig:nn-scaling} are
descriptive hierarchical bootstrap intervals (seeds, then paths), added
after the declared analysis; the declared contrasts are those of the
table. The historical H5a rule selected all three networks. H5b fails for the CNN and
the ViT. For the MLP its interval lies inside the margin, but only
because 3\% and 0\% of paths leave the ball within the window at
$h_0/2$ and $h_0/4$, so both arms are censored at $1000h_0$; we do not
count this as evidence of invariance. At $h_0/2$ and $h_0/4$ no MLP or
CNN iid path leaves the ball within the window (time-unit RMST 1000),
whereas CNN paths at rates 0.25--0.5 still do (RMST 190--202 at
$h_0/2$ and 597--648 at $h_0/4$). Across matched switching rates,
the ratio of time-unit RMSTs between $h_0/2$ and $h_0$ ranges from 1.37
to 2.95 in the CNN and ViT, except at the fastest rate, where the ViT
ratio is 0.99; between $h_0/4$ and $h_0/2$ it ranges from 1.09 to 3.20.
Strong persistence is the most step-dependent: at rate 0.02 the ViT RMST
grows from 290 to 969 between $h_0$ and $h_0/4$. The minimizing rate
stays at 0.25--0.5 for the CNN and ViT.

\paragraph{Exploratory minibatch sharpness.} To locate these networks
relative to the small-step regime, we estimated at each Study 2 start of
seed 5 the top eigenvalue $\kappa$ of the loss Hessian plus weight decay,
for 16 uniformly random 64-subsets and for the whole pool (20 power
iterations with Hessian--vector products; not declared). The median
minibatch value of $h_0\kappa$ is 1.07 (MLP), 0.71 (CNN), 1.26
(ResNet-18) and 0.69 (ViT), so $h_0\kappa/(1-\beta)$ lies between 6.9
and 12.6; whole-pool values of $h_0\kappa$ are 0.28--0.62. The largest
probed minibatch value, 2.45, is below the heavy-ball stability
threshold $2(1+\beta)=3.8$, so the persistence effects are not linear
instabilities of individual batches, but the steps are far from the
regime $h\kappa/(1-\beta)\ll1$ in which the switching limit applies.

\begin{table*}[t]
\centering
\caption{Study 3: declared contrasts for the ball at $r_1$ (time-unit RMST
ratios with 95\% hierarchical bootstrap intervals) and the fraction of
paths leaving the ball within $1000h_0$ at each step (all samplers).
$^\ast$Inside the margin only because both arms are censored.}
\label{tab:nn-study3}
\small
\begin{tabular}{@{}lll*{3}{>{\centering\arraybackslash}p{1.25cm}}@{}}
\toprule
 & \multicolumn{2}{c}{Declared contrast, ball $r_1$} & \multicolumn{3}{c}{Paths leaving the ball by $1000h_0$} \\
\cmidrule(lr){2-3}\cmidrule(l){4-6}
Network & H5a: iid, $h_0/4$ vs $h_0$ & H5b: $\tilde\lambda=0.25$, $h_0/4$ vs $h_0/2$ & $h_0$ & $h_0/2$ & $h_0/4$ \\
\midrule
MLP / MNIST & 3.68\ci{1.87}{7.42} & 1.03\ci{1.01}{1.06}$^\ast$ & 96\% & 3\% & 0\% \\
CNN / F-MNIST & 4.79\ci{4.43}{5.21} & 3.20\ci{2.89}{3.53} & 100\% & 81\% & 44\% \\
ViT / CIFAR-10 & 1.54\ci{1.40}{1.68} & 1.95\ci{1.83}{2.08} & 100\% & 100\% & 91\% \\
\bottomrule
\end{tabular}

\end{table*}

\paragraph{Compute.} The new network checkpoints record NVIDIA GeForce
RTX 3090 GPUs and PyTorch 2.10.0+cu128. The 240 Study 1 continuation
files contain 1,920 paths, the 420 Study 2 files 3,360 paths and the 560
Study 3 files 4,480 paths. Their recorded batch execution times sum to
22,746.9, 39,248.6 and 61,291.3 seconds (6.32, 10.90 and 17.03
device-hours), excluding training, calibration, and CPU surrogate
analysis. These sums are not elapsed wall-clock time.
The older neural diagnostic in Appendix~\ref{app:extended-validation}
used an RTX 3060 Laptop GPU and is a separate experiment.

\subsection{Studies 4--7: new datasets, momentum, pool and batch size, burn-in length}
\label{app:nn-studies456}

\begin{figure*}[t]
\centering
\includegraphics[width=\textwidth]{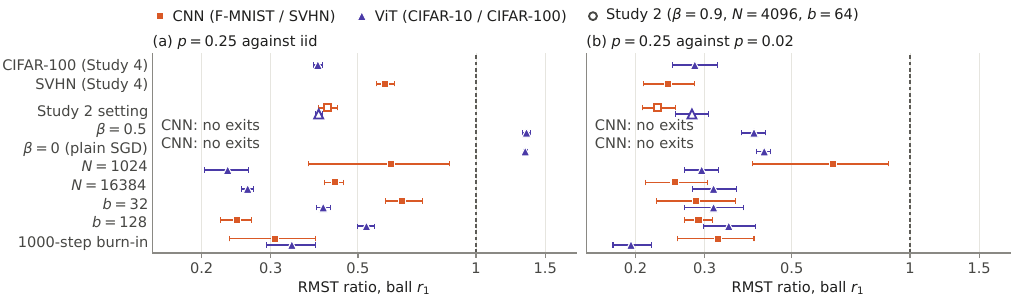}
\caption{Studies 4--7: declared RMST ratios for the ball at $r_1$ with
descriptive bootstrap ranges; below one, $p=0.25$ exits
earlier. Open markers: the Study 2 setting. At $\beta\le0.5$ no CNN path
exited.}
\label{fig:nn-robust}
\end{figure*}

\paragraph{Declarations.} Studies 4--6 were declared after Studies 1--3,
and Study 7 after review, each in a hashed file before its runs. All four
test the interior minimum of Study 2 in the same form: for the ball at
$r_1$, ${\rm RMST}(p{=}0.25)/{\rm RMST}({\rm iid})<1$ and
${\rm RMST}(p{=}0.25)/{\rm RMST}(p{=}0.02)<1$, both 95\% hierarchical
bootstrap intervals below one, assessed separately for every network
and setting (hypotheses H6--H9). Every study uses the seven samplers of
Study 2 and reports every cell.
\begin{itemize}
\item \emph{Study 4 (new datasets, H6).} ResNet-18 and the ViT on
CIFAR-100 \citep{krizhevsky2009learning} (100 classes; 30 and 60 epochs
with crops and flips) and the CNN with three input channels on SVHN
\citep{netzer2011reading} (15 epochs, no augmentation). Learning rates and
radii were fixed on pilot seed~0 by the rules of Study~1 before the
formal declaration; seeds 1--4, 32 paths per seed and sampler, $J=1000$.
For ResNet-18 on CIFAR-100 neither declared candidate passed the pilot:
the burn-in diverged at 0.02, and 3 of 24 pilot paths diverged at 0.01.
By the declared rule this configuration was reported and not run. The
admitted learning rates are 0.05 (CNN, SVHN) and 0.01 (ViT, CIFAR-100).
Test accuracies over seeds 0--4 range from 90.97--91.36\% (CNN, SVHN),
51.16--52.57\% (ViT, CIFAR-100) and 74.24--74.46\% (ResNet-18, CIFAR-100).
\item \emph{Study 5 (momentum, H7).} The CNN and the ViT at $\beta=0$
(plain SGD) and $\beta=0.5$, at the Study 2 learning rates, from new
100-step burn-ins at each $\beta$ on the Study 2 checkpoints and pools
(seeds 5--8). The windows $J=1000(1-\beta)/(1-0.9)$, that is 10,000 and
5,000 updates, keep $Jh/(1-\beta)$ equal to the Study 2 window; RMSTs are
convertible to $\beta=0.9$ equivalent updates as
${\rm RMST}\,0.1/(1-\beta)$; the tables retain raw updates.
Paths per seed and sampler: 16 (CNN) and 8 (ViT). The MLP was not
included because one of its probed minibatches exceeds the plain-SGD
stability threshold $h\kappa<2$ at its learning rate, so that persistent
echoing could diverge and divergence would count as exit.
\item \emph{Study 6 (pool and batch size, H8).} The CNN and the ViT with
$N\in\{1024,16384\}$ at $b=64$, and with $b\in\{32,128\}$ at $N=4096$
($\beta=0.9$), each with its own burn-in on the Study 2 checkpoints;
16 paths per seed and sampler, $J=1000$.
\item \emph{Study 7 (burn-in length, H9).} Training ends with batch size
128 and a decayed learning rate, whereas the burn-in and continuations use
$b=64$ and a constant learning rate. Study 7 lengthens the burn-in
tenfold, to 1,000 iid momentum steps, in the Study 2 setting of the CNN
and the ViT (seeds 5--8, same checkpoints and pools; its first 100 steps
are the Study 2 burn-in); 16 paths per seed and sampler, $J=1000$.
\end{itemize}
Studies 5--7 keep the radii of Study~1. A note records that two
workers were added on a freed GPU and processed the declared Study 5 jobs
in reverse order; jobs, seeds and analysis were unchanged. Two Study 7
notes record that the preparation of the ViT starts ran out of GPU memory,
because another job occupied the shared GPUs, and was repeated with
smaller chunks in the direction and curvature computations; this changes
only the order of summation in those secondary quantities, not the
burn-in, the start state or the ball endpoint.

\paragraph{Study 4 results.} The historical H6 rule selected both admitted configurations
(Table~\ref{tab:nn-study4}). For the CNN on SVHN, $p=0.25$ gave ball-RMST
ratios of 0.59 \ci{0.56}{0.62} to iid and 0.24 \ci{0.21}{0.28} to
$p=0.02$; for the ViT on CIFAR-100, 0.40 \ci{0.39}{0.41} and
0.28 \ci{0.25}{0.32}. Exits were fastest at $p=0.25$ for the ball at
both radii in both networks (at $r_1$ the CNN's $p=0.5$ is within 1\%).
The same contrasts hold for the random projection in both networks and
for $e_{\rm noise}$ in the ViT; for the CNN on SVHN the $e_{\rm noise}$
interval against iid includes one (0.76 \ci{0.56}{1.02}). RR exited 6\%
earlier than iid for the CNN (0.94 \ci{0.89}{0.98}) and 2\% later for the
ViT (1.02 \ci{1.02}{1.03}), as for the ViT on CIFAR-10. Four of 1,792
paths (CNN on SVHN, persistent sampling with $p\in\{0.1,0.05,0.02\}$)
became non-finite, each after all of its recorded exits.

\paragraph{Study 5 results.} H7 fails in all four cells
(Table~\ref{tab:nn-ablations}). For the CNN no path of any sampler reached
the ball radius within the window at either $\beta$, so its two cells
carry no evidence either way. For the ViT, exits slowed monotonically
with persistence at both $\beta=0.5$ and $\beta=0$: relative to iid,
$p=0.25$ was slower by factors of 1.34 \ci{1.32}{1.37} and
1.33 \ci{1.32}{1.35}, and $p=0.02$ by 3.35 \ci{3.09}{3.65} and
3.14 \ci{3.00}{3.31}, while RR was 3--4\% faster than iid. At
$\beta=0.5$ only 9\% of the $p=0.02$ paths exited within the window, so
the first of these ratios understates the slowdown. The contrast
with $p=0.02$ stayed below one, but the interior minimum of Study 2 did
not appear: iid and RR exited first. Momentum and the effective step are
confounded in this design, because the raw step was held fixed, so that
$h/(1-\beta)$ fell by factors of 5 and 10; the windows equalize effective
time but not the reach of the fluctuations, as the CNN's censoring shows.
At the step sizes studied, the speed-up from moderate persistence
therefore requires heavy-ball momentum, whereas the dependence of exit
times on the sampler at a fixed one-step law persists without it
(factors up to 3.4). No Study 5 path became non-finite.

\paragraph{Study 6 results.} The historical H8 rule selected all eight cells
(Table~\ref{tab:nn-ablations}). With the larger pool ($N=16384$) and the
smaller batch ($b=32$) every sampler exits, and $p=0.25$ gave ratios of
0.26--0.65 to iid and 0.25--0.32 to $p=0.02$. With the smaller pool
($N=1024$, epochs of 16 batches) no iid or RR path left the ball within
the window in either network, while persistent sampling still produced
exits (61\% and 100\% of paths at $p=0.25$ for the CNN and the ViT); the
fastest level moved to $p=0.1$ for the CNN. The same happened for the CNN
with the larger batch ($b=128$): no iid or RR exits, while every
$p=0.25$ path exited; the ViT behaved as in Study 2 (0.53
\ci{0.50}{0.55} and 0.35 \ci{0.30}{0.40}). Where the iid RMST equals the
window, the ratio to iid compares with a fully censored arm, so the
contrast with $p=0.02$ carries the evidence. RR exited 22--26\% later than
iid for the ViT with $N=16384$ and $b=32$, and 3\% earlier with $b=128$.
No Study 6 path became non-finite.

\paragraph{Study 7 results.} The historical H9 rule selected both networks
(Table~\ref{tab:nn-ablations}). After the 1,000-step burn-in, $p=0.25$
gave ratios of 0.31 \ci{0.24}{0.39} to iid and 0.33 \ci{0.26}{0.40} to
$p=0.02$ for the CNN, and 0.34 \ci{0.29}{0.39} and 0.20 \ci{0.18}{0.22}
for the ViT, and it was the fastest level in both. The extra 900 steps
moved the start by 1.5--1.6 ball radii $r_1$ (CNN) and 2.1--2.3 (ViT),
so the continuations begin from different states, and for the CNN they
reduced the norm of the start momentum buffer from 2.6--3.3 to 0.4--1.5.
From these more settled starts no CNN iid or RR path left the ball within
the window, whereas every $p=0.25$ path did. The speed-up thus persists after a longer burn-in; this weakens a
short-transient explanation without establishing stationarity. For the ViT, RR and iid were indistinguishable
(1.04 \ci{0.88}{1.28}). No Study 7 path became non-finite.

\begin{table}[ht]
\centering
\caption{Study 4: ball RMST by sampler (updates, maximum 1000) and the declared H6 contrasts with 95\%
descriptive bootstrap ranges. ResNet-18 on
CIFAR-100 had no admissible learning rate in the pilot and was not run.}
\label{tab:nn-study4}
\small
\setlength{\tabcolsep}{4.5pt}
\begin{tabular}{@{}lrrrrrrrll@{}}
\toprule
 & \multicolumn{7}{c}{RMST (updates), ball $r_1$} & \multicolumn{2}{c}{${\rm RMST}(p{=}0.25)$ relative to} \\
\cmidrule(lr){2-8}\cmidrule(l){9-10}
Network & iid & $p{=}0.5$ & $0.25$ & $0.1$ & $0.05$ & $0.02$ & RR & iid & $p{=}0.02$ \\
\midrule
ViT / CIFAR-100 & 211 & 107 & 84 & 106 & 150 & 295 & 216 & 0.40\ci{0.39}{0.41} & 0.28\ci{0.25}{0.32} \\
CNN / SVHN & 198 & 117 & 116 & 155 & 234 & 481 & 186 & 0.59\ci{0.56}{0.62} & 0.24\ci{0.21}{0.28} \\
\bottomrule
\end{tabular}

\end{table}

\begin{table}[ht]
\centering
\caption{Studies 5--7: ball RMST by sampler and the declared H7--H9
contrasts with descriptive bootstrap ranges. The Study 2 rows
($\beta=0.9$, $N=4096$, $b=64$, 100-step burn-in) are the reference;
every other row changes one of these. RMSTs are in updates; the windows
are 1000 updates except for $\beta=0.5$ (5000) and $\beta=0$ (10,000).}
\label{tab:nn-ablations}
\footnotesize
\setlength{\tabcolsep}{4pt}
\begin{tabular}{@{}llrrrrrrrll@{}}
\toprule
 & & \multicolumn{7}{c}{RMST (updates), ball $r_1$} & \multicolumn{2}{c}{${\rm RMST}(p{=}0.25)$ relative to} \\
\cmidrule(lr){3-9}\cmidrule(l){10-11}
Network & Setting & iid & $p{=}0.5$ & $0.25$ & $0.1$ & $0.05$ & $0.02$ & RR & iid & $p{=}0.02$ \\
\midrule
CNN & Study 2 & 209 & 94 & 87 & 119 & 185 & 383 & 209 & 0.42\ci{0.40}{0.44} & 0.23\ci{0.21}{0.25} \\
F-MNIST & $\beta=0.5$ & 5000 & 5000 & 5000 & 5000 & 5000 & 5000 & 5000 & \multicolumn{2}{c}{no exits within the window} \\
 & $\beta=0$ & 10000 & 10000 & 10000 & 10000 & 10000 & 10000 & 10000 & \multicolumn{2}{c}{no exits within the window} \\
 & $N=1024$ & 1000 & 1000 & 606 & 468 & 657 & 953 & 1000 & 0.61\ci{0.37}{0.86} & 0.64\ci{0.40}{0.88} \\
 & $N=16384$ & 166 & 78 & 72 & 95 & 141 & 288 & 170 & 0.44\ci{0.41}{0.46} & 0.25\ci{0.21}{0.31} \\
 & $b=32$ & 72 & 46 & 46 & 61 & 91 & 163 & 70 & 0.65\ci{0.59}{0.73} & 0.28\ci{0.23}{0.36} \\
 & $b=128$ & 1000 & 491 & 247 & 300 & 445 & 853 & 1000 & 0.25\ci{0.22}{0.27} & 0.29\ci{0.27}{0.31} \\
 & burn-in 1000 & 1000 & 396 & 307 & 388 & 609 & 943 & 1000 & 0.31\ci{0.24}{0.39} & 0.33\ci{0.26}{0.40} \\
\addlinespace[3pt]
ViT & Study 2 & 203 & 106 & 81 & 101 & 140 & 290 & 219 & 0.40\ci{0.39}{0.41} & 0.28\ci{0.25}{0.31} \\
CIFAR-10 & $\beta=0.5$ & 1487 & 1622 & 1999 & 2972 & 3979 & 4989 & 1429 & 1.34\ci{1.32}{1.37} & 0.40\ci{0.37}{0.43} \\
 & $\beta=0$ & 2885 & 3234 & 3848 & 5211 & 6717 & 9046 & 2792 & 1.33\ci{1.32}{1.35} & 0.43\ci{0.41}{0.44} \\
 & $N=1024$ & 1000 & 885 & 233 & 278 & 439 & 791 & 1000 & 0.23\ci{0.20}{0.26} & 0.29\ci{0.27}{0.32} \\
 & $N=16384$ & 290 & 113 & 76 & 83 & 128 & 241 & 367 & 0.26\ci{0.25}{0.27} & 0.32\ci{0.28}{0.36} \\
 & $b=32$ & 126 & 60 & 51 & 65 & 93 & 163 & 154 & 0.41\ci{0.39}{0.43} & 0.32\ci{0.27}{0.38} \\
 & $b=128$ & 290 & 197 & 153 & 165 & 252 & 442 & 283 & 0.53\ci{0.50}{0.55} & 0.35\ci{0.30}{0.40} \\
 & burn-in 1000 & 346 & 124 & 118 & 183 & 284 & 604 & 360 & 0.34\ci{0.29}{0.39} & 0.19\ci{0.18}{0.22} \\
\bottomrule
\end{tabular}

\end{table}

\paragraph{Compute for Studies 4--7.} The 224, 168, 448 and 112
continuation files of Studies 4--7 contain 1,792, 1,344, 3,584 and 896
paths; their recorded batch execution times sum to 16,956.6, 70,032.6,
45,665.7 and 5,787.8 seconds (4.71, 19.45, 12.68 and 1.61 device-hours).
The Study 4 pilot took 0.42 and the training of its 15 networks 0.91
device-hours; the eight Study 7 start states took 0.33 device-hours,
excluding the attempts that ran out of memory.

\subsection{A mechanism for the interior minimum}
\label{app:nn-mechanism}
This analysis was added after review and uses only the archived paths of
Studies 2, 3 and 5; it is exploratory. For every path we regenerate the
sampler stream from its archived seed, which marks the updates that draw
a fresh batch, and split the path into runs of one reused batch. The
outward progress of update $n$, $\|\theta_n-\theta_0\|-\|\theta_{n-1}-
\theta_0\|$, is recorded exactly in the traces; squared update lengths
are estimated without bias from the random projection,
$D\,(e_{\rm rand}^\top(\theta_n-\theta_{n-1}))^2$, because $e_{\rm rand}$
is drawn independently of training. Every quantity uses the first 400
updates of every path in $\beta=0.9$, $h_0$ equivalent time (one update
counts as $(h/h_0)\,0.1/(1-\beta)$), whether or not the path has
exited, and is reported relative to its mean under iid sampling in the
same setting. Three Study 2 ResNet-18 paths that became non-finite within
this horizon are excluded.

\paragraph{With heavy-ball momentum a reused batch keeps the parameters
moving outward.} In runs of $p=0.02$ at $\beta=0.9$
(Figure~\ref{fig:nn-mechanism}a), the squared update length peaks at the
second or third update of a run, at 1.9--3.0 times its iid mean, as the
momentum buffer aligns with the batch. The outward progress per update rises from
1.1--1.3 times the iid mean at the first update to a peak of 1.8--2.5 at
updates 5--9, and it falls below the iid mean after 17--24 updates, once
the batch is fitted. Averaged over the run lengths of each sampler, the
outward progress per update is largest at $p=0.25$ for the MLP, CNN and
ResNet-18 (1.26--1.74 times iid) and at $p=0.1$ for the ViT (1.42,
against 1.40 at $p=0.25$), and it falls to 0.60--0.83 at $p=0.02$
(Figure~\ref{fig:nn-mechanism}c): the ordering of the exit times. A few
echoes per batch thus keep the iterate in the productive part of each
run, whereas fifty echoes spend most updates on a batch that is already
fitted, and iid sampling never lets the buffer align with one batch.
With smaller steps (Study 3) the productive part lengthens to 28--37
updates and the averaged progress peaks at switching rates $p\,h_0/h$ of
0.25--1, close to the fastest rates of Study 3 (0.25--0.5).

\paragraph{Without heavy-ball momentum each fresh batch pulls the
parameters back.} At $\beta\le0.5$ the first update after a fresh batch
moves the parameters toward $\theta_0$: its outward progress is $-1.4$
to $-3.1$ times the iid mean (Figure~\ref{fig:nn-mechanism}b), so each
refresh undoes part of the progress made on the previous batch. The
averaged progress then decreases with persistence, to 0.58--0.80 at
$p=0.02$, which is the reversal of Study 5. A plausible reading is that
the start lies near a minimizer of the pool loss: a reused batch pulls
the parameters toward its own minimizer, a fresh batch pulls them back,
and with $\beta=0.9$ the momentum buffer carries the outward velocity
across the change of batch (the first update of a run still moves
outward). Study 8 (Appendix~\ref{app:nn-study8}) tested the last part by
intervention: removing the carried velocity stops that first update from
moving outward, but slows exits no more than resetting the buffer at
random times.

The best persistence level of the averaged progress is the same in 52 of
55 network-seed cells of Studies 2, 3 and 5, and the adjacent level in
the other three; at $\beta\le0.5$ no persistence level exceeds the iid
rate by more than 2\% in any seed. These are descriptive statistics of
archived paths, not a causal decomposition, and the outward progress is
a proxy that does not account for the exit events themselves.

\begin{figure}[ht]
\centering
\includegraphics[width=\textwidth]{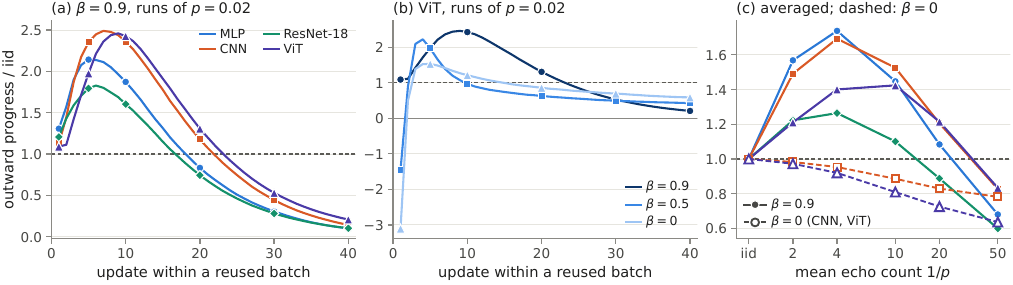}
\caption{Post-hoc mechanism of the interior minimum: outward progress
per update, $\|\theta_n-\theta_0\|-\|\theta_{n-1}-\theta_0\|$, relative
to its iid mean in the same setting. (a) By position within a run of
one reused batch ($p=0.02$, $\beta=0.9$, Study 2). (b) The same for the
ViT at $\beta=0.9$, 0.5 and 0 (Studies 2 and 5). (c) Averaged over all
updates, against the mean echo count, at $\beta=0.9$ (Study 2) and
$\beta=0$ (Study 5, dashed). First 400 equivalent updates of every path.}
\label{fig:nn-mechanism}
\end{figure}

\subsection{Study 8: momentum-buffer resets}
\label{app:nn-study8}
\paragraph{Declaration.} Study 8 tests the mechanism of
Appendix~\ref{app:nn-mechanism} by intervention. It was declared after
that analysis, in a hashed file recorded before any of its runs together
with the path and analysis code. It continues the Study 2 starts of the
MLP, CNN and ViT (seeds 5--8; ResNet-18 was omitted for cost) at
$\beta=0.9$, with the Study 2 learning rates and the Study 1 radii, for
$J=1000$ updates and 16 paths per seed in each of four arms: persistent
sampling with $p=0.25$ (i) unchanged, (ii) with the momentum buffer set
to zero before every update that uses a freshly drawn batch, and (iii)
with the buffer set to zero before each update with probability 0.25,
independently of the batches; and (iv) iid sampling. A reset makes that
update a plain gradient step, after which momentum accumulates again; no
reset precedes the first update, so all arms start from the same full
state. The three persistent arms of a path share its batch stream, so
arms (ii) and (iii) differ only in when the buffer is reset, not in how
often on average. If the buffer carries outward velocity across changes
of batch, resets at refresh remove it at every change, whereas random
resets remove it at a quarter of the changes and otherwise interrupt a
run. We declared, for the ball at $r_1$ and separately for each network,
that resets at refresh slow exits (H10a: ${\rm RMST}$ of arm (ii) over
arm (i) above one), that they slow them more than random resets (H10b:
(ii) over (iii) above one) and, as a secondary directional hypothesis,
that they make $p=0.25$ slower than iid (H10c: (ii) over (iv) above
one), each with its 95\% hierarchical bootstrap interval above one. The
secondary analyses were the same contrasts for the other events and the
outward progress by position within a run, with the declared direction
that resets at refresh lower the progress at the first update of a run.

\begin{figure}[ht]
\centering
\includegraphics[width=\textwidth]{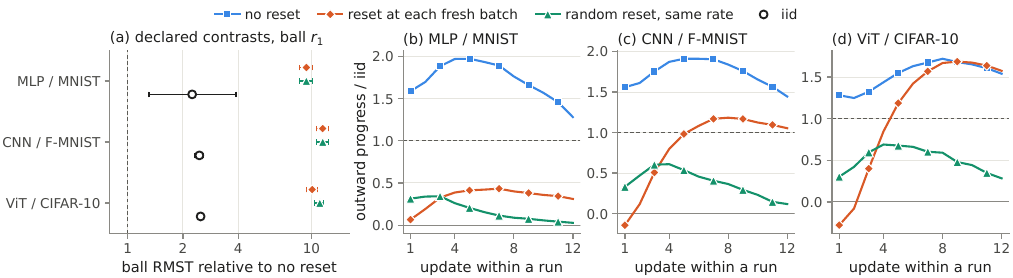}
\caption{Study 8. (a) Ball RMST at $r_1$ relative to persistent sampling
with $p=0.25$ and no reset, with 95\% hierarchical bootstrap intervals;
no MLP or CNN path of either reset arm exited, so those ratios are
limited by the window. (b--d) Outward progress per update,
$\|\theta_n-\theta_0\|-\|\theta_{n-1}-\theta_0\|$, by position within a
run of one batch, relative to its iid mean (first 400 updates of every
path).}
\label{fig:nn-study8}
\end{figure}

\begin{table}[ht]
\centering
\caption{Study 8: RMST by arm (persistent sampling with $p=0.25$ and no
reset, resets at every fresh batch, random resets at the same rate; iid)
and the declared ratios of the resets at refresh to the other arms, with
descriptive bootstrap ranges; no significance highlighting. An RMST
of 1000 means that no path exited; dashes mark ratios between two such
arms. H10a--c concern the ball at $r_1$; the other events are
secondary.}
\label{tab:nn-study8}
\footnotesize
\setlength{\tabcolsep}{4pt}
\begin{tabular}{@{}llrrrrlll@{}}
\toprule
 & & \multicolumn{4}{c}{RMST (updates), maximum 1000} & \multicolumn{3}{c}{Declared ratios, resets at refresh to} \\
\cmidrule(lr){3-6}\cmidrule(l){7-9}
Network & Event & no reset & refresh & random & iid & H10a: no reset & H10b: random & H10c: iid \\
\midrule
MLP / MNIST & ball $r_1$ & 107 & 1000 & 1000 & 239 & 9.38\ci{8.67}{10.10} & \multicolumn{1}{c}{--} & 4.18\ci{2.33}{7.23} \\
 & ball $r_2$ & 122 & 1000 & 1000 & 407 & 8.17\ci{7.47}{8.84} & \multicolumn{1}{c}{--} & 2.46\ci{1.43}{4.82} \\
 & $e_{\rm noise}$ $r_1$ & 287 & 1000 & 985 & 545 & 3.48\ci{1.60}{9.72} & 1.02\ci{1.00}{1.10} & 1.84\ci{1.18}{3.46} \\
 & $e_{\rm noise}$ $r_2$ & 297 & 1000 & 985 & 610 & 3.36\ci{1.59}{9.34} & 1.02\ci{1.00}{1.10} & 1.64\ci{1.15}{2.69} \\
 & $e_{\rm rand}$ $r_1$ & 135 & 997 & 996 & 540 & 7.39\ci{5.75}{9.66} & 1.00\ci{0.99}{1.02} & 1.85\ci{1.34}{2.89} \\
 & $e_{\rm rand}$ $r_2$ & 189 & 1000 & 1000 & 636 & 5.28\ci{3.83}{6.98} & \multicolumn{1}{c}{--} & 1.57\ci{1.21}{2.24} \\
\addlinespace[3pt]
CNN / F-MNIST & ball $r_1$ & 87 & 1000 & 1000 & 213 & 11.52\ci{10.71}{12.34} & \multicolumn{1}{c}{--} & 4.69\ci{4.34}{5.07} \\
 & ball $r_2$ & 160 & 1000 & 1000 & 593 & 6.25\ci{5.96}{6.54} & \multicolumn{1}{c}{--} & 1.69\ci{1.52}{1.87} \\
 & $e_{\rm noise}$ $r_1$ & 81 & 823 & 891 & 161 & 10.21\ci{6.98}{15.06} & 0.92\ci{0.79}{1.08} & 5.11\ci{3.44}{8.25} \\
 & $e_{\rm noise}$ $r_2$ & 88 & 835 & 924 & 219 & 9.46\ci{6.55}{14.12} & 0.90\ci{0.78}{1.03} & 3.82\ci{2.50}{6.50} \\
 & $e_{\rm rand}$ $r_1$ & 244 & 937 & 906 & 538 & 3.84\ci{2.87}{5.31} & 1.03\ci{0.93}{1.16} & 1.74\ci{1.42}{2.24} \\
 & $e_{\rm rand}$ $r_2$ & 296 & 962 & 963 & 709 & 3.25\ci{2.48}{4.42} & 1.00\ci{0.92}{1.07} & 1.36\ci{1.12}{1.68} \\
\addlinespace[3pt]
ViT / CIFAR-10 & ball $r_1$ & 82 & 827 & 910 & 204 & 10.10\ci{9.45}{10.80} & 0.91\ci{0.86}{0.96} & 4.05\ci{3.81}{4.32} \\
 & ball $r_2$ & 299 & 1000 & 1000 & 594 & 3.34\ci{3.19}{3.54} & \multicolumn{1}{c}{--} & 1.68\ci{1.65}{1.75} \\
 & $e_{\rm noise}$ $r_1$ & 35 & 207 & 280 & 253 & 5.85\ci{2.35}{8.41} & 0.74\ci{0.46}{1.06} & 0.82\ci{0.56}{3.80} \\
 & $e_{\rm noise}$ $r_2$ & 36 & 252 & 302 & 272 & 7.04\ci{2.47}{10.12} & 0.83\ci{0.50}{1.14} & 0.93\ci{0.66}{3.59} \\
 & $e_{\rm rand}$ $r_1$ & 96 & 406 & 592 & 260 & 4.22\ci{2.76}{6.47} & 0.69\ci{0.47}{0.99} & 1.56\ci{1.04}{2.44} \\
 & $e_{\rm rand}$ $r_2$ & 152 & 587 & 723 & 359 & 3.87\ci{2.81}{5.46} & 0.81\ci{0.65}{1.02} & 1.64\ci{1.18}{2.40} \\
\bottomrule
\end{tabular}

\end{table}

\paragraph{Results.} Table~\ref{tab:nn-study8} and
Figure~\ref{fig:nn-study8} give every declared endpoint. Without resets,
$p=0.25$ again exited earlier than iid, with ratios of
0.45 \ci{0.26}{0.77}, 0.41 \ci{0.39}{0.43} and 0.40 \ci{0.39}{0.41} for
the MLP, CNN and ViT, close to Study 2. The empirical H10a direction appears in all three networks:
resets at refresh slowed exits by factors of 9.38 \ci{8.67}{10.10},
11.52 \ci{10.71}{12.34} and 10.10 \ci{9.45}{10.80}. No MLP or CNN path
of this arm left the ball within the window, against 94\% of the ViT
paths, so the first two factors are limited by the window. The empirical H10c direction also appears, with ratios to iid of 4.18 \ci{2.33}{7.23}, 4.69 \ci{4.34}{5.07}
and 4.05 \ci{3.81}{4.32}. H10b fails in all three networks: random resets
at the same rate (247--251 per path in both arms) slowed exits at least
as much.
For the MLP and CNN no path of either reset arm left the ball, so the
contrast carries no information there; for the ViT, resets at refresh
slowed exits less than random ones, 0.91 \ci{0.86}{0.96}, the opposite of
the prediction. The secondary events agree: the empirical H10a direction appears for every event in
every network, and no H10b interval lies above one. No Study 8 path
became non-finite.

The recorded first-update progress has the declared direction. With resets at refresh, the
outward progress at the first update of a run fell from 1.59, 1.56 and
1.28 times the iid mean to 0.07, $-0.14$ and $-0.28$: without the carried
velocity, the first update on a fresh batch barely moves outward or moves
back toward $\theta_0$, as at $\beta\le0.5$
(Figure~\ref{fig:nn-study8}b--d). The progress then recovers within the
run as momentum accumulates on the batch: for the CNN and the ViT it
exceeds the iid mean from the sixth and fifth update of a run until the
batch is fitted, and for the ViT it then nearly matches the arm without
resets. With random resets it peaks at 0.34--0.69 times the iid mean
within the first four updates and then declines. Averaged over all
updates, the two reset arms are nearly equal (0.26--0.52 times the iid
mean with resets at refresh and 0.26--0.49 with random resets, against
1.40--1.74 without resets), as their exit times are.

Study 8 documents two features of the intervention in
Appendix~\ref{app:nn-mechanism}; it does not identify causal components. In these experiments the speed-up disappears when the buffer is reset on
average every four updates, $p=0.25$ loses its advantage and exits at
least four times later than iid, whose buffer is never reset. Resets also
shorten the effective step, so, as in Study 5, this part does not
separate accumulation from step length. The results do not support an additional cost specific to removing
velocity at batch changes: resets at refresh reduce the first update's
outward progress, but do not increase RMST more than random resets.
The fully censored MLP and CNN comparisons cannot resolve differences
beyond the observation window. The two reset arms have the same
nominal rate and marginal distribution of buffer ages, but differ in
which gradients are combined. This observation alone does not identify
alignment as the cause. Equal-length follow-ups reverse the
refresh/random ordering across lengths
(Appendix~\ref{app:priority-controls}), precluding a universal ranking. These conclusions rest on three networks at one
learning rate and one level of persistence. The 96 continuation files
(768 paths) took 3,454.2 seconds (0.96 device-hours).

\paragraph{Post-hoc pairing sensitivity.}
The declared bootstrap resampled paths independently between arms,
although the persistent arms share batch streams. A subsequent check
resampled their path indices jointly, with independent path draws for
each occurrence of a resampled seed (20,000 replicates). For the ball
at $r_1$, the refresh/no-reset 95\% intervals were
$[8.78,9.97]$, $[10.75,12.27]$ and $[9.51,10.75]$ for the MLP, CNN and
ViT. The refresh/random interval remained $[1,1]$ in the two fully
censored configurations and was $[0.864,0.952]$ for the ViT. This
sensitivity check preserves the directions of the reported comparisons;
it neither replaces the declared analysis nor resolves the censored tails.
The supplement includes the integer exit-time inputs, their source
hashes and a standalone script for this sensitivity check; reproducing
it does not require rerunning network training.

\subsection{Numerical precision}
\label{app:nn-precision}
Continuations run in single precision, and the CNN's convolutions use
TF32 kernels by default. After review we reran, under a protocol recorded
before the runs, the first eight paths (four for the ViT) of the iid,
$p=0.25$ and $p=0.02$ arms of every Study 2 seed of the MLP, CNN and ViT,
with the archived sampler streams, once in float64 (network, data, start
state and momentum buffer cast exactly; no TF32) and once in float32 with
the original settings. Table~\ref{tab:nn-precision} compares the ball
exits at $r_1$ with the archived paths.

For the MLP and the ViT the float32 reruns reproduce every archived exit
update. In float64 the ViT's exits are again all identical, and 83\% of
the MLP's are, with a median difference of zero updates. For the CNN,
whose convolutions use nondeterministic cuDNN kernels, even the float32
rerun changes the exits by a median of three updates, because the
momentum dynamics amplify small differences; float64 behaves alike. The
two declared ratios change by at most 0.003 at $r_1$ and 0.016 at $r_2$
in every network. The CNN reruns therefore differ by a few updates in the median,
while these precision checks preserve the two ratio point estimates
to the reported tolerance; they do not revalidate every declared interval. A protocol note records that the float64 ViT runs
first ran out of GPU memory on the shared server and were repeated with
only the pool, rather than the whole dataset, on the GPU, which leaves
the computation unchanged. The reruns took 2.46 device-hours.

\begin{table}[ht]
\centering
\caption{Precision check against the archived float32 paths (ball,
$r_1$): fraction of paths with the same exit update, median and maximum
absolute difference in updates, and the two declared Study 2 ratios on
these paths, archived and rerun (point estimates).}
\label{tab:nn-precision}
\small
\setlength{\tabcolsep}{4pt}
\begin{tabular}{@{}llrrrrll@{}}
\toprule
 & & & \multicolumn{3}{c}{Exit update at $r_1$} & \multicolumn{2}{c}{Archived $\to$ rerun} \\
\cmidrule(lr){4-6}\cmidrule(l){7-8}
Network & Arm & Paths & identical & median $|\Delta|$ & max $|\Delta|$ & $p{=}0.25$/iid & $p{=}0.25$/$p{=}0.02$ \\
\midrule
MLP / MNIST & float32 rerun & 96 & 100\% & 0.0 & 0 & 0.472 $\to$ 0.472 & 0.166 $\to$ 0.166 \\
 & float64 & 96 & 83\% & 0.0 & 26 & 0.472 $\to$ 0.471 & 0.166 $\to$ 0.166 \\
CNN / F-MNIST & float32 rerun & 96 & 14\% & 3.0 & 33 & 0.412 $\to$ 0.409 & 0.228 $\to$ 0.227 \\
 & float64 & 96 & 8\% & 3.0 & 87 & 0.412 $\to$ 0.414 & 0.228 $\to$ 0.230 \\
ViT / CIFAR-10 & float32 rerun & 48 & 100\% & 0.0 & 0 & 0.396 $\to$ 0.396 & 0.286 $\to$ 0.286 \\
 & float64 & 48 & 100\% & 0.0 & 0 & 0.396 $\to$ 0.396 & 0.286 $\to$ 0.286 \\
\bottomrule
\end{tabular}

\end{table}

\subsection{Censoring and seed-level checks}
\label{app:nn-seedlevel}
These checks were added after all declared analyses and change no
declared result.

\paragraph{Censoring.} Table~\ref{tab:nn-censoring} gives the percentage
of paths censored at the window for the primary event. Censoring is rare
except for one ResNet-18 seed in Study 1 (a third of the paths of every
sampler), the MLP's iid, RR and $p=0.02$ arms in Study 2, the CNN without
heavy-ball momentum, the small pool and large batch of Study 6, and the
CNN's iid and RR arms after the longer burn-in of Study 7; Study 8
reports its censoring with its results (Table~\ref{tab:nn-study8}). Where
an arm is fully censored, its RMST equals the window, so a ratio to it
compares with the window rather than with an observed exit time.

\begin{table}[ht]
\centering
\caption{Percentage of paths censored at the window (ball, $r_1$); Study
3 reports exit fractions in Table~\ref{tab:nn-study3}. Dashes: sampler
not in the study.}
\label{tab:nn-censoring}
\small
\begin{tabular}{@{}llrrrrrrr@{}}
\toprule
Study & Network & iid & $p{=}0.5$ & $0.25$ & $0.1$ & $0.05$ & $0.02$ & RR \\
\midrule
Study 1 & MLP & 2 & -- & 0 & -- & 0 & -- & 2 \\
 & CNN & 0 & -- & 0 & -- & 0 & -- & 0 \\
 & ResNet-18 & 33 & -- & 33 & -- & 33 & -- & 33 \\
 & ViT & 0 & -- & 0 & -- & 0 & -- & 0 \\
\addlinespace[2pt]
Study 2 & MLP & 12 & 0 & 0 & 0 & 0 & 12 & 11 \\
 & CNN & 0 & 0 & 0 & 0 & 0 & 0 & 0 \\
 & ResNet-18 & 0 & 0 & 0 & 0 & 0 & 14 & 0 \\
 & ViT & 0 & 0 & 0 & 0 & 0 & 0 & 0 \\
\addlinespace[2pt]
Study 4 & CNN / SVHN & 0 & 0 & 0 & 0 & 0 & 1 & 0 \\
 & ViT / CIFAR-100 & 0 & 0 & 0 & 0 & 0 & 0 & 0 \\
\addlinespace[2pt]
Study 5 & CNN, $\beta=0.5$ & 100 & 100 & 100 & 100 & 100 & 100 & 100 \\
 & ViT, $\beta=0.5$ & 0 & 0 & 0 & 0 & 0 & 91 & 0 \\
 & CNN, $\beta=0$ & 100 & 100 & 100 & 100 & 100 & 100 & 100 \\
 & ViT, $\beta=0$ & 0 & 0 & 0 & 0 & 0 & 0 & 0 \\
\addlinespace[2pt]
Study 6 & CNN, $N=1024$ & 100 & 100 & 39 & 2 & 11 & 70 & 100 \\
 & ViT, $N=1024$ & 100 & 72 & 0 & 0 & 0 & 12 & 100 \\
 & CNN, $N=16384$ & 0 & 0 & 0 & 0 & 0 & 0 & 0 \\
 & ViT, $N=16384$ & 0 & 0 & 0 & 0 & 0 & 0 & 0 \\
 & CNN, $b=32$ & 0 & 0 & 0 & 0 & 0 & 0 & 0 \\
 & ViT, $b=32$ & 0 & 0 & 0 & 0 & 0 & 0 & 0 \\
 & CNN, $b=128$ & 100 & 14 & 0 & 0 & 0 & 25 & 100 \\
 & ViT, $b=128$ & 0 & 0 & 0 & 0 & 0 & 0 & 0 \\
\addlinespace[2pt]
Study 7 & CNN, burn-in 1000 & 100 & 2 & 0 & 0 & 2 & 62 & 100 \\
 & ViT, burn-in 1000 & 0 & 0 & 0 & 0 & 0 & 2 & 2 \\
\bottomrule
\end{tabular}

\end{table}

\paragraph{Seed-level checks and inferential status.}
The original hierarchical bootstrap resamples only three or four training
seeds; three seeds admit only ten distinct seed multisets. Its nominal
95\% ranges are retained as descriptive resampling summaries, not
validated population confidence intervals. The original unadjusted
network-wise hypothesis labels are historical classifications, not
confirmatory conclusions of this revision.

Table~\ref{tab:nn-seedlevel} preserves every seed-direction count,
geometric-mean ratio and nominal or Bonferroni-adjusted $t$ range.
These are sensitivity calculations: approximate normality of seed-level
log ratios cannot be adequately checked with 3--4 seeds, and multiplicity
adjustment does not repair that limitation. We remove significance
highlighting and do not interpret exclusion of one as calibrated evidence.
Counts over several architectures, contrasts or reused starts are not a
single independent sample. In particular, the former pooled Study~2
sign-test claim is not used for inference.

The new studies address different estimands. Appendix~\ref{app:replication-followup}
tests the probability that an assigned training run is finite and a
specified 16-path empirical RMST ordering holds; the mechanism studies
test their specified joint validity and direction events. Their corrected
binomial tests do not estimate a population mean RMST ratio, repair the
old ResNet/ViT sample sizes, or convert a capped mean into complete MFPT.
The original failure records and numerical estimates remain unchanged.

\begin{table}[ht]
\centering
\caption{Post-hoc seed-level checks of the declared primary contrasts
(ball, $r_1$). Seeds: seeds whose ratio points in the predicted
direction. Ratio: geometric mean over seeds. Intervals: $t$-intervals
over seeds at 95\% and Bonferroni-corrected over the 16, 8, 4, 8, 16, 4
and 6 primary contrasts of Studies 1, 2, 4, 5, 6, 7 and 8. All ranges
are descriptive sensitivity calculations, without significance highlighting.}
\label{tab:nn-seedlevel}
\footnotesize
\setlength{\tabcolsep}{4pt}
\begin{tabular}{@{}lllrlll@{}}
\toprule
Study & Network & Contrast & Seeds & Ratio & 95\% $t$-interval & Bonferroni $t$-interval \\
\midrule
Study 1 & MLP & RR/iid ($>1$) & 1/4 & 0.80 & [0.57,\,1.12] & [0.32,\,2.01] \\
 & MLP & $p{=}0.05$/iid ($<1$) & 0/4 & 1.61 & [1.46,\,1.79] & [1.22,\,2.13] \\
 & MLP & $p{=}0.25$/iid ($<1$) & 4/4 & 0.61 & [0.53,\,0.71] & [0.41,\,0.92] \\
 & MLP & $p{=}0.05$/$p{=}0.25$ ($<1$) & 0/4 & 2.63 & [2.28,\,3.03] & [1.77,\,3.89] \\
 & CNN & RR/iid ($>1$) & 1/4 & 0.97 & [0.93,\,1.02] & [0.85,\,1.11] \\
 & CNN & $p{=}0.05$/iid ($<1$) & 3/4 & 0.86 & [0.72,\,1.03] & [0.52,\,1.42] \\
 & CNN & $p{=}0.25$/iid ($<1$) & 4/4 & 0.43 & [0.40,\,0.46] & [0.36,\,0.51] \\
 & CNN & $p{=}0.05$/$p{=}0.25$ ($<1$) & 0/4 & 2.00 & [1.76,\,2.27] & [1.41,\,2.83] \\
 & ResNet-18 & RR/iid ($>1$) & 0/3 & 0.95 & [0.83,\,1.08] & [0.54,\,1.65] \\
 & ResNet-18 & $p{=}0.05$/iid ($<1$) & 0/3 & 1.43 & [0.66,\,3.11] & [0.06,\,35.88] \\
 & ResNet-18 & $p{=}0.25$/iid ($<1$) & 2/3 & 0.79 & [0.47,\,1.33] & [0.09,\,6.77] \\
 & ResNet-18 & $p{=}0.05$/$p{=}0.25$ ($<1$) & 0/3 & 1.80 & [0.51,\,6.38] & [0.01,\,342.51] \\
 & ViT & RR/iid ($>1$) & 4/4 & 1.07 & [1.06,\,1.08] & [1.05,\,1.10] \\
 & ViT & $p{=}0.05$/iid ($<1$) & 4/4 & 0.72 & [0.66,\,0.78] & [0.57,\,0.90] \\
 & ViT & $p{=}0.25$/iid ($<1$) & 4/4 & 0.41 & [0.40,\,0.42] & [0.38,\,0.44] \\
 & ViT & $p{=}0.05$/$p{=}0.25$ ($<1$) & 0/4 & 1.75 & [1.63,\,1.89] & [1.42,\,2.16] \\
\addlinespace[2pt]
Study 2 & MLP & $p{=}0.25$/iid ($<1$) & 4/4 & 0.52 & [0.17,\,1.58] & [0.05,\,5.83] \\
 & MLP & $p{=}0.25$/$p{=}0.02$ ($<1$) & 4/4 & 0.16 & [0.15,\,0.17] & [0.13,\,0.19] \\
 & CNN & $p{=}0.25$/iid ($<1$) & 4/4 & 0.42 & [0.38,\,0.46] & [0.35,\,0.50] \\
 & CNN & $p{=}0.25$/$p{=}0.02$ ($<1$) & 4/4 & 0.23 & [0.21,\,0.25] & [0.19,\,0.28] \\
 & ResNet-18 & $p{=}0.25$/iid ($<1$) & 3/3 & 0.69 & [0.60,\,0.78] & [0.47,\,0.996] \\
 & ResNet-18 & $p{=}0.25$/$p{=}0.02$ ($<1$) & 3/3 & 0.20 & [0.19,\,0.21] & [0.16,\,0.24] \\
 & ViT & $p{=}0.25$/iid ($<1$) & 4/4 & 0.40 & [0.39,\,0.40] & [0.38,\,0.41] \\
 & ViT & $p{=}0.25$/$p{=}0.02$ ($<1$) & 4/4 & 0.28 & [0.25,\,0.31] & [0.22,\,0.35] \\
\addlinespace[2pt]
Study 4 & CNN / SVHN & $p{=}0.25$/iid ($<1$) & 4/4 & 0.59 & [0.55,\,0.62] & [0.53,\,0.64] \\
 & CNN / SVHN & $p{=}0.25$/$p{=}0.02$ ($<1$) & 4/4 & 0.24 & [0.20,\,0.30] & [0.17,\,0.35] \\
 & ViT / CIFAR-100 & $p{=}0.25$/iid ($<1$) & 4/4 & 0.40 & [0.38,\,0.41] & [0.37,\,0.42] \\
 & ViT / CIFAR-100 & $p{=}0.25$/$p{=}0.02$ ($<1$) & 4/4 & 0.28 & [0.23,\,0.35] & [0.20,\,0.40] \\
\bottomrule
\end{tabular}

\end{table}

\begin{table}[p]
\ContinuedFloat
\centering
\caption{Post-hoc seed-level checks, continued: Studies 5--8. All ranges
remain descriptive sensitivity calculations.}
\footnotesize
\setlength{\tabcolsep}{4pt}
\begin{tabular}{@{}lllrlll@{}}
\toprule
Study & Network & Contrast & Seeds & Ratio & 95\% $t$-interval & Bonferroni $t$-interval \\
\midrule
Study 5 & CNN, $\beta=0.5$ & $p{=}0.25$/iid ($<1$) & -- & \multicolumn{3}{l}{no exits within the window} \\
 & CNN, $\beta=0.5$ & $p{=}0.25$/$p{=}0.02$ ($<1$) & -- & \multicolumn{3}{l}{no exits within the window} \\
 & ViT, $\beta=0.5$ & $p{=}0.25$/iid ($<1$) & 0/4 & 1.35 & [1.30,\,1.39] & [1.25,\,1.45] \\
 & ViT, $\beta=0.5$ & $p{=}0.25$/$p{=}0.02$ ($<1$) & 4/4 & 0.40 & [0.35,\,0.46] & [0.30,\,0.53] \\
 & CNN, $\beta=0$ & $p{=}0.25$/iid ($<1$) & -- & \multicolumn{3}{l}{no exits within the window} \\
 & CNN, $\beta=0$ & $p{=}0.25$/$p{=}0.02$ ($<1$) & -- & \multicolumn{3}{l}{no exits within the window} \\
 & ViT, $\beta=0$ & $p{=}0.25$/iid ($<1$) & 0/4 & 1.33 & [1.31,\,1.36] & [1.29,\,1.38] \\
 & ViT, $\beta=0$ & $p{=}0.25$/$p{=}0.02$ ($<1$) & 4/4 & 0.42 & [0.39,\,0.46] & [0.36,\,0.50] \\
\addlinespace[2pt]
Study 6 & CNN, $N=1024$ & $p{=}0.25$/iid ($<1$) & 4/4 & 0.56 & [0.26,\,1.19] & [0.07,\,4.46] \\
 & CNN, $N=1024$ & $p{=}0.25$/$p{=}0.02$ ($<1$) & 4/4 & 0.59 & [0.29,\,1.20] & [0.08,\,4.24] \\
 & ViT, $N=1024$ & $p{=}0.25$/iid ($<1$) & 4/4 & 0.23 & [0.18,\,0.29] & [0.12,\,0.43] \\
 & ViT, $N=1024$ & $p{=}0.25$/$p{=}0.02$ ($<1$) & 4/4 & 0.29 & [0.26,\,0.33] & [0.22,\,0.40] \\
 & CNN, $N=16384$ & $p{=}0.25$/iid ($<1$) & 4/4 & 0.44 & [0.40,\,0.47] & [0.35,\,0.55] \\
 & CNN, $N=16384$ & $p{=}0.25$/$p{=}0.02$ ($<1$) & 4/4 & 0.25 & [0.20,\,0.32] & [0.13,\,0.49] \\
 & ViT, $N=16384$ & $p{=}0.25$/iid ($<1$) & 4/4 & 0.26 & [0.25,\,0.27] & [0.23,\,0.29] \\
 & ViT, $N=16384$ & $p{=}0.25$/$p{=}0.02$ ($<1$) & 4/4 & 0.32 & [0.28,\,0.35] & [0.23,\,0.43] \\
 & CNN, $b=32$ & $p{=}0.25$/iid ($<1$) & 4/4 & 0.65 & [0.55,\,0.78] & [0.41,\,1.05] \\
 & CNN, $b=32$ & $p{=}0.25$/$p{=}0.02$ ($<1$) & 4/4 & 0.29 & [0.22,\,0.37] & [0.14,\,0.58] \\
 & ViT, $b=32$ & $p{=}0.25$/iid ($<1$) & 4/4 & 0.41 & [0.40,\,0.42] & [0.37,\,0.45] \\
 & ViT, $b=32$ & $p{=}0.25$/$p{=}0.02$ ($<1$) & 4/4 & 0.32 & [0.28,\,0.36] & [0.22,\,0.45] \\
 & CNN, $b=128$ & $p{=}0.25$/iid ($<1$) & 4/4 & 0.25 & [0.21,\,0.29] & [0.16,\,0.38] \\
 & CNN, $b=128$ & $p{=}0.25$/$p{=}0.02$ ($<1$) & 4/4 & 0.29 & [0.26,\,0.32] & [0.22,\,0.38] \\
 & ViT, $b=128$ & $p{=}0.25$/iid ($<1$) & 4/4 & 0.53 & [0.48,\,0.57] & [0.42,\,0.67] \\
 & ViT, $b=128$ & $p{=}0.25$/$p{=}0.02$ ($<1$) & 4/4 & 0.35 & [0.28,\,0.44] & [0.18,\,0.66] \\
\addlinespace[2pt]
Study 7 & CNN, burn-in 1000 & $p{=}0.25$/iid ($<1$) & 4/4 & 0.30 & [0.19,\,0.46] & [0.14,\,0.63] \\
 & CNN, burn-in 1000 & $p{=}0.25$/$p{=}0.02$ ($<1$) & 4/4 & 0.32 & [0.22,\,0.46] & [0.17,\,0.60] \\
 & ViT, burn-in 1000 & $p{=}0.25$/iid ($<1$) & 4/4 & 0.34 & [0.27,\,0.43] & [0.23,\,0.50] \\
 & ViT, burn-in 1000 & $p{=}0.25$/$p{=}0.02$ ($<1$) & 4/4 & 0.20 & [0.17,\,0.22] & [0.16,\,0.25] \\
\addlinespace[2pt]
Study 8 & MLP & refresh/no reset ($>1$) & 4/4 & 9.39 & [8.68,\,10.16] & [8.05,\,10.96] \\
 & MLP & refresh/random ($>1$) & -- & \multicolumn{3}{l}{no exits in either arm} \\
 & CNN & refresh/no reset ($>1$) & 4/4 & 11.54 & [10.29,\,12.96] & [9.21,\,14.47] \\
 & CNN & refresh/random ($>1$) & -- & \multicolumn{3}{l}{no exits in either arm} \\
 & ViT & refresh/no reset ($>1$) & 4/4 & 10.08 & [9.20,\,11.04] & [8.43,\,12.05] \\
 & ViT & refresh/random ($>1$) & 0/4 & 0.91 & [0.85,\,0.97] & [0.81,\,1.03] \\
\bottomrule
\end{tabular}

\end{table}

\clearpage
\section{A fresh 32-seed replication at each of two networks}
\label{app:replication-followup}
This follow-up froze its protocol before new training or continuation outputs.
It uses MNIST--MLP and FashionMNIST--CNN, each with seeds 100--131, freshly
trained for 10 and 15 epochs respectively using the earlier training recipe.
The pool size is 4096, batch size 64, momentum 0.9, weight decay 0.0005,
burn-in 100 updates, and continuation step 0.05. The two radii remain the
earlier pilot values: $(4.5658,4.9892)$ for MLP and $(5.1212,7.1269)$ for CNN.
Continuation uses float32 with TF32 disabled and deterministic cuDNN settings.
Each trained start contributes 16 paths for each of iid, $p=0.25$ and $p=0.02$
sampling, with horizon $J=1000$. All 64 assigned training/start configurations
were finite, and all 3,072 planned paths completed without nonfinite states.

\paragraph{Prespecified estimand.}
For each assigned training seed and comparator, success means that training
and burn-in remain finite and the 16-path empirical ball-$r_1$ RMST for
$p=0.25$ is strictly smaller than that for the comparator. Ties and nonfinite
training/start outcomes count as failures, without replacement. Infrastructure
or budget interruptions would remain incomplete. This is a joint stability
and direction probability for a finite-path protocol, not a test of the
population mean RMST or its ratio. The training seed is the independent unit;
paths sharing that start are not separate training replicates.

Each of four comparisons uses a one-sided exact binomial test against success
probability $1/2$, with family level $0.025$ and Bonferroni per-comparison
level $0.00625$. Formal decisions were withheld until all 32 assigned seeds
at the corresponding network had terminal outcomes. All four comparisons
have 32/32 successes: the raw $p$-value is $2^{-32}=2.3283\times10^{-10}$,
the Bonferroni-adjusted value is $9.3132\times10^{-10}$, and each one-sided
Clopper--Pearson lower bound is $0.00625^{1/32}=0.85334$. Simultaneous coverage
of these four lower bounds is at least 97.5\% under the seed-level sampling
assumptions. This inference does not use the pooled number of paths as $n$.

\begin{table}[htbp]
\centering
\caption{Descriptive mean RMST at the primary radius across all 32 assigned
seeds per network. Units are updates; the ratios are not confidence bounds.}
\label{tab:replication-followup}
\begin{tabular}{lrrrr}
\toprule
Network & iid & $p=0.25$ & $p=0.02$ & $p=0.25$/iid\\
\midrule
MNIST--MLP & 254.480 & 109.375 & 696.078 & 0.430\\
FashionMNIST--CNN & 212.025 & 88.076 & 387.486 & 0.415\\
\bottomrule
\end{tabular}
\end{table}

At $r_1$, censored fractions for iid/$p=0.25$/$p=0.02$ are
$0.08984/0/0.14844$ for MLP and $0/0/0$ for CNN. At $r_2$, mean RMSTs are
$552.287/126.211/788.943$ for MLP and $607.615/160.768/723.045$ for CNN;
their censored fractions are $0.41016/0/0.29492$ and
$0.01367/0/0.10352$, respectively. These secondary summaries are descriptive.
The audit reconstructed every path's first exit from its saved distance
sequence and every seed's RMST, censoring and success indicators; all matched.
The frozen source hashes matched, and 3,205 archived evidence-file hashes
were verified after transfer. Full checkpoints are retained locally and
are not included in the submission package.

This replication strengthens evidence for the direction of the intermediate
persistence effect at these two networks and previously selected radii.
It does not extend the switching theorem to neural training, establish a
complete MFPT ordering, or identify the mechanism of momentum resets.

\section{Equal-length reset controls with a retained negative contrast}
\label{app:mechanism-followup}
The formal control protocol was frozen before the separate calibration and
pilot outputs. It reuses the 32 starts per network from
Appendix~\ref{app:replication-followup}; this is not another independent
sample of trained networks. Four separate training seeds, 9001--9004,
per network supplied the median of all 800 raw no-reset update lengths
in fixed 200-step windows, without exit-based truncation. The resulting
lengths were 0.1468144953 for MNIST--MLP and 0.1810008138 for
FashionMNIST--CNN. A paired pilot checked implementation, not effect size.

For persistent sampling with refresh probability $p=1/4$, each start has
16 paths in each of six arms: no reset, reset at every fresh batch, or
independent random reset with probability $1/4$, under either raw
momentum or normalized updates of the frozen length. All arms share
batch streams and reset uniforms where relevant, inherit the burn-in
buffer, and skip reset on the first update. The primary radius and
$J=1000$ are unchanged. All 64 starts and all 6,144 continuations were
finite, with no zero directions.

\paragraph{Prespecified inference.}
For each seed and comparator, success requires finite training/burn-in
and all six finite continuation arms, together with strictly larger
16-path primary-radius RMST for normalized refresh-reset updates than
for the normalized comparator. Ties and nonfinite outcomes are failures;
budget or infrastructure interruption is incomplete. The training seed
is the independent unit. Four one-sided exact binomial tests against
success probability $1/2$ use Bonferroni threshold $0.025/4=0.00625$;
no decision was made before all 32 assigned seeds finished at each network.

\begin{table}[htbp]
\centering
\caption{Declared normalized refresh-reset direction: larger finite-path
RMST than the comparator. The unsupported CNN contrast is retained.}
\label{tab:followup-mechanism-primary}
\begin{tabular}{llrrr}
\toprule
Network & Comparator & Successes & Adjusted $p$ & CP lower bound\\
\midrule
MNIST--MLP & no reset & 32/32 & $9.31\times10^{-10}$ & 0.85334\\
MNIST--MLP & random reset & 32/32 & $9.31\times10^{-10}$ & 0.85334\\
FashionMNIST--CNN & no reset & 32/32 & $9.31\times10^{-10}$ & 0.85334\\
FashionMNIST--CNN & random reset & 0/32 & $1$ & $0$\\
\bottomrule
\end{tabular}
\end{table}
The one-sided Clopper--Pearson bounds use per-comparison error 0.00625,
so their simultaneous coverage is at least 97.5\% under the seed-level
sampling assumptions. They concern the joint finite-path protocol
endpoint, not the population mean RMST, its ratio, or complete MFPT.

\begin{table}[htbp]
\centering
\caption{Descriptive mean primary-radius RMST over all 32 assigned starts,
in updates. These means are not population-mean confidence statements.}
\label{tab:followup-mechanism-rmst}
\begin{tabular}{llrrr}
\toprule
Network & Updates & No reset & Refresh reset & Random reset\\
\midrule
MNIST--MLP & raw & 108.186 & 1000.000 & 1000.000\\
MNIST--MLP & normalized & 96.002 & 319.756 & 204.670\\
FashionMNIST--CNN & raw & 87.982 & 1000.000 & 1000.000\\
FashionMNIST--CNN & normalized & 78.711 & 268.422 & 329.563\\
\bottomrule
\end{tabular}
\end{table}
All raw refresh-reset and random-reset paths were censored at 1000;
raw no-reset and all normalized paths had zero censoring. Thus the two
raw reset arms cannot be ordered beyond the observation window.
The normalized controls retain reset-dependent exit differences when
nominal update lengths are fixed, but do not consistently support the
proposed ordering of refresh-linked and independent random resets.
In particular, all 32 CNN seed-level RMSTs have the strict reverse
ordering; none is a tie. This descriptive reversal does not redefine
the original one-sided hypothesis.
Figure~\ref{fig:followup-mechanism-pairs} displays all paired seed-level values.
Normalization changes the optimizer; it does not yield a causal
decomposition of unmodified momentum SGD.

\begin{figure}[htbp]
\centering
\includegraphics[width=.94\linewidth]{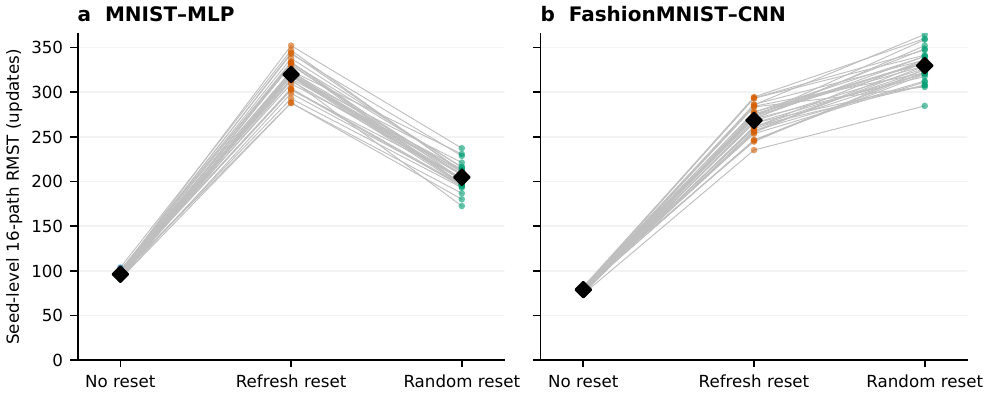}
\caption{All 32 trained starts per network in the formal normalized-update
control. Gray lines pair the same start across reset arms; colored points
are its 16-path primary-radius RMST, and black diamonds are descriptive
means over starts. All plotted paths exited by the 1000-step cap.
The CNN refresh-versus-random ordering is opposite to the declared
hypothesis for every start. These paired observations are not
population-mean confidence intervals.}
\label{fig:followup-mechanism-pairs}
\end{figure}

\paragraph{Audit and numerical scope.}
All saved first exits, capped means, censoring, reset/freshness flags,
seed success records, source hashes and calibration hashes were checked.
The independent audit also checked all 1,600 calibration steps and
96 paired pilot paths. The largest relative error of a recorded nonzero
normalized update-vector norm from its target was $2.03\times10^{-7}$.
These records contain the commanded update vector's norm, not a separate
measurement of the parameter displacement after fp32 subtraction.
They therefore do not certify exact physical step equality or rigorous
rounding control. The separately frozen, independent-seed sensitivity
study below measures that displacement and varies the nominal length.

\subsection{Independent step-length, radius and window sensitivity}
\label{app:mechanism-robustness}
Before reading the formal control outcomes, we froze one bounded follow-up
with 16 new training seeds, 3000--3015, per network. It retains the original
training, burn-in, pool, radii and calibrated lengths, multiplying the latter
by $0.5,1,2$. Each start has eight paired paths for each factor and each of
the three normalized reset arms: 72 continuations per start and 2,304 in total.
All factors and arms share the specified minibatch and reset random streams.
The stopping window is 1000, and both original radii are recorded. We also
describe the same paths capped at 500; these are not additional independent runs.

\paragraph{Joint endpoint and complete audit.}
For each comparator, success requires finite training, burn-in and all 72
continuations, no zero directions, actual displacement lengths within relative
error $2\times10^{-5}$, and strictly larger primary-radius refresh-reset
eight-path RMST at \emph{all three} factors. A tie or failed constituent
condition is a failure. All 32 assigned configurations completed before
analysis, without seed replacement or significance stopping. Four one-sided
binomial tests and CP bounds use error $0.025/4$ per comparison. This is a
separate supplementary family, not pooled 0.025 error control across studies.

The audit reconstructed all first exits, censoring, reset draws and composite
outcomes from 955,421 saved update records, and checked frozen source and
protocol hashes. No nonfinite path, zero direction or numerical-gate failure
occurred. We measure displacement by converting both fp32 parameter vectors
to fp64 \emph{before} subtraction. The largest observed relative length error
was $3.182\times10^{-7}$, and the largest relative norm of the difference
between realized displacement and commanded update was $5.522\times10^{-6}$.
These are measured consistency checks, not rigorous rounding certification
or independent replay of every gradient.

\begin{table}[htbp]
\centering
\caption{Frozen all-three-length endpoint on independent starts. The
refresh-versus-random hypothesis fails at both networks.}
\label{tab:followup-robustness-primary}
\begin{tabular}{llrrr}
\toprule
Network & Comparator & Successes & Adjusted $p$ & CP lower bound\\
\midrule
MNIST--MLP & no reset & 16/16 & $6.10\times10^{-5}$ & 0.72819\\
MNIST--MLP & random reset & 0/16 & $1$ & $0$\\
FashionMNIST--CNN & no reset & 16/16 & $6.10\times10^{-5}$ & 0.72819\\
FashionMNIST--CNN & random reset & 0/16 & $1$ & $0$\\
\bottomrule
\end{tabular}
\end{table}

\paragraph{Reversals and censoring.}
At the primary radius and cap 1000, every MLP seed has larger refresh-reset
RMST than random-reset RMST at factors $0.5$ and $1$, but smaller RMST at
factor $2$. Every CNN seed has larger RMST at $0.5$ and smaller RMST at
$1$ and $2$. These descriptive signs explain the failed all-factor endpoint;
they are not new one-sided tests with a reversed hypothesis.
Figure~\ref{fig:followup-length} shows the means and censoring together.
At half length, primary-radius refresh-reset censoring is 30.469\% for MLP
and 96.094\% for CNN even at 1000 steps. At the larger CNN radius and cap
500, all three half-length arms are fully censored. Thus a short-window tie
does not establish equal exit laws, and the primary-radius direction does
not assert strict separation at every radius and window. The complete
descriptive cells are retained in Table~\ref{tab:followup-length-cells}.
Actual step control therefore leaves directional effects, but their
refresh-versus-random ordering is sensitive to the chosen length. This
limits a general reset-mechanism interpretation of the original results.

\begin{figure}[htbp]
\centering
\includegraphics[width=.92\linewidth]{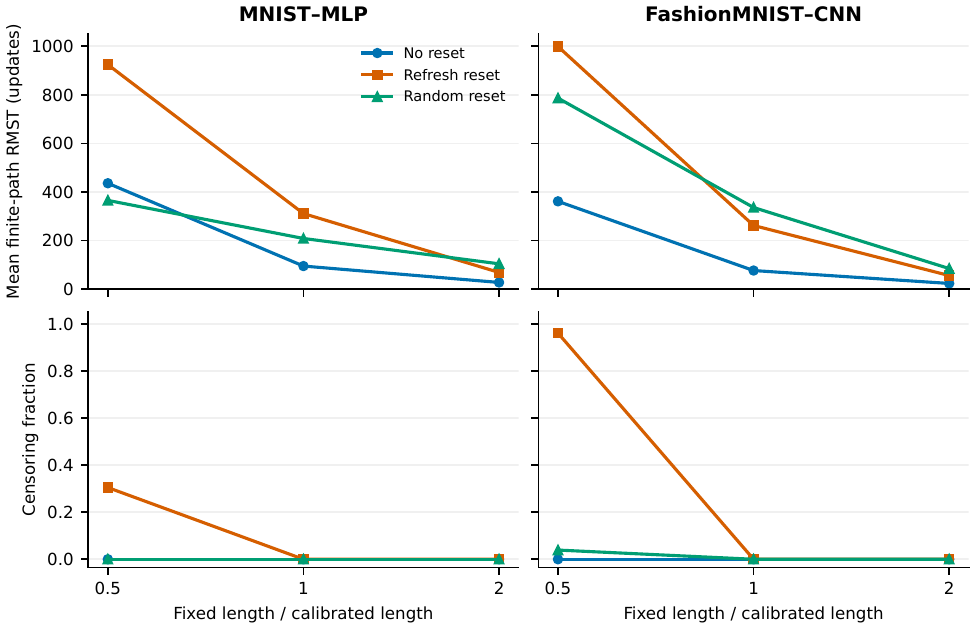}
\caption{Independent length sensitivity at the primary radius and cap 1000.
Top: descriptive mean of 16 seed-level eight-path RMSTs. Bottom: corresponding
path censoring fractions. Lines connect the three fixed factors and do not
interpolate an inferred dose-response law. No population-mean confidence
interval is asserted. The all-factor inference is in
Table~\ref{tab:followup-robustness-primary}.}
\label{fig:followup-length}
\end{figure}

\begin{table}[htbp]
\centering\small
\caption{All declared descriptive radius/window cells for the independent length study. Each entry is mean seed-level RMST in updates, followed by censoring percentage in parentheses. All 16 assigned starts per network are included; each arm has eight paths per start.}
\label{tab:followup-length-cells}
\begin{tabular}{ccc rrr}
\toprule
Length factor & Radius & Cap & No reset & Refresh reset & Random reset\\
\midrule
\multicolumn{6}{c}{MNIST--MLP}\\
0.5 & $r_1$ & 500 & 436.070 (0.0) & 500.000 (100.0) & 365.539 (0.8)\\
0.5 & $r_1$ & 1000 & 436.070 (0.0) & 923.828 (30.5) & 365.570 (0.0)\\
0.5 & $r_2$ & 500 & 499.109 (93.0) & 500.000 (100.0) & 423.531 (9.4)\\
0.5 & $r_2$ & 1000 & 537.211 (0.0) & 988.891 (85.2) & 426.633 (0.0)\\
1 & $r_1$ & 500 & 95.578 (0.0) & 311.445 (0.0) & 209.055 (0.0)\\
1 & $r_1$ & 1000 & 95.578 (0.0) & 311.445 (0.0) & 209.055 (0.0)\\
1 & $r_2$ & 500 & 114.859 (0.0) & 382.289 (3.1) & 241.938 (0.0)\\
1 & $r_2$ & 1000 & 114.859 (0.0) & 383.109 (0.0) & 241.938 (0.0)\\
2 & $r_1$ & 500 & 28.195 (0.0) & 69.383 (0.0) & 104.758 (0.0)\\
2 & $r_1$ & 1000 & 28.195 (0.0) & 69.383 (0.0) & 104.758 (0.0)\\
2 & $r_2$ & 500 & 32.289 (0.0) & 82.781 (0.0) & 126.719 (0.0)\\
2 & $r_2$ & 1000 & 32.289 (0.0) & 82.781 (0.0) & 126.719 (0.0)\\
\midrule
\multicolumn{6}{c}{FashionMNIST--CNN}\\
0.5 & $r_1$ & 500 & 361.672 (0.0) & 500.000 (100.0) & 499.766 (99.2)\\
0.5 & $r_1$ & 1000 & 361.672 (0.0) & 998.336 (96.1) & 786.219 (3.9)\\
0.5 & $r_2$ & 500 & 500.000 (100.0) & 500.000 (100.0) & 500.000 (100.0)\\
0.5 & $r_2$ & 1000 & 771.617 (0.0) & 1000.000 (100.0) & 999.383 (99.2)\\
1 & $r_1$ & 500 & 77.141 (0.0) & 262.414 (0.0) & 336.281 (0.0)\\
1 & $r_1$ & 1000 & 77.141 (0.0) & 262.414 (0.0) & 336.281 (0.0)\\
1 & $r_2$ & 500 & 152.266 (0.0) & 490.383 (80.5) & 497.609 (94.5)\\
1 & $r_2$ & 1000 & 152.266 (0.0) & 588.961 (0.0) & 678.445 (0.0)\\
2 & $r_1$ & 500 & 24.055 (0.0) & 56.516 (0.0) & 85.359 (0.0)\\
2 & $r_1$ & 1000 & 24.055 (0.0) & 56.516 (0.0) & 85.359 (0.0)\\
2 & $r_2$ & 500 & 39.094 (0.0) & 113.438 (0.0) & 186.594 (0.0)\\
2 & $r_2$ & 1000 & 39.094 (0.0) & 113.438 (0.0) & 186.594 (0.0)\\
\bottomrule
\end{tabular}
\end{table}

\section{Full-parameter smooth-network comparisons and sufficient tail conditions}\label{app:fullparam-followup}
These constructions and post-hoc interval checks complement the finite-window comparisons. They do not assert literature priority or certify floating-point optimizer trajectories.
Layerwise parameter-smoothness estimates also underlie neural gradient-descent
analyses \citep[Lemmas 4.3--4.5]{mu2026generalized}. Here a third-derivative
majorant supplies a regional remainder for an exit comparison; the norm
recurrence is not claimed as a literature first.

\subsection{An explicit bound over the entire parameter ball}
Consider a finite feedforward network with coordinatewise $\tanh$ hidden
activations and a scalar linear output $f_\theta(x)$. The vector $\theta$
contains every weight and bias, with the Euclidean norm on their concatenation.
For a fixed example $(x_i,y_i)$ let
\[
 L_i(\theta)=\tfrac12(f_\theta(x_i)-y_i)^2+
             \tfrac\lambda2\norm{\theta}^2.
\]
All derivative norms below are multilinear operator norms induced by
Euclidean norms. Fix a center $\theta_*$ and radius $R\ge0$.

\begin{proposition}[Constructive third-derivative majorant]
For each layer $l$, set
$w_l=\norm{W_{l,*}}_F+R$ and $b_l=\norm{b_{l,*}}_2+R$.
Initialize $A_0=\norm{x_i}$ and $J_0=H_0=T_0=0$. Recursively define
\begin{align*}
 Z_{0,l}&=w_l A_{l-1}+b_l,&
 Z_{1,l}&=\sqrt{A_{l-1}^2+1}+w_lJ_{l-1},\\
 Z_{2,l}&=2J_{l-1}+w_lH_{l-1},&
 Z_{3,l}&=3H_{l-1}+w_lT_{l-1}.
\end{align*}
For a hidden layer of width $m_l$, put $c_2=4/(3\sqrt3)$ and
\begin{align*}
 A_l&=\min\{\sqrt{m_l},Z_{0,l}\},&J_l&=Z_{1,l},\\
 H_l&=Z_{2,l}+c_2Z_{1,l}^2,&
 T_l&=Z_{3,l}+3c_2Z_{1,l}Z_{2,l}+2Z_{1,l}^3.
\end{align*}
For the final scalar linear layer, set $(F,J,H,T)=(Z_0,Z_1,Z_2,Z_3)$.
Then, throughout $\norm{\theta-\theta_*}\le R$,
\[
 \norm{D^3L_i(\theta)}\le M_i(R):=3JH+(F+|y_i|)T.
\]
Consequently, for $\norm{z}\le r\le R$,
\[
 \norm{\nabla L_i(\theta_*+z)-\nabla L_i(\theta_*)-
              \nabla^2L_i(\theta_*)z}
 \le \tfrac12 M_i(R)r^2.
\]
\end{proposition}

\begin{proof}
Write a layer preactivation as $z=Wa+b$ and take unit directions $u,v,w$
in the full parameter space. Its first derivative is
$u_Wa+WDa[u]+u_b$. Cauchy--Schwarz on the layer's weight and bias blocks
gives $\norm{u_Wa+u_b}\le\sqrt{\norm a^2+1}$, proving the first-derivative
majorant. The second derivative is
\[
 W D^2a[u,v]+u_W Da[v]+v_W Da[u].
\]
The third is $W D^3a[u,v,w]$ plus the three terms with one derivative
falling on $W$ and two on $a$. Triangle inequalities yield $Z_2,Z_3$;
these estimates need not be simultaneously sharp in different parameter blocks.

For $\sigma=\tanh$, $|\sigma'|\le1$, $|\sigma''|\le c_2$ and
$|\sigma'''|\le2$. Coordinatewise products satisfy
$\norm{a\odot b}_2\le\norm a_2\norm b_2$, and likewise for three factors.
The multivariate chain rule gives the displayed hidden-layer recursions.
The value bound follows from $|\tanh t|\le\min\{1,|t|\}$.

The third derivative of squared loss consists of three products of $Df$
and $D^2f$, and the term $(f-y_i)D^3f$. The quadratic penalty has zero
third derivative. This proves $M_i(R)$ by induction. Integrating $D^3L_i$
along a segment inside the convex ball gives the Hessian difference bound
$M_i(R)\norm z$. A second integration gives the gradient remainder.
\end{proof}

For any minibatch $B$, average the constants to obtain
$M_B(R)=|B|^{-1}\sum_{i\in B}M_i(R)$. These constants cover the full ball,
not only a projection or observed trajectories. They may be conservative.

\subsection{The envelope evaluated in the experiment}
Use signed velocity $u_n=\beta u_{n-1}-g_{B_n}(\theta_*+x_{n-1})$ and
$x_n=x_{n-1}+hu_n$, with $h>0$ and $0\le\beta<1$. The reference shares
the initial position and velocity, but replaces the gradient by
$c_{B_n}+H_n z_{n-1}$, where $c_B=g_B(\theta_*)$.

\begin{corollary}[A scalar majorant of the full-state comparison]
Let $E_0=Q_0=0$ and $\norm{x_0}<R$. For sampled Hessians
$H_n=H_{B_n}:=\nabla g_{B_n}(\theta_*)$, define
\begin{align*}
 \rho_n&=\min\{R,\norm{z_{n-1}}+E_{n-1}\},\\
 \eta_n&=\tfrac12 M_{B_n}(R)\rho_n^2,\\
 Q_n&=\beta Q_{n-1}+\norm{H_n}_F E_{n-1}+\eta_n,\\
 E_n&=E_{n-1}+hQ_n.
\end{align*}
Then velocity error is at most $Q_n$ and position error at most $E_n$
through the actual exit-causing update. For a common Hessian $\bar H$,
use $H_n=\bar H$ and add $\norm{H_{B_n}-\bar H}_F\rho_n$ to $\eta_n$.
Writing $\tau=\inf\{n\ge1:\norm{x_n}\ge R\}$, define
\[
 T_- =\inf\{n\ge1:\norm{z_n}+E_n\ge R\},\qquad
 T_+ =\inf\{n\ge1:\norm{z_n}-E_n\ge R\}.
\]
With $\inf\varnothing=\infty$, one has $T_-\le\tau\le T_+$,
and the same ordering after capping at any finite $J$.
\end{corollary}

\begin{proof}
For $n\le\tau$, the previous actual position remains inside the ball.
Induction and the position error bound imply $\norm{x_{n-1}}\le\rho_n$.
Subtract the two velocity updates and apply the remainder estimate and
$\norm H_{\mathrm{op}}\le\norm H_F$. Subtraction of position updates
then proves the induction step. For common curvature, add and subtract
$H_{B_n}x_{n-1}$ to obtain the extra mismatch term.
If $\tau<T_-$, the upper triangle inequality at the exit step gives a
contradiction. If $T_+<\tau$, the lower triangle inequality at $T_+$ gives
a contradiction. This argument requires no independence of the batches.
\end{proof}

The construction is an exact-arithmetic statement. Our implementation uses
ordinary float64 norms, Hessians and gradients, with absolute tolerance
$2\times10^{-11}$ in the observed error checks. It is not a directed-rounding
certificate. The scalar majorant may be looser than the manuscript's
matrix-propagator envelope. No unobserved-tail or complete-mean claim follows.

\subsection{Completed independent numerical validation}
Real MNIST 0/1 images are average-pooled into four inputs. A $4$--$4$--$4$--$1$
network has 45 trainable parameters, including all biases. Each seed samples
64 training images, fits all parameters for 400 Adam updates with learning
rate $0.01$ and ridge coefficient $0.001$, and starts continuation with zero
velocity. The finite-pool training accuracy ranges from 93.75\% to 100\%;
test-set generalization is not evaluated.

The exploratory pilot used seeds 9001--9002. Based on that pilot, the
independent validation froze $h=10^{-4}$, $\beta=0.9$, batch size 8,
$J=100$, radii $10^{-4},10^{-3},10^{-2}$, seeds 200--215 and 16 paths per
sampler per seed. This produces 1,536 actual paths, each compared with two
references. The 3,072 reference records are not independent trajectories.
Table~\ref{tab:followup-smooth} reports every radius and sampler for the sampled-Hessian
reference. Width means are conditional on observed exit by $J$.

\begin{table}[htbp]
\centering
\caption{Independent finite-window validation. Width is measured in updates.
Persistent sampling refreshes a uniform batch with probability $0.25$.}
\label{tab:followup-smooth}
\begin{tabular}{llrrr}
\toprule
Radius & Sampler & Exits / 256 & Mean exit width & Maximum exit width\\
\midrule
$10^{-4}$ & iid & 256 & 0.000 & 0\\
$10^{-4}$ & persistent & 256 & 0.000 & 0\\
$10^{-3}$ & iid & 248 & 1.403 & 60\\
$10^{-3}$ & persistent & 256 & 0.191 & 15\\
$10^{-2}$ & iid & 0 & --- & ---\\
$10^{-2}$ & persistent & 34 & 33.206 & 65\\
\bottomrule
\end{tabular}
\end{table}

All recorded error and capped-inclusion checks passed with the stated
floating-point tolerance. However, the prespecified overall feasibility gate
\emph{failed}: at the primary radius $10^{-3}$, 504/512 paths exited,
or 98.4375\%, exceeding the declared upper limit of 95\%.
The other criteria held: all 16 seeds had exits, every observed exit was
multi-step, and 495/504 exited paths had width at most 10 updates (98.214\%).
These facts do not permit removing the failed criterion after seeing data.
The worst primary width was 60 updates. At $10^{-2}$, high censoring and wide
intervals further restrict the useful range. These are descriptive numerical
observations, not independent-path confidence statements, a successful
overall validation, or evidence of large-network scalability.

\subsection{Equal update lengths do not eliminate directional reset effects}
Consider normalized momentum
$v_n=\beta v_{n-1}+g_n$ and
$\theta_n=\theta_{n-1}-\ell v_n/\norm{v_n}$ with $\ell>0$ and
$\theta_0=v_0=0$. Resetting before step 2 replaces the inherited buffer by zero.

\begin{proposition}[A two-step first-exit construction]
Let the first two realized gradients be unit vectors $a,b$ with
$c=\langle a,b\rangle\in(-1,1)$ and $\beta\in(0,1)$.
After two updates, the distances of the retained and reset trajectories obey
\begin{align*}
 D_{\rm keep}^2&=2\ell^2\left(1+
        \frac{\beta+c}{\sqrt{1+\beta^2+2\beta c}}\right),\\
 D_{\rm reset}^2&=2\ell^2(1+c),\qquad D_{\rm keep}>D_{\rm reset}.
\end{align*}
If $\max\{\ell,D_{\rm reset}\}<R\le D_{\rm keep}$, then
$\tau_{\rm keep}=2<\tau_{\rm reset}$ despite identical update lengths.
Such a radius exists for orthogonal $a,b$; it is not asserted for all $c$.
\end{proposition}

\begin{proof}
The first displacement is $-\ell a$ in both arms. The second is $-\ell b$
after reset, or $-\ell(\beta a+b)/\sqrt{1+\beta^2+2\beta c}$ when retaining
momentum. Expansion gives the formulas. The function
$F(t)=(t+c)/\sqrt{1+t^2+2tc}$ has
$F'(t)=(1-c^2)/(1+t^2+2tc)^{3/2}>0$ and $F(0)=c$, proving strict inequality.
The radius condition excludes a first-step exit in either arm and a second-step
exit after reset, while forcing the retained trajectory to exit at step 2.
\end{proof}

For $\beta=0.9$, $\ell=1$, orthogonal $a,b$ and $R=1.6$, the two distances
are approximately $1.827$ and $1.414$. Affine component losses realize the
gradients exactly; a pool with gradients $\{a,-a,b,-b\}$ has zero mean and
admits the prescribed sequence with positive iid probability. Conditioning
on that sequence does not order unconditional RMST or MFPT.

The proposition explains why an equal-length intervention can retain a
directional first-exit effect. It does not predict the average effect in
trained networks or identify a causal decomposition of unmodified SGD.
The elementary proof supplies a concrete interpretation, not an assertion
of literature priority.

\subsection{A computable sufficient forcing condition for uniform finite means}
The preceding bounds also suggest an explicit way to supply a uniform
remaining-mean input. Write $g_i=\nabla L_i(\theta_*)$,
$H_i=\nabla^2L_i(\theta_*)$, and retain the whole-ball constants $M_i(R)$.
Define
\[
 G=\max_i\left\{\norm{g_i}+\norm{H_i}_F R+\tfrac12 M_i(R)R^2\right\},
 \qquad U=\max\{\norm{u_0},G/(1-\beta)\}.
\]
Every surviving velocity then has norm at most $U$.

\begin{proposition}[A fixed-batch forcing certificate in exact arithmetic]
Suppose an admissible batch $B_*$ satisfies
\[
 a:=\norm{g_*}-\norm{H_*}_F R-\tfrac12 M_*(R)R^2>0,
\]
where stars denote averages over $B_*$. Choose an integer
\[
 L\ge \frac{2R(1-\beta)}{ha}+\frac{\beta U}{a}
                 +\frac{\beta}{1-\beta}.
\]
From every surviving history, imposing this batch for $L$ updates forces
exit. For uniform iid batches of size $B$ from $N$ examples, put
$C=\binom NB$ and $q=C^{-L}$. For persistent sampling with refresh probability
$p\in(0,1)$, one can instead take $q=(p/C)(1-p)^{L-1}$.
In either case,
\[
 \Pr(\tau>kL)\le(1-q)^k,\qquad \mathbb E\tau\le L/q,
\]
and $L/q$ also bounds the expected remaining exit time after any surviving
history of the process with the stated initial-velocity bound.
\end{proposition}

\begin{proof}
The gradient norm is at most $G$ inside the ball, so induction on
$\norm{u_n}\le\beta\norm{u_{n-1}}+G$ proves the uniform velocity bound.
With $d=g_*/\norm{g_*}$, the whole-ball remainder estimate implies
$d^Tg_{B_*}(\theta_*+x)\ge a$ for $\norm{x}\le R$.
If the forced trajectory survives for $j$ updates, its projected velocity is
at most $\beta^jU-a(1-\beta^j)/(1-\beta)$. Summing shows that the projected
displacement over $L$ steps is at most
\[
 h\left[\left(U+\frac{a}{1-\beta}\right)
       \frac{\beta(1-\beta^L)}{1-\beta}
       -\frac{aL}{1-\beta}\right]\le-2R.
\]
This contradicts survival, since the initial projection is less than $R$.
An earlier exit already satisfies the conclusion; gradients outside the
ball are never needed.

Under iid sampling the specified word has probability $C^{-L}$. Under
persistent sampling, refresh to $B_*$ on the first step and retain it for
$L-1$ steps. This event has probability $(p/C)(1-p)^{L-1}$ regardless of the
stored batch. Condition successively on surviving histories to obtain the
geometric tail; summing survival probabilities gives $L/q$. The same argument
starts at any surviving history. It does not apply to the constrained words
of random reshuffling.
\end{proof}

\paragraph{Post-hoc candidate check, not numerical certification.}
For each of the 16 saved validation centers and each radius, we formed candidate
8-example batches by selecting the largest projections onto each example
gradient, then evaluated the displayed margin for the resulting batch.
All 48 center--radius pairs had a positive float64 candidate margin.
For persistent sampling with $p=0.25$, the ranges of $L$ and $\log_{10}(L/q)$
were respectively 29--58 and 15.21--19.13 at $R=10^{-4}$;
30--62 and 15.35--19.66 at $R=10^{-3}$; and 44--163 and 17.26--32.70
at $R=10^{-2}$. Thus even successful rigorous verification would give
extremely loose upper bounds, not useful numerical MFPT estimates.
The search is a sufficient-condition diagnostic; failure to find a batch
would not establish impossibility. These ordinary floating-point margins are
not certificates. The separate verification below avoids trusting their
computed gradients and Hessians.

\subsection{Integer interval verification of the forcing condition}
The forcing proof only needs a fixed direction $d$ with $\norm d_2\le1$
and $d^Tg_{B_*}\ge a>0$ throughout the ball. The direction need not be
aligned with the exact center gradient. For each previously selected batch,
we interpret the stored approximate gradient coordinates as exact rational
constants and divide by their exact $\ell_1$ norm. This gives
$\norm d_1=1$ and hence $\norm d_2\le1$, without trusting gradient accuracy.

We enclose all parameters in the coordinate box
$[\theta_{*,j}-R,\theta_{*,j}+R]$, which contains the Euclidean ball.
For each layer, propagate the value and directional derivative using
\[
 Dz[d]=(dW)a+WDa[d]+db,\qquad
 D\tanh(z)[d]=(1-\tanh^2 z)Dz[d].
\]
The loss derivative is $(f-y)Df[d]+\lambda\theta^Td$; average over the
fixed batch. Its lower interval endpoint supplies $a$. The forcing proof
remains valid with any verified upper bound on all gradient norms.
Independently, the value and first-derivative recurrences supply
\[
 G_{\mathrm{new}}:=\max_i\{(F_i+|y_i|)J_i+
                  \lambda(\norm{\theta_*}+R)\},
\]
whose outward-rounded upper endpoint replaces $G$ in the forcing proof.
No stored autodifferentiated Hessian is used.

\paragraph{Arithmetic.}
Every endpoint is an integer divided by $S=2^{96}$. Addition is exact;
products, quotients and conversions use integer floor/ceiling. Square roots
use integer square roots. For $e^x$, range-reduce nonnegative $x$ to
$t\le1/8$, evaluate the degree-32 Taylor sum, and bound the tail by
$2t^{33}/33!$; repeated squaring recovers $e^x$. For negative $x$ use
reciprocals. The increasing map $(e^{2x}-1)/(e^{2x}+1)$ then bounds $\tanh$.
These constructions establish outward enclosure without floating-point
transcendental routines. Code checks include 1,000 exact-rational point cases,
1,000 nondegenerate interval cases, 257 high-precision transcendental probes,
and 192 independent autodifferentiation comparisons; the tests supplement
the arithmetic argument rather than prove it by sampling.

\begin{table}[htbp]
\centering
\caption{Post-hoc interval verification on all 16 saved centers per radius.
Logarithms are approximate display values; probability bounds are stored as
exact symbolic expressions.}
\label{tab:followup-forcing-interval}
\begin{tabular}{lrrr}
\toprule
Radius & Positive lower margins & Block length $L$ & $\log_{10}(L/q)$ range\\
\midrule
$10^{-4}$ & 16/16 & 313--1863 & 51.72--246.15\\
$10^{-3}$ & 16/16 & 323--1990 & 52.99--262.05\\
$10^{-2}$ & 16/16 & 448--4517 & 68.75--578.13\\
\bottomrule
\end{tabular}
\end{table}

All 48 specified cases have strictly positive lower endpoints. Rational
arithmetic verifies the block-length inequality. For $C=\binom{64}{8}$,
the exact mean bounds are $LC^L$ for iid sampling and
$4LC(4/3)^{L-1}$ for persistent sampling with $p=1/4$.
This computation concerns the ideal exact-real $\tanh$ network, interpreting
saved binary64 centers and features as exact constants and using
$h=1/10000$, $\beta=9/10$, $\lambda=1/1000$, decimal radii and zero
initial velocity, with the ideal stated sampling laws. It does not certify
the floating-point optimizer recurrence or earlier exit brackets. The much
larger bounds reflect a conservative global gradient bound and remain
quantitatively uninformative. This post-hoc verification does not change the
failed statistical validation gate, cover ReLU networks or reshuffling, or
establish literature priority. The proof and interval implementation remain
open to independent review.

\subsection{Independent longer-window sensitivity study}
A subsequent descriptive study froze seeds 300--315, 16 paths per sampler
per seed, radii $10^{-3},10^{-2}$, $h=10^{-4}$ and $J=500$ before its outputs.
All 1,024 actual paths completed the floating-point inclusion checks. At
$R=10^{-3}$ all 256 paths per sampler exited: sampled-Hessian mean widths
among exits were 2.754 updates for iid and 0.496 for persistent sampling,
with maxima 122 and 25. At $R=10^{-2}$, only 14/256 iid and 178/256 persistent
paths exited; their mean exit widths were 309.929 and 349.792 updates, with
maxima 428 and 466. Thus merely extending the window does not resolve the
large-radius precision problem. This study uses new centers, so it is not a
paired-window comparison, and does not replace the failed original gate.
\subsection{Sharper tail bounds from survival geometry and batch families}
\label{app:survival-family-tail}
The preceding fixed-batch bound pays both for an unrestricted velocity
bound and for repeatedly selecting one batch. We can reduce both costs.
The argument below is a sufficient-condition refinement, not a claim of
literature priority for geometric tails or subset-counting methods.

\begin{proposition}[Survival-constrained batch-family forcing]
\label{prop:survival-family-tail}
Consider $u_{n+1}=\beta u_n-g_{B_{n+1}}(x_n)$ and
$x_{n+1}=x_n+hu_{n+1}$, with $h>0$, $0\le\beta<1$, $\|x_0\|<R$,
and exit $\tau=\inf\{n\ge1:\|x_n\|\ge R\}$.
If all batch gradients have a known norm bound $G$ inside the ball, set
\[
 U_* = \min\left\{
 \max\{\|u_0\|,G/(1-\beta)\},\,
 \max\{\|u_0\|,2R/h\}\right\}.
\]
Without such a gradient bound, use $U_*=\max\{\|u_0\|,2R/h\}$ instead.
Let a fixed direction $d$, $\|d\|_2\le1$, and a nonempty batch family
$\mathcal A$ satisfy $d^\top g_B(x)\ge a>0$ for every $B\in\mathcal A$
and $\|x\|<R$. Write $\pi$ for this family's probability under a fresh
uniform batch. If an integer $L\ge1$ satisfies
\begin{equation}
 h\left[\left(U_*+\frac{a}{1-\beta}\right)
 \frac{\beta(1-\beta^L)}{1-\beta}-\frac{aL}{1-\beta}\right]
 \le -2R,\label{eq:survival-forcing-L}
\end{equation}
then from every surviving history, the probability of exit within the
next $L$ updates is at least
\[
 q_{\rm iid}=\pi^L,\qquad
 q_{\rm pers}=p\pi(1-p+p\pi)^{L-1},
\]
respectively for iid batches and persistent uniform refresh with
probability $p\in(0,1)$. In either case,
\[
 \Pr(\tau>kL)\le(1-q)^k,\qquad
 \mathbb E[\tau-n\mid\mathcal F_n]\le L/q
 \quad\hbox{on }\{\tau>n\}.
\]
The second statement includes the full stored sampler and velocity state
in the history. It does not apply to random reshuffling.
\end{proposition}

\begin{proof}
When available, the gradient bound and the momentum recursion give
$\|u_n\|\le\max\{\|u_0\|,G/(1-\beta)\}$ at every surviving state.
For $1\le n<\tau$, both consecutive positions lie inside the ball;
therefore $h\|u_n\|=\|x_n-x_{n-1}\|<2R$. Taking the minimum of these
available bounds, and including $n=0$, gives $\|u_n\|\le U_*$.
This use of survival does not bound velocities on unrestricted paths.

Start at any surviving history and impose $L$ batches from $\mathcal A$.
If the process has not exited earlier, the projected velocity after $j$
updates is at most $\beta^jU_*-a(1-\beta^j)/(1-\beta)$.
Summing gives the left side of \eqref{eq:survival-forcing-L} as an upper
bound on the projected displacement. A displacement of at most $-2R$
is incompatible with both endpoints being in the open ball. Thus the
batch event forces exit; gradients after an earlier exit are irrelevant.

Under iid sampling the event has probability $\pi^L$. For a persistent
sampler, require a refresh into $\mathcal A$ on the first update, of
probability $p\pi$ regardless of the stored batch. At each later update,
retaining a good batch or refreshing into the family has probability
$1-p+p\pi$. The resulting bound is uniform in the earlier history.
Successive conditioning yields geometric block survival. Summing the
survival probabilities bounds the expected remaining time by $L/q$,
without assuming integrability beforehand.
\end{proof}

\paragraph{Counting a verified family without enumerating all batches.}
For a fixed rational direction, let $l_i$ be a rigorous lower bound on
$d^\top g_i(x)$ over the full ball. Choose $b=\min_i l_i$ and a rational
$\Delta\ge0$, and integers $w_i\ge0$ such that
$b+\Delta w_i\le l_i$. For batches of size $B$, define
\[
 \mathcal A_t=\{S:|S|=B,\ \sum_{i\in S}w_i\ge t\},\quad
 a_t=b+\Delta t/B.
\]
Whenever $a_t>0$, this family satisfies the proposition. Its probability
is exactly $\binom NB^{-1}\sum_{s\ge t}C_{B,s}$, where $C_{B,s}$ is the
coefficient of $z^Bx^s$ in $\prod_{i=1}^N(1+zx^{w_i})$.
Integer dynamic programming computes these counts, including repeated
weights, while always selecting distinct example indices. Quantization
only discards potentially good batches; it cannot invalidate inclusion.

\paragraph{Fixed deterministic verification on the same 48 cases.}
Before this computation, we froze all 16 saved centers, three radii,
previous rational directions, and the earlier exact gradient upper
bounds. Individual directional derivatives were recomputed using the
same outward-rounded integer interval kernel. We used
$\Delta=(\max_i l_i-\min_i l_i)/32$ and downward-quantized weights
(all zero if $\Delta=0$). Every occupied positive-margin sum threshold
was considered, with $L\le4096$, using the least integer satisfying
\eqref{eq:survival-forcing-L}. We retained the kinematically improved
singleton certificate as a fallback, and compared candidate mean bounds
with exact rational arithmetic separately for each sampler. This is
post-hoc optimization of valid deterministic upper bounds, not a new
statistical validation, new training sample, or multiplicity-adjusted test.

\begin{table}[htbp]
\centering
\caption{Ranges of approximate $\log_{10}$ uniform remaining-mean upper
bounds over all 16 centers per radius, in updates. Original: fixed batch
and global velocity bound. Survival: kinematic bound and exact finite
geometric sum, retaining the old singleton probability. Family: best
verified family or singleton fallback. Smaller bounds need not be tight.}
\label{tab:survival-family-tail}
\small
\begin{tabular}{llrrr}
\toprule
Radius & Sampler & Original & Survival & Family\\
\midrule
$10^{-4}$ & iid & 3021.70--17973.82 & 68.37--242.55 & 24.45--108.23\\
 & persistent & 51.72--246.15 & 11.84--14.64 & 5.20--10.11\\
$10^{-3}$ & iid & 3118.18--19198.89 & 455.04--1892.91 & 192.69--854.46\\
 & persistent & 52.99--262.05 & 17.67--36.90 & 14.13--36.90\\
$10^{-2}$ & iid & 4324.07--43574.76 & 4324.07--37169.73 & 2378.23--37169.73\\
 & persistent & 68.75--578.13 & 68.75--495.10 & 68.75--495.10\\
\bottomrule
\end{tabular}
\end{table}

All 96 sampler-specific bounds are no larger than their predecessors;
48/48 iid and 46/48 persistent bounds strictly improve.
The small-radius persistent bounds decrease substantially, but remain
far above the observed finite exit times. At the largest radius the
batch-family refinement is never selected for persistent sampling:
the weaker margins offset the larger family probability. Iid bounds
remain extremely large. Thus the calculation improves a sufficient
tail certificate without establishing a useful complete-MFPT estimator.
The exact-real interpretation, exclusion of floating-point optimizer
certification, and the failed original smooth-validation gate remain
unchanged. A separate checker recomputes interval inputs, integer subset
counts, and every displacement, probability and bound comparison.
\subsection{Two-sided moment certificates and a negative occupation control}
\label{app:moment-tail}
A block need not maintain one gradient sign to force exit. Its weighted
signed forcing can instead be large in either direction. The following
certificate uses the second and fourth raw moments of a finite batch
reference law, together with a uniform bound on its state-dependent error.
The moment inequality is the classical Cauchy--Schwarz (equivalently,
Paley--Zygmund) argument; we claim no priority for that inequality.

\begin{proposition}[Two-sided moment forcing]
\label{prop:moment-tail}
Use the dynamics, ball, and uniform surviving-state velocity bound $U_*$
of Proposition~\ref{prop:survival-family-tail}. Suppose a fixed direction
$\|d\|_2\le1$, constants $A_B$ indexed by batches, and $\delta\ge0$ satisfy
\[
 |d^\top g_B(x)-A_B|\le\delta
 \quad\text{for every batch $B$ and $\|x\|<R$}.
\]
For $L\ge1$, define
\[
 w_j=\frac{1-\beta^{L-j+1}}{1-\beta},\qquad
 c_L=\sum_{j=1}^L\beta^j,\qquad
 T_L=\frac{2R}{h}+U_*c_L+\delta\sum_{j=1}^Lw_j.
\]
Let $\widetilde B_1$ be a fresh uniform batch and let subsequent batches
be retained with probability $1-p$ or independently refreshed uniformly
with probability $p$, where $0<p\le1$. For this reference batch law set
\[
 Y_L=\sum_{j=1}^L w_j A_{\widetilde B_j},\qquad
 m_2=\mathbb E Y_L^2,\qquad m_4=\mathbb E Y_L^4.
\]
If $m_2>T_L^2$, then, from any surviving optimizer history,
\begin{equation}
 \Pr(\tau\le n+L\mid\mathcal F_n)\ge q_L,
 \qquad q_L=p\frac{(m_2-T_L^2)^2}{m_4}>0
 \quad\text{on }\{\tau>n\}.
 \label{eq:moment-tail-q}
\end{equation}
Consequently $\mathbb E[\tau-n\mid\mathcal F_n]\le L/q_L$ on survival
and $\Pr(\tau>kL)\le(1-q_L)^k$. The iid case is $p=1$.
Certified bounds $\underline m_2\le m_2$ and
$\overline m_4\ge m_4$ give the same conclusions with
$q_L=p(\underline m_2-T_L^2)^2/\overline m_4$ whenever
$\underline m_2>T_L^2$.
\end{proposition}

\begin{proof}
Fix any surviving history at time $n$. Require a fresh uniform batch at
the next update, an event of probability $p$ independent of the stored
batch and optimizer state. Conditional on this event, the next $L$
batch indices have precisely the reference law above. If all $L$
updates survive, unrolling the momentum recursion gives
\[
 \frac{d^\top(x_{n+L}-x_n)}{h}
 =c_Ld^\top u_n-Y_L-\sum_{j=1}^Lw_j e_j,
 \qquad |e_j|\le\delta.
\]
Both endpoints lie in the open ball, so the absolute displacement is
less than $2R$. The velocity bound and the nonnegative weights therefore
imply $|Y_L|<T_L$. Thus $|Y_L|>T_L$ forces exit, without evaluating
gradients after an earlier exit. The batch chain can be defined beyond
exit solely to specify this forcing event.

For $X=Y_L^2$, Cauchy--Schwarz yields
\[
 m_2=\mathbb E X
 \le T_L^2+\mathbb E[X\mathbf 1_{\{X>T_L^2\}}]
 \le T_L^2+\sqrt{m_4\Pr(X>T_L^2)}.
\]
Rearranging and multiplying by the first-refresh probability proves
\eqref{eq:moment-tail-q}. The bound is uniform in the full history;
successive conditioning and summation of block survival probabilities
give the tail and remaining-mean statements. Replacing moments by the
stated one-sided bounds only decreases the certified probability.
Stationarity of the actual optimizer or its current batch is not assumed.
\end{proof}

\paragraph{Whole-ball enclosures and exact finite-population moments.}
For each example, obtain a rigorous interval
$[l_i,u_i]\ni d^\top g_i(x)$ over the entire ball. Put
$a_i=(l_i+u_i)/2$ and $\delta_i=(u_i-l_i)/2$. For a uniform subset $S$
of size $b$, take $A_S=b^{-1}\sum_{i\in S}a_i$ and let $\delta$ be the
average of the $b$ largest $\delta_i$. These deterministic constants
meet the proposition's condition for every batch, including batches
not selected in the previous family construction. Because the interval
endpoints are dyadic, all raw moments
$\mu_s=\mathbb E A_S^s$, $0\le s\le4$, can be computed exactly by an
integer subset-sum moment recursion. No independence of examples within
a batch is assumed.

To propagate the correlated batch law, let
$M_{r,s}=\mathbb E[Y_t^r A_{\widetilde B_t}^s]$ for $r+s\le4$.
Initialize $M_{0,s}=\mu_s$ and $M_{r,s}=0$ for $r>0$, using an auxiliary
stationary batch at $t=0$. For each successive weight $w=w_{t+1}$, use
\begin{align*}
 P_{r,s}&=(1-p)M_{r,s}+pM_{r,0}\mu_s,\\
 M'_{r,s}&=\sum_{k=0}^r\binom rk w^k P_{r-k,s+k}.
\end{align*}
The auxiliary initialization gives a fresh uniform first batch after
one transition. After $L$ steps, $M_{2,0}$ and $M_{4,0}$ are the required
raw, uncentered moments, retaining any nonzero mean forcing. Outward
integer interval arithmetic encloses this recursion.

\paragraph{Frozen verification and all-case results.}
We reused the same 16 saved centers, three radii, and rational directions
as Appendix~\ref{app:survival-family-tail}, with $N=64$, $b=8$,
$h=10^{-4}$, $\beta=0.9$, and $p=1/4$ for persistent sampling. Before
computation, we froze the eight lengths $L=8,16,32,64,128,256,512,1024$,
96-bit interval arithmetic, and a 1200-second CPU wall-time limit.
All 768 case--sampler--length candidates were retained; no lengths,
directions, radii, or precision were added after inspecting outcomes.
Each sampler-specific result is the smaller of its valid moment bound
and the prior batch-family certificate.

\begin{table}[htbp]
\centering
\caption{Moment refinement on all 16 centers per radius. Positive margin
counts centers with at least one certified $\underline m_2>T_L^2$ in the
frozen grid. Ranges are approximate $\log_{10}$ remaining-mean upper
bounds in updates, including every fallback; they are not MFPT estimates.}
\label{tab:moment-tail}
\small
\begin{tabular}{llcrr}
\toprule
Radius & Sampler & Positive margin & Prior family & With moments\\
\midrule
$10^{-4}$ & iid & 0/16 & 24.45--108.23 & 24.45--108.23\\
 & persistent & 15/16 & 5.20--10.11 & 3.85--10.11\\
$10^{-3}$ & iid & 0/16 & 192.69--854.46 & 192.69--854.46\\
 & persistent & 0/16 & 14.13--36.90 & 14.13--36.90\\
$10^{-2}$ & iid & 0/16 & 2378.23--37169.73 & 2378.23--37169.73\\
 & persistent & 0/16 & 68.75--495.10 & 68.75--495.10\\
\bottomrule
\end{tabular}
\end{table}

Exactly 15/96 sampler-specific bounds improve, all at $R=10^{-4}$ under
persistent sampling. Their upper bounds range from approximately
$7.15\times10^3$ to $1.59\times10^5$ updates, with selected lengths
$256$, $512$, or $1024$. Seed 208 at this radius has no positive moment
margin and retains its approximately $1.28\times10^{10}$ bound.
Every iid and larger-radius case also retains the prior certificate.
Failure of this sufficient moment condition does not imply infinite
mean exit time or invalidate the earlier finite-mean certificate.

A separate checker recomputed all 48 fresh-batch moment sets using
finite-population power-sum identities, checked all 768 probability and
threshold formulas, and independently propagated exact rational moments
using the reversed stationary refresh kernel for $L=8$ in all 96
case--sampler pairs and for all 15 selected improvements. Those 111
checks were enclosed by the production intervals. The whole-ball
interval kernel is inherited from the separately audited calculation;
this is not a second independent implementation of that kernel.

\paragraph{Negative control: allowing unfavorable batches.}
Before the moment computation, a separately frozen occupation-count
search allowed up to $k$ unfavorable updates in a block. If favorable
batches have projected gradient at least $a$ and all batches at least
$b_0\le a$, the largest possible displacement with $k$ unfavorable
updates is bounded by
\[
 D_k=h\left[U_*c_L-a\sum_{j=1}^Lw_j
                  +(a-b_0)\sum_{j=1}^k w_j\right].
\]
Indeed, the weights decrease with $j$, so placing unfavorable updates
first maximizes this bound. The condition $D_k\le-2R$ forces exit.
We computed the probability of at most $k$ unfavorable batches by a
binomial recursion for iid sampling and a two-state occupation recursion
for persistent sampling, requiring the first batch to refresh into the
favorable family. With at most two previously selected families per case,
six fixed length multipliers, and $L\le512$, the search checked 688
candidates. None of the 556 candidates with $k>0$ improved the prior
bound, even when an independent checker replaced rounded probabilities
by exact rationals. Three $k=0$ candidates changed the log bound by less
than $1.6\times10^{-8}$, with no substantive gain. We retain this negative
result and did not expand its grid. The 171 probabilities rounded down
to zero were numerical unavailability, not scientific failures.

Both studies are deterministic certificate refinements on reused centers,
not new independent training experiments. They certify the stated
exact-real tanh dynamics with saved binary inputs treated as exact
constants, not floating-point optimizer trajectories. The failed original
smooth-validation gate, large-radius limitations, and lack of a general
large-network or random-reshuffling tail certificate remain unchanged.
In particular, these bounds do not establish accurate complete-MFPT
prediction or convergence of complete means under a diffusion or
switching approximation.

\subsection{Local complete-mean certificates for full ResNet-18 and ViT dynamics}
\label{app:large-head-tail}
The finite-pool optimizer does not acquire continuous noise merely
because it is stochastic. From a fixed full optimizer and sampler state,
an $r$-step transition has at most $J^r$ atoms, where $J$ is the finite
number of legal batches. The union of reachable states over finite
times is countable. Consequently neither a diffuse-noise minorization
nor Lebesgue irreducibility can be assumed for this exact chain;
a Gaussian replacement changes this issue. Related work on finite-sum
SGD uses iterated-function-system structure and establishes absorbing
sets under separability and monotonicity assumptions
\citep{shirokoff2025separable}. Those assumptions have not been
verified for the coupled momentum networks here. This is why the
following mean bounds require explicit history-uniform exit events,
rather than an appeal to generic noise-driven exploration.

The full-network tail question has two distinct targets: a finite,
history-uniform remaining-mean bound, and a bound tight enough for useful
complete-MFPT prediction. This initial restricted construction establishes
the first on smaller balls, but neither target at the original radii;
Appendix~\ref{app:complete-head-cones} subsequently extends finite-mean
certification to the original ResNet radii. The construction below uses
the classification head and coupled weight decay; the backbone remains
fully trainable. No bound on a full-network Hessian or third derivative
is required. This is a structural specialization of the forcing argument,
not a claim of priority for the softmax gradient identity or geometric tails.

\begin{proposition}[Classification-head cone certificate]
\label{prop:large-head-tail}
Let $\theta=(\psi,W,b)$ parametrize logits
$z_\theta(x)=W\phi_\psi(x)+b\in\mathbb R^C$ and use minibatch
cross-entropy with coupled penalty $\lambda\|\theta\|_2^2/2$, $\lambda>0$.
All parameters follow $u_{n+1}=\beta u_n-g_{B_{n+1}}(\theta_n)$,
$\theta_{n+1}=\theta_n+hu_{n+1}$, with $0\le\beta<1$, $h>0$.
Exit is from $\|\theta-\theta_*\|_2<R$, starting at $\theta_*$ with
stored velocity $u_0$. The head is trainable and its entries use the
stated decay, learning rate, and momentum, without Nesterov or adaptive
preconditioning. Assume finite real logits and well-defined updates
before exit; a selected ReLU subgradient in the backbone is allowed.

Choose a nonempty proper class subset $S$, with $s=|S|$, and define
$a_j=C^{-1}-s^{-1}\mathbf 1_{\{j\in S\}}$. Suppose fixed $(v,t)$ obey
\[
 v^\top\phi_\psi(x)+t\ge0
 \quad\text{for every pool input and every $\theta$ in the ball}.
\]
Let $D=(0,av^\top,at)\ne0$, and choose $\rho\ge\|D\|_2$. Set
\[
 m=\langle D,\theta_*\rangle,\qquad
 A=\lambda(m-R\rho),\qquad
 U=\max\{|\langle D,u_0\rangle|,2R\rho/h\}.
\]
If $m>R\rho$, and the finite pool contains $N_S\ge B$ examples with
labels in $S$, choose
\begin{equation}
 L=\left\lceil\frac{2R\rho(1-\beta)}{hA}
              +\frac{\beta U}{A}+\frac{\beta}{1-\beta}\right\rceil,
 \qquad \pi=\frac{\binom{N_S}{B}}{\binom NB}.
 \label{eq:head-tail-length}
\end{equation}
For iid uniform batches, let $q=\pi^L$; for persistent uniform refresh
with probability $p\in(0,1]$, let $q=p\pi(1-p+p\pi)^{L-1}$. Then
\[
 \Pr(\tau>kL)\le(1-q)^k,\qquad
 \mathbb E[\tau-n\mid\mathcal F_n]\le L/q
 \quad\text{on }\{\tau>n\}.
\]
These are full-parameter ball exits, even though $D$ is supported only
on the head. The statement does not cover random reshuffling.
\end{proposition}

\begin{proof}
For an example with label $y\in S$, write $P_S$ for its total softmax
probability on $S$. Differentiating only the affine head gives
\[
 \langle D,\nabla_\theta\ell(\theta;x,y)\rangle
 =\big(v^\top\phi_\psi(x)+t\big)\frac{1-P_S}{s}\ge0.
\]
This identity holds for the current backbone at every update; it does
not freeze its features. Every batch whose labels all lie in $S$ thus
satisfies
\[
 \langle D,g_B(\theta)\rangle
 \ge\lambda\langle D,\theta\rangle
 \ge\lambda(m-R\rho)=A>0.
\]
At any surviving time $n\ge1$, consecutive parameter vectors lie in
the ball, so $|\langle D,u_n\rangle|<2R\rho/h$; $U$ also covers $n=0$.
If the next $L$ updates use only such batches and remain inside, their
total projected displacement is at most
\[
 h\left[\left(U+\frac{A}{1-\beta}\right)
 \frac{\beta(1-\beta^L)}{1-\beta}-\frac{AL}{1-\beta}\right]
 \le h\left[\frac{\beta U}{1-\beta}
       +\frac{\beta A}{(1-\beta)^2}-\frac{AL}{1-\beta}\right]
 \le-2R\rho.
\]
The last inequality follows from \eqref{eq:head-tail-length}; both
endpoints inside the ball would instead give absolute projected
displacement less than $2R\|D\|_2\le2R\rho$. This contradiction forces
exit. The label-restricted batch family has probability $\pi$. For
persistent sampling, require the first update to refresh into this
family, and each later update to retain or refresh into it. This yields
the stated probability independently of the stored batch and full
optimizer history. Geometric block conditioning completes the proof.
\end{proof}

\paragraph{Directions that avoid full-network derivative bounds.}
The bias choice $v=0,t=1$ works for any backbone with a trainable affine
head. Its condition is
$\overline b-\overline b_S>R\sqrt{1/s-1/C}$, using exact norms;
we round the square root upward in the computation. This is the only
direction type used for ViT, whose final LayerNorm features can be negative.

For ResNet-18, the final residual ReLU and average pooling give
$\phi_\psi(x)\ge0$ for every finite parameter state, including with
trainable BatchNorm affine parameters and the saved running statistics.
Write $\Delta=\overline W-\overline W_S$ for the difference of head-row
means and $\delta_b=\overline b-\overline b_S$. The choice
$v=(\Delta)_+$, $t=(\delta_b)_+$ therefore satisfies the feature condition
and gives
\[
 m=\|v\|_2^2+t^2,\qquad
 \|D\|_2^2=(1/s-1/C)(\|v\|_2^2+t^2).
\]
Thus the verified radius depends on stored head coefficients and a
structural feature sign, rather than a product of backbone derivative
bounds. If $D=0$, it supplies no certificate. Positive decay is essential
to this particular margin; excluding the head bias from decay changes
the applicability of the bias construction.

\paragraph{Frozen computation on existing large-network checkpoints.}
We selected all available Study 2 starts: ResNet-18 seeds 5, 6, 8
(11,173,962 trainable parameters) and ViT seeds 5--8 (1,806,538).
The previously excluded ResNet seed 7 remains unavailable owing to its
nonfinite burn-in and is not replaced. We retained $N=4096$, $B=64$,
$h=0.01$, $\beta=0.9$, $\lambda=0.0005$, and persistent $p=1/4$.
The rational values define the exact-real dynamics being certified;
saved binary32 coefficients and initial buffers are treated as exact.

Before head extraction, we froze radii $10^{-4},10^{-3},10^{-2},10^{-1},1$
plus both original radii for each architecture. Class families comprise
all singletons, all singleton complements, and prefixes of sizes 2--8
after sorting the center's head biases increasingly with index tie
breaking. Both direction types are used for ResNet and only bias
directions for ViT. Every candidate and failed margin is retained.
Exact rational calculations give $A$, $U$, $L$, and batch probabilities;
head-direction norms are rounded upward to a dyadic number with 64
fractional bits. To avoid materializing huge powers, a 256-step rational
power calculation gives an exact lower bound $q\ge2^{-K}$. The saved
certificate is the integer pair $(L,K)$, meaning
$\mathbb E[\tau-n\mid\mathcal F_n]\le L2^K$ on survival. Selection among
frozen candidates uses exact integer comparisons, not display logarithms.

\begin{table}[htbp]
\centering
\caption{Large-network exact-real tail certificates. Success counts
checkpoint--radius conditions, not independent training replications or
coverage probabilities. Entries are ranges of approximate $\log_{10}$
mean upper bounds in updates; a dash denotes failure of this sufficient
condition, not evidence for infinite MFPT. No original-radius case passes.}
\label{tab:large-head-tail}
\small
\begin{tabular}{llcrr}
\toprule
Network & Radius & Certified & iid & Persistent\\
\midrule
ResNet-18 & $10^{-4}$ & 3/3 & 82.45--110.85 & 8.70--10.01\\
 & $10^{-3}$ & 3/3 & 539.35--574.58 & 30.91--31.82\\
 & $10^{-2}$ & 3/3 & 5178.60--5524.80 & 231.43--239.54\\
 & $10^{-1}$ & 3/3 & 53880.25--57637.73 & 2253.85--2349.89\\
 & $1$ & 3/3 & 993529.79--1108466.99 & 37045.70--40180.37\\
 & $4.0620,\ 5.5001$ & 0/6 & --- & ---\\
\midrule
ViT & $10^{-4}$ & 4/4 & 431.78--792.59 & 24.06--25.47\\
 & $10^{-3}$ & 4/4 & 4093.58--4985.88 & 125.77--134.84\\
 & $10^{-2}$ & 4/4 & 42120.40--51564.06 & 1147.87--1245.14\\
 & $10^{-1}$ & 4/4 & 652564.67--856611.31 & 14207.65--15773.35\\
 & $1$ & 0/4 & --- & ---\\
 & $2.5968,\ 4.6262$ & 0/8 & --- & ---\\
\bottomrule
\end{tabular}
\end{table}

In total, 31/49 conditions certify finite remaining means for both
samplers, yielding 62 selected certificates. An independent exact
checker reconstructed all 270 directions from exported head tensors,
checked all 1890 candidate margins, 846 forcing inequalities, 1692
probability bounds, and all 62 selections. Extraction checked that each
full saved parameter vector equals the model state and that saved
parameters and buffers are finite; hashes identify the complete remote
checkpoints, metadata, CIFAR label files, and installed ResNet source.
The independent checker audits the extracted tensors and certificates;
it does not independently reload the full remote checkpoints.

\paragraph{What this resolves and what remains open.}
The result supplies a constructive, history-uniform finite-mean condition
for specified full-parameter ReLU ResNet-18 and ViT neighborhoods. It
does not promote finite-window observations to complete means. The
smallest reported persistent bounds are already of order $10^{8.70}$
updates for ResNet and $10^{24.06}$ for ViT, so the result is an existence
certificate with poor quantitative precision. All 14 original-radius
cases fail this restricted-direction condition, as do the four ViT
cases at radius 1; the following subsection improves the ResNet result. No
general claim that neural SGD exits every parameter ball is made.
The certificate does not cover floating-point trajectories, random
reshuffling, other decay conventions, or the sign-definite switching
limit. These deterministic calculations reuse trained starts and do not
add independent training samples or resolve the earlier failed
smooth-validation gate. Tight complete-MFPT prediction at the original
network radii remains open.

\paragraph{Why an unconditional all-radius guarantee is impossible.}
This limitation is not only computational. Consider a zero-feature
initialization with zero head weights, zero velocity, and a backbone
whose features vanish at zero parameters. All backbone and head-weight
gradients then remain zero; only the output bias $b_n$ moves. Such a
degenerate initialization is available in the stated architectures,
with valid fixed normalization buffers. Write
$\xi_n=\operatorname{softmax}(b_n)-\widehat y_{B_{n+1}}$, so
$\|\xi_n\|_2\le\sqrt2$. Coupled decay and momentum give
\[
 b_{n+1}=(1+\beta-h\lambda)b_n-\beta b_{n-1}-h\xi_n,
 \qquad b_{-1}=b_0=0.
\]
When $0<h\lambda<(1-\sqrt\beta)^2$, the homogeneous roots lie in
$[0,1)$ and the impulse coefficients are nonnegative, with sum
$1/(h\lambda)$. Consequently $\|b_n\|_2\le\sqrt2/\lambda$ for every
batch sequence. A ball of radius $R>\sqrt2/\lambda$ is never exited,
so its MFPT is infinite. This includes the stated nominal $h$, $\beta$,
and $\lambda$. The example uses different, degenerate initial states;
it does not establish infinite MFPT for any trained checkpoint.
It rules out removing all radius and state conditions from a general
large-network complete-mean assertion.

\subsection{Complete head cones certify the original ResNet radii}
\label{app:complete-head-cones}
The preceding construction restricts the head direction to a particular
rank-one form. Its original-radius failures do not rule out other head
directions. We now remove that restriction within an explicitly defined
feature-sign relaxation. This certifies finite means at both original
ResNet-18 radii and gives a precise obstruction for the corresponding
feature-unrestricted ViT relaxation. Quantitative mean prediction remains
unresolved.

Write the augmented head as $H=[W\ b]$ and a direction as $D=[A\ c]$,
with zero backbone coordinates. For nonnegative features define the
closed convex cone
\[
 \mathcal K_S^+=\{D:D_{yj}=\min_{1\le k\le C}D_{kj}
       \text{ for all }y\in S\text{ and all columns }j\}.
\]
All rows in $S$ attain the same minimum in each column; other rows may
differ freely above it. For unrestricted real features define instead
\[
 \mathcal K_S^{\mathbb R}
 =\{[A\ c]:A_{1j}=\cdots=A_{Cj}\text{ for every weight column }j,
                 \ c_y=\min_k c_k\text{ for every }y\in S\}.
\]
These cones enforce a nonnegative cross-entropy directional derivative
uniformly over the stated feature domain and every simplex probability
vector. Actual network features and probabilities are coupled; the
relaxation discards that additional information.

\begin{proposition}[Optimal geometric radius within a head cone]
\label{prop:complete-head-cone}
Retain the dynamics and decay assumptions of
Proposition~\ref{prop:large-head-tail}. Let $S$ be any nonempty class
subset with $N_S\ge B$, and use $\mathcal K=\mathcal K_S^+$ when the
features are always nonnegative, or $\mathcal K=\mathcal K_S^{\mathbb R}$
without that restriction. Let $P=\operatorname{proj}_{\mathcal K}H_*$
be the Euclidean projection of the center's augmented head. A direction
$D\in\mathcal K$ with
\[
 \langle D,H_*\rangle-R\|D\|_F>0
\]
exists if and only if $R<\|P\|_F$. When this holds, $D=P$ supplies
$m=\|P\|_F^2$; with $\rho\ge\|P\|_F$ satisfying $m>R\rho$, the
same $A,U,L,\pi,q$ formulas and remaining-mean bound in
Proposition~\ref{prop:large-head-tail} apply. The geometric optimality
does not assert that this direction minimizes $L/q$ among all directions.
\end{proposition}

\begin{proof}
For label $y\in S$ and softmax vector $p$, the head derivative is
\[
 \langle D,\nabla_H\ell\rangle
 =\sum_{k=1}^C p_k\big[(A_k-A_y)^\top\phi+(c_k-c_y)\big]\ge0.
\]
For $\mathcal K_S^+$ every bracket is nonnegative when $\phi\ge0$.
For $\mathcal K_S^{\mathbb R}$ the weight differences vanish and the
bias differences are nonnegative. Consequently coupled decay gives
the same whole-ball lower bound $\lambda(m-R\rho)$ as before, and
the forcing and block-probability proofs apply without a rank-one
restriction on $D$.

For completeness, the standard cone-projection argument is short.
The projection variational inequality gives
$\langle H_*-P,D-P\rangle\le0$ for every $D\in\mathcal K$.
Taking $D=0$ and $D=2P$ gives $\langle H_*-P,P\rangle=0$, hence
\[
 \langle H_*,D\rangle\le\langle P,D\rangle
 \le\|P\|_F\|D\|_F.
\]
If $R\ge\|P\|_F$, no positive margin is possible in the cone.
If $R<\|P\|_F$, choose $D=P$ and use
$\langle H_*,P\rangle=\|P\|_F^2$. This proves the equivalence.
The projection identity and convex optimality conditions are classical;
their specialization here identifies the scope of the head certificate.
\end{proof}

\paragraph{Exact projection and an independent optimality certificate.}
For a head column $x\in\mathbb R^C$, start by pooling the indices in
$S$. In increasing order, add any outside entry below the current pooled
mean, stopping when all remaining entries are at least that mean.
Set the pooled entries to their mean and leave the others unchanged.
The resulting $z$ is checked independently by cone feasibility and the
residual $r=x-z$:
\[
 \sum_i r_i=0,\qquad r_j\le0\ (j\notin S),\qquad
 r_j(z_j-z_S)=0\ (j\notin S).
\]
These imply $r$ lies in the negative polar cone and $\langle r,z\rangle=0$,
which certify the unique Euclidean projection. ViT weight columns are
projected onto the constant-row subspace by averaging; only its bias
column uses the minimum-row cone. All pooling, residual and inner-product
checks use exact integer-scaled binary input coefficients and rational
arithmetic. No numerical optimizer or sampled feature test is used.

\paragraph{Frozen exhaustive class-family calculation.}
We retained the same seven checkpoints, optimizer, batch size and pool
as Appendix~\ref{app:large-head-tail}. Before the new calculation we
froze all $2^{10}-1=1023$ nonempty class subsets, including the full
class set, and the two original radii plus radius 1 as a reference.
This gives 21 checkpoint--radius cases. The previous study and exported
heads were already available: this is deterministic post-hoc certificate
optimization, not blinded validation or additional training replication.
We used the same upward 64-bit dyadic norm enclosure and exact 256-step
probability-power blocks. For each case and sampler, selection minimizes
the saved $L2^K$ bound over these projected directions, using exact
integer comparisons.

\begin{table}[htbp]
\centering
\caption{Full head-cone refinement. Bounds are approximate $\log_{10}$
mean upper bounds in updates, rounded to integers for display.
Thus 69365 represents an upper bound of order $10^{69365}$ updates,
not a mean of 69365 updates. Dashes denote no certificate in this cone
family. Every available checkpoint is retained.}
\label{tab:complete-head-cones}
\small
\begin{tabular}{llcrr}
\toprule
Network & Radius & Certified & iid & Persistent\\
\midrule
ResNet-18 & $1$ & 3/3 & 993315--1108216 & 8441--8742\\
 & $4.0620$ & 3/3 & 28090583--31699279 & 69365--74589\\
 & $5.5001$ & 3/3 & 92752830--111129204 & 183126--212121\\
ViT & $1$ & 0/4 & --- & ---\\
 & $2.5968$ & 0/4 & --- & ---\\
 & $4.6262$ & 0/4 & --- & ---\\
\bottomrule
\end{tabular}
\end{table}

The maximum feasible geometric radii across all class cones are
approximately 7.0123, 6.8055 and 6.8539 for the three ResNet checkpoints.
Both original radii are therefore certified in all six cases, for both
iid and persistent sampling. These are full-parameter exact-real exit
certificates for the original trained states, including their stored
initial momentum. The previously failed restricted-direction results
remain valid as results for that smaller family, but no longer imply
that no original-radius ResNet certificate is available.

For ViT the maximal radii in the unrestricted-feature cones are only
approximately 0.4853, 0.4869, 0.4616 and 0.4633. Exact squared-norm
comparisons rule out every direction in these cones at all three tested
radii. In fact this ceiling has the closed form
\[
 \max_{S\ne\varnothing}\|\operatorname{proj}_{\mathcal K_S^{\mathbb R}}
               [W_*\ b_*]\|_F^2
 = C\|\overline W_*\|_2^2+\|b_*\|_2^2.
\]
The weight projection always keeps only the common row; no bias projection
can increase its norm. Choosing a singleton attaining the smallest
center bias leaves the entire bias vector unchanged, attaining equality.
All four exact identities were checked against the exhaustive calculation.
This excludes further direction or class-subset search within
this specific relaxation. It does not exclude certificates using actual
LayerNorm feature bounds, state-dependent directions, multi-stage
forcing, or other Lyapunov arguments, and does not imply infinite MFPT.

The independent checker verified 2,364,153 projected columns via the
optimality conditions above, all 7161 cone directions, 21483 margins,
6016 forcing inequalities, 12032 probability bounds and 18 selected
mean bounds. Production and checking took approximately 24.45 and
43.19 CPU wall seconds, respectively, with no GPU or new trajectories.
The finite-mean existence gap is closed for these six original-radius
ResNet cases under the stated ideal-real assumptions. Their huge upper
bounds still do not provide accurate complete-MFPT estimates or a useful
RMST tail correction. Original-radius ViT, floating-point trajectory
certification and random reshuffling remain unresolved.
Appendix~\ref{app:two-stage-head} subsequently reduces the forcing
lengths using a two-stage argument while retaining these geometric
conditions and the ViT obstruction.

\subsection{Two-stage forcing with the exact decay recurrence}
\label{app:two-stage-head}
The preceding bound pays a large linear penalty for a worst-case
surviving velocity. A short initial block can instead remove positive
projected momentum. The subsequent comparison retains the exact
weight-decay recurrence, improving the numerical bound without changing
the geometric cone or assuming a frozen backbone. Characteristic-root
analysis of heavy-ball recurrences is classical
\citep{flammarion2015}; the result below applies that analysis to a
history-uniform exit certificate with one-sided forcing.

\begin{proposition}[Two-stage head forcing]
\label{prop:two-stage-head}
Retain a valid direction and all assumptions of
Proposition~\ref{prop:complete-head-cone}. Write
\[
 a=m-R\rho>0,\quad b=m+R\rho,\quad A=\lambda a,
 \quad U=\max\{|\langle D,u_0\rangle|,2R\rho/h\}.
\]
Suppose $0<\beta<1$ and
$0<h\lambda<(1-\sqrt\beta)^2$. Define
\[
 r_\pm=\frac{1+\beta-h\lambda
             \pm\sqrt{(1+\beta-h\lambda)^2-4\beta}}{2},
 \qquad \kappa=\frac{r_+-\beta}{r_+-r_-}.
\]
Choose integers $k\ge0$ and $j\ge1$ satisfying
\[
 \beta^k\le\frac{A}{A+(1-\beta)U},
 \qquad \kappa r_+^j b\le a.
\]
Then $L=k+j$ consecutive batches with labels in $S$ force exit.
The iid and persistent block probabilities and the remaining-mean
bound $L/q$ in Proposition~\ref{prop:large-head-tail} apply with this
new $L$, uniformly over surviving histories.
\end{proposition}

\begin{proof}
Let $s_n=\langle D,H_n\rangle$ and $v_n=\langle D,u_n\rangle$.
During survival, $a<s_n<b$. On a favorable batch the projected
cross-entropy gradient $e_{n+1}$ is nonnegative, so
\[
 v_{n+1}=\beta v_n-\lambda s_n-e_{n+1},
 \qquad s_{n+1}=s_n+hv_{n+1}.
\]
At any surviving starting history, $v_0\le U$: use the stored initial
velocity at time zero, and the displacement between two surviving
points thereafter. Conditional on survival throughout the first $k$
favorable updates,
\[
 v_k\le\beta^k U-A\frac{1-\beta^k}{1-\beta}\le0.
\]
If exit occurred earlier there is nothing left to prove.

Reindex the surviving state at $k$ as time zero. Eliminating velocity
gives a forced second-order recurrence. Set
\[
 c_n=\frac{r_+^{n+1}-r_-^{n+1}}{r_+-r_-},\qquad
 d_n=\kappa r_+^n-(\kappa-1)r_-^n.
\]
The root condition implies $\beta<r_-<r_+<1$ and $\kappa>1$;
hence $c_n>0$ and $0<d_n\le\kappa r_+^n$.
Solving the recurrence, or checking its initial values and induction,
gives for $j\ge1$
\[
 s_j=d_j s_0+h\beta c_{j-1}v_0
             -h\sum_{t=1}^j c_{j-t}e_t
       \le \kappa r_+^j b\le a.
\]
This contradicts survival through that update. This argument uses only
nonnegative projected cross-entropy forcing before exit, not linearity
of the network. Requiring the first batch in each block to refresh into
the favorable family yields the same probability lower bound for every
stored sampler state. Iterated conditioning proves the geometric block
tail and complete remaining-mean bound.
\end{proof}

\paragraph{Frozen refinement and exact arithmetic.}
We reused all seven checkpoints, all 1023 subsets and all 21 cases from
Appendix~\ref{app:complete-head-cones}. All geometric failures were
retained. This is a deterministic refinement after observing the earlier
bounds, not independent training replication. The nominal optimizer
has $h\lambda=1/200000$ and $\beta=9/10$, satisfying the root condition.
We enclosed the square root with rational dyadics, used an upper bound
on $r_+$ and $\kappa$, and rounded every power multiplication upward
on a 96-bit dyadic grid. The first-stage inequality was checked with
exact rational powers. Probability bounds used the same exact 256-step
blocks as before. The independent checker used a different power
evaluation order with 192-bit upward rounding, checked its upper bound
against each stored enclosure, and independently selected each minimum
$L2^K$. No floating-point logarithm participates in certification.

\begin{table}[htbp]
\centering
\caption{Two-stage ResNet refinement over three retained checkpoints.
Entries are approximate $\log_{10}$ upper bounds in updates, rounded
to integers. Every case is certified under both samplers. These are
logarithms of very large bounds, not predictions of the mean.}
\label{tab:two-stage-head}
\small
\begin{tabular}{lrrr}
\toprule
Radius & iid & Persistent & Persistent forcing length $L$\\
\midrule
$1$ & 56187--60520 & 802--827 & 5834--6011\\
$4.0620$ & 910282--976464 & 3410--3548 & 26555--27632\\
$5.5001$ & 2111189--2449079 & 5399--5730 & 42366--44974\\
\bottomrule
\end{tabular}
\end{table}

All 18 selected sampler bounds strictly improve over the preceding
full-cone calculation. Independent checks cover all 6016 positive-margin
forcing conditions, 12032 probability bounds and 18 selected certificates;
production and checking took 22.07 and 107.49 CPU wall seconds, with no
GPU use or new trajectories. At the original radii, persistent bounds now
have decimal logarithms approximately 3410--5730, compared with
69365--212121 previously. The favorable block lengths are much shorter,
but their probabilities remain extremely small. Thus this refinement
does not yet supply a useful complete-MFPT estimate or RMST tail
correction. It adds no new ViT certificate: the same positive geometric
margin is still required, and all twelve tested ViT cases retain the
cone obstruction. Actual floating-point trajectories, random reshuffling,
and broader network classes remain outside this result.

\paragraph{Counting overlapping favorable runs.}
The disjoint-block mean bound also loses information about favorable
runs that cross block boundaries. Classical pattern waiting-time
calculations \citep{zhao2017waiting} remove this loss without changing
any geometric forcing premise. The following refinement applies to a
\emph{fixed} favorable batch family; it does not automatically apply to
state-dependent batch-selection rules.

\begin{proposition}[An overlapping-run remaining-mean bound]
\label{prop:overlap-head-tail}
Suppose every $L\ge1$ consecutive batches from a fixed family force
exit, uniformly over surviving optimizer histories. Let the family
have uniform-refresh probability $0<\pi<1$. Under persistent sampling
with $0<p\le1$, set $r=1-p+p\pi$ and $q=p\pi r^{L-1}$. Then
\[
 \mathbb E[\tau-n\mid\mathcal F_n]
 \le \frac{1-\pi r^{L-1}}{(1-\pi)q}
 \le \frac{1}{(1-\pi)q}
 \quad\text{on }\{\tau>n\}.
\]
For iid sampling use $p=1$, giving
$(1-\pi^L)/[(1-\pi)\pi^L]$ for the first upper bound.
\end{proposition}
\begin{proof}
Mark a batch favorable by 1 and unfavorable by 0. This binary chain
has transitions $0\to1$ with probability $p\pi$ and $1\to1$ with
probability $r$. Count only future consecutive favorable draws,
discarding any progress already made before time $n$. Starting from
an unfavorable stored batch is worst: since $r\ge p\pi$, a common
uniform-random-number coupling preserves the ordering of the binary
chains started from 0 and 1, and hence orders their run completion
times. Exit occurs no later than completion of the first future run
of $L$ favorable batches, by the uniform forcing premise.

Let $E_j$ be the mean waiting time with a current run of length $j$,
where $j=0$ denotes an unfavorable stored batch and $E_L=0$.
These means are finite by the preceding disjoint-block argument.
First-step conditioning gives
\[
 E_0=\frac1{p\pi}+E_1,\qquad
 E_j=1+rE_{j+1}+(1-r)E_0\quad(1\le j<L).
\]
For $L=1$ the first equation uses $E_1=0$ directly. Otherwise,
substitution from $j=L-1$ down to 1 yields
\[
 E_0=\frac1{p\pi r^{L-1}}
       +\frac{1-r^{L-1}}{(1-r)r^{L-1}}
     =\frac{1-\pi r^{L-1}}{(1-\pi)q}.
\]
The coupling and pathwise exit comparison hold after every surviving
history, proving the stated remaining-mean bound.
\end{proof}

The corresponding exact run-survival upper envelope is
$e_0^\top Q_L^t\boldsymbol1$, where the transient run-length matrix
has $Q_{00}=1-p\pi$, $Q_{01}=p\pi$ if $L>1$, and, for $1\le j<L$,
$Q_{j0}=1-r$ and $Q_{j,j+1}=r$ when $j+1<L$. The last favorable
transition is absorbing. This is an upper envelope for network
survival, not an identification of the network exit law.

\paragraph{Probability arithmetic and the size of the improvement.}
We keep all 6016 previously certified directions and forcing lengths.
Instead of rounding the logarithmic cost of every 256-step probability
block, we use a normalized 192-bit binary significand, rounding down
after every multiplication and squaring. Its separate exponent prevents
underflow. From $q\ge m2^e\ge2^{-K}$ we obtain the exact certificate
\[
 \mathbb E[\tau-n\mid\mathcal F_n]\le\frac{2^K}{1-\pi}.
\]
All candidate $\pi$ values are below $1/2$, so coefficients
$1/(1-\pi)$ lie between 1 and 2; selecting minima first by $K$ and
then by that rational coefficient is exact. An independent 90-digit
directed-decimal calculation, with a different power evaluation order,
checks all 12032 probability bounds. Exact rational first-step linear
systems check the waiting-time formula on 54 small chains, including
$L=1$ and iid sampling.

All 18 selected bounds improve, with approximate decimal-log reductions
between $8.88$ and $61.00$, arising jointly from overlapping runs and
tighter arithmetic. At the original ResNet radii, the persistent
decimal-log upper bounds become $3382$--$3519$ for $R=4.0620$ and
$5357$--$5686$ for $R=5.5001$; the corresponding iid ranges are
$910259$--$976431$ and $2111135$--$2449027$. These numbers remain
astronomically large and do not provide useful mean predictions or
tail corrections. The 12 unresolved ViT checkpoint--radius cases are
unchanged. This is a deterministic refinement of existing certificates,
with no new training, trajectories, or geometric validity claims.

\paragraph{A limitation of the scalar relaxation, not a network lower bound.}
The remaining conservatism cannot be removed merely by optimizing the
same one-sided scalar recurrence. Retain only $a<s_n<b$, the original
momentum/decay equations and $e_n\ge0$ from the proof above. With
$e_n=0$, $v_0=0$ and $s_0=b-\varepsilon$, its solution is
$s_n=d_n(b-\varepsilon)$. Here $d_0=1$, $d_n\to0$, and
$d_n-d_{n+1}=h\lambda c_n>0$. If $d_n b>a$, a sufficiently small
$\varepsilon>0$ therefore keeps every scalar iterate through time $n$
strictly inside $(a,b)$. No uniform forcing proof using only these
scalar inequalities can require fewer than $n+1$ updates. This scalar
witness is not asserted to be reachable by the actual neural optimizer.

For every one of the 6016 valid candidates, exact outward enclosures
verify $d_{j-1}b>a$, where $j$ is its preceding contraction length.
Thus the necessary scalar horizon is at least $j$, while the existing
sufficient horizon is $k+j$. At either original radius, the smallest
ratio $j/(k+j)$ over the entire valid dictionary is
$26451/26555>0.99608$. Refining only that scalar relaxation can
therefore shorten these sufficient horizons by less than $0.4\%$.
This conclusion leaves room for additional architecture, reachability,
or cross-direction information.

Even allowing all the fixed favorable families to compete does not
make their pattern waiting time small. To quantify this statement,
initialize the sampler with a legal batch containing all ten labels,
so it belongs to no proper-label favorable family. Such a batch exists
in every retained pool. Let $T_*$ be the first future completion of
any family's \emph{shorter}, optimistic run length $j_S$, and put
\[
 \Lambda=\sum_S\pi_S(1-p+p\pi_S)^{j_S-1}.
\]
For any window start, the probability of membership in family $S$ is
at most $\pi_S$ under this sampler initialization. A union bound over
families and window starts, with no independence assumption between
families, gives $\Pr(T_*\le t)\le t\Lambda$. Consequently,
\[
 \Lambda\le2^{-K}\quad\Longrightarrow\quad
 \mathbb E[T_*]\ge2^{K-2}\qquad(K\ge2),
\]
by summing survival probabilities for the first $2^{K-1}$ times.
The original certificate patterns are longer, so their first completion
time is at least $T_*$. The shorter patterns need not themselves force
network exit.

Independent directed-arithmetic checks validate all 18 summed-rate
bounds. For persistence at the original radii, the decimal logarithms
of these \emph{pattern-waiting} lower bounds range from about
$3368$--$3504$ and $5342$--$5670$, respectively. Network exit may occur
far earlier than any such pattern: these are \emph{not} lower bounds
on its MFPT, nor claims that a particular sampler--parameter history
is reachable. They identify a limitation of the frozen independent-cone,
consecutive-run proof strategy. Interleaved controls, stronger geometric
information and actual reachable-set constraints remain outside this
obstruction.

\subsection{Original-radius ViT: evidence and remaining scope}
\label{app:vit-adaptive-energy}
For the original ViT checkpoints, the tested radii are $2.5968$ and
$4.6262$. The local classification-head certificate and the cone
obstruction in Appendices~\ref{app:large-head-tail}
and~\ref{app:complete-head-cones} concern specified sufficient
conditions. Failure of those conditions neither proves an infinite
mean nor rules out a different certificate.

The remaining issue is uniformity over complete surviving states,
including their incoming momentum and stored minibatch. Finite
trajectory checks, sampled curvature, and successful exits from the
original checkpoints do not establish this premise. Below we report
audited finite-window continuation studies, then give a compact-cover
sufficient condition and a counterexample explaining why pointwise
exit possibilities alone are insufficient. No original-radius ViT
complete-MFPT guarantee is claimed.

\subsection{Fresh continuations beyond the original observation window}
\label{app:vit-fresh-tail}
An independent reconstruction of the archived Study~2 ViT traces
finds 31 of 128 paths still inside the larger ball at update 1,000
for $p=0.02$ (9, 10, 9 and 3 for training seeds 5--8).
No iid path is censored at either radius. These are original data,
not additional replications. To measure the omitted finite-window
tail, we froze 16 fresh continuation streams for each of the same
four checkpoints and each of iid and $p=0.02$, giving 128 new paths.
All 1,806,538 parameters are updated with the original checkpoint,
momentum buffer, 4,096-example pool, batch size 64, $h=0.01$,
$\beta=0.9$, and weight decay $0.0005$. Each path stops at its first
exit from the larger ball or the fixed horizon $H=4000$; this also
determines the first exit from the nested smaller ball. No checkpoint
is retrained, so there are four training units, not 128.

\paragraph{Execution repair and audit.}
Parameters and optimizer arithmetic remain binary32; distances are
computed after conversion of both parameter and center to binary64.
The first execution failed its prescribed exact replay check despite
identical batch indices and exit steps. A fixed 32-step diagnostic
localized the initial discrepancy to the patch-convolution gradient.
A separately frozen amendment enabled deterministic PyTorch and
cuDNN algorithms, retaining all 128 RNG seeds and all scientific
settings. The failed execution and diagnostic remain archived and
are excluded from the statistical sample. This kernel change can
change floating-point trajectories; the old and repaired paths are
not pooled.

The repaired execution completes all 128 paths and eight prescribed
full replays, with no numerical failures or budget interruptions.
Whole traces, batch sequences, and terminal parameter and momentum
hashes match in all eight replays. Independent CPU reconstruction
checks all 91,663 sampled batches and all 16 saved terminal states;
the largest CPU/GPU terminal-distance discrepancy is below
$9.31\times10^{-13}$. All 301 transferred files match their recorded
hashes. A separate integer-sum audit reconstructs both first exits,
the survival counts, and the finite-window aggregates below.
These checks are not directed-rounding certification.

\begin{table}[htbp]
\centering\small
\caption{Fresh ViT continuations at the original radii. Each row uses
64 paths from four reused checkpoints. $N_{>1000}$ counts paths
surviving update 1,000. All observed paths exit by 1,388, so the
empirical RMST at 4,000 equals the uncensored observed sample mean;
this does not certify the population MFPT.}
\label{tab:vit-fresh-tail}
\begin{tabular}{lrrrrr}
\toprule
Sampler & $R$ & RMST$_{1000}$ & RMST$_{4000}$ & $N_{>1000}$ & Max.\\
\midrule
iid & 2.5968 & 204.859 & 204.859 & 0 & 223\\
$p=0.02$ & 2.5968 & 261.313 & 261.313 & 0 & 495\\
iid & 4.6262 & 593.219 & 593.219 & 0 & 632\\
$p=0.02$ & 4.6262 & 801.109 & 839.016 & 12 & 1388\\
\bottomrule
\end{tabular}
\end{table}

\paragraph{Prespecified tail-area inference.}
For any integer-valued exit time, including $\tau=\infty$,
\[
 (\tau\wedge H)-(\tau\wedge L)
   =\sum_{n=L}^{H-1}\mathbf1_{\{\tau>n\}}.
\]
Thus the primary endpoint at $R=4.6262$, $p=0.02$, $L=1000$ and
$H=4000$ measures only the additional finite-window survival area.
Its observed mean is $2426/64=37.90625$ updates, with checkpoint
means $4.9375$, $60.4375$, $59.25$ and $27$.
Conditionally on the four fixed checkpoints, the 64 independent
bounded path variables need not be identically distributed.
Hoeffding's inequality gives the prespecified 95\% interval with
half-width $3000\sqrt{\log(40)/128}$, clipped to $[0,3000]$:
\[
 [0,\ 547.194319].
\]
This interval includes zero; the experiment does not establish a
strictly positive conditional mean tail area at that confidence level.
Other arms, radii and windows are descriptive, and none of these
inferences concern a population of newly trained checkpoints.

All sample survival counts at 2,000 and 4,000 are zero, but this does
not bound $\E[(\tau-4000)_+]$. In particular, observing every sampled
exit does not verify the all-history remaining-mean or drift premise
needed for a complete-MFPT theorem. The original-radius ViT guarantee
therefore remains unresolved. The repaired GPU wave used 97,488
full-network updates including replays and 0.409897 GPU hours;
the failed execution and diagnostic are separately accounted for.

\subsection{A prespecified longer-tail diagnostic}
\label{app:vit-defensive-tail}
To seek longer surviving histories, a separate frozen experiment uses
eight new paths per checkpoint from each of the target sampler
$p=0.02$ and a proposal sampler $q=0.002$: 64 paths from the same
four checkpoints, with no new training units. All optimizer settings
and the deterministic execution amendment above are retained. The
fixed horizon is 8,192, the primary radius is $4.6262$, and paths
stop at their first exit. This is a standard defensive-mixture
diagnostic \citep{hesterberg1995defensive}, not a new importance-sampling
method or a complete-MFPT certificate.

Let $P_s,Q_s$ denote the two augmented path laws at checkpoint $s$.
We record each Bernoulli refresh flag, including a refresh that happens
to redraw the held batch. Both laws have the same uniform first batch
and fresh-batch distribution. If $r_n$ refreshes occur in the $n-1$
trials after the first batch, their prefix density ratio and weight are
\[
 L_n=\frac{dQ_{s,n}}{dP_{s,n}}
 =\left(\frac qp\right)^{r_n}
  \left(\frac{1-q}{1-p}\right)^{n-1-r_n},
 \qquad w_n=\frac{2}{1+L_n}\in[0,2].
\]
For each path define
\[
 Z=\sum_{n=4000}^{8191}w_n\mathbf1_{\{\tau>n\}},
 \qquad \widehat A=\frac1{64}\sum_{i=1}^{64}Z_i.
\]
Equal fixed allocation to the two arms gives
$\tfrac12\E_{P_s}[w_nY_n]+\tfrac12\E_{Q_s}[w_nY_n]
=\E_{P_s}[Y_n]$ for every bounded prefix-measurable $Y_n$.
Taking $Y_n=\mathbf1_{\{\tau>n\}}$ and summing shows that
$\widehat A$ is unbiased for the four-checkpoint average of
$\E_{P_s}[(\tau\wedge8192)-(\tau\wedge4000)]$.
After exit the indicators vanish, so no further updates are needed.
The 64 independent, not necessarily identically distributed variables
lie in $[0,8384]$. The frozen Hoeffding 95\% confidence set is therefore
\[
 \left[\widehat A-8384\sqrt{\frac{\log40}{128}},\,
       \widehat A+8384\sqrt{\frac{\log40}{128}}\right]
       \cap[0,4192].
\]
This inference is conditional on the four fixed checkpoints, rather
than a population of independently trained networks.

\paragraph{Results and audit.}
All 64 paths exit before 4,000: the largest exit step is 1,145 among
the 32 target paths and 2,988 among the 32 proposal paths. At 1,000,
7 target paths and all 32 proposal paths survive. Thus the proposal
produces longer observed histories, but none reaches the primary
tail window. The primary estimate is exactly zero and its interval is
$[0,1423.290389]$ updates (upper endpoint rounded upward for display).
This broad interval does not establish a zero population tail area.
The weighted RMST at 8,192 is 827.981776 updates; it is descriptive
and is not a complete-MFPT estimate. The stopped-weight mean is
0.999999945 and is used only as a diagnostic, without self-normalization.

There are no numerical failures, horizon-censored paths or budget
interruptions. All eight prescribed full replays match. The run uses
125,646 updates including replays and 0.708835 GPU allocation hours.
All 279 transferred files match their recorded hashes. Independent
sampler reconstruction checks 125,646 rows, 72 terminal states and
39 intermediate states. A 384-bit integer-interval recurrence audits
all 111,974 original-path weights and scalar aggregates; the maximum
relative discrepancy of the frozen floating-point calculation is below
$8.19\times10^{-15}$. This scalar check does not certify real-arithmetic
network trajectories. No sample is pooled with the preceding study,
and the frozen sample size, proposal and stopping rule are unchanged.
Neither this finite-window result nor the proposal's observed exits
bound $\E[(\tau-8192)_+]$ or cover all surviving histories. The
original-radius ViT complete-MFPT problem remains unresolved.

\subsection{A compact-cover alternative to one-step drift}
\label{app:vit-compact-cover}
Failure of a particular one-step potential does not rule out
state-dependent finite sequences of batches that force exit.
The following standard compactness argument states the uniform
condition such a route must actually verify. It uses the finite-map
viewpoint of finite-pool SGD, without importing separability or
monotonicity assumptions from \citet{shirokoff2025separable}.

\begin{proposition}[Robust finite-word cover]
\label{prop:vit-compact-cover}
Consider continuous exact-real full-state update maps $F_a$ for a
finite set of legal batches. The full state includes momentum and,
for persistence, the held batch. Let $K$ be a compact set containing
every reachable surviving state. Suppose that for each $s\in K$
there is a finite batch word $w_s$ whose execution from $s$ has a
prefix with $\|\theta-\theta_*\|>R$. Assume every fixed finite word
has a history-uniform positive probability under the original sampler.
Then there are finite $L$ and $q>0$ such that, at every surviving history,
\[
 \Pp(\tau>n+kL\mid\mathcal H_n)\le(1-q)^k,
 \qquad \E[\tau-n\mid\mathcal H_n]\le L/q.
\]
The assertions are restricted to $\{\tau>n\}$.
\end{proposition}
\begin{proof}
A strictly exterior prefix stays exterior on a neighborhood of its
starting state by continuity of the finite composition. These
neighborhoods cover $K$; choose a finite subcover with words
$w_1,\ldots,w_m$. Set $L=\max_i|w_i|$ and let $q$ be the minimum
of their positive uniform probability bounds. Each surviving state
belongs to a covered neighborhood, so at least one of these events
forces exit within $L$ updates. Iterated conditioning gives the
geometric bound, and summing the integer survival probabilities in
blocks gives $L/q$. No disjointness or addition of word probabilities
is assumed.
\end{proof}

For the original momentum update, survival at $n\ge1$ implies
$h\|u_n\|=\|\theta_n-\theta_{n-1}\|<2R$.
Consequently the closure of all reachable surviving full states is
compact: the parameter lies in the closed $R$-ball, momentum is
bounded by $\max\{\|u_0\|,2R/h\}$, and sampler memory is finite.
This observation does not identify that closure or certify a cover
of it. It can contain limiting states never reached in finite time.
The strict exterior inequality in the proposition is stronger than
the original boundary-touching exit rule and is what makes each
witness robust to perturbing its starting state. Continuity here
concerns the ideal-real smooth ViT, not a rounding certificate for
binary32 execution.

\paragraph{Exact word probabilities under persistence.}
For $J\ge2$ legal batches and $0<p\le1$, the batch transition is
$Q(a,c)=(1-p)\mathbf1_{\{a=c\}}+p/J$.
A specified word of length $\ell$ with $r$ consecutive constant-batch
runs therefore has minimum probability over the unknown incoming
batch
\[
 q_w=(p/J)^r(1-p+p/J)^{\ell-r}.
\]
The first batch and each change contribute $p/J$; repeating the
previous batch contributes $1-p+p/J$. For a uniformly initialized
first batch the probability is $q_w/p$, and $p=1$ recovers
$J^{-\ell}$. The case $J=1$ has probability one. These events belong
to the original sampler and do not impose a different control policy.

\paragraph{Why pointwise escapability is insufficient.}
Let $h=1$, $\beta=0$, $z_0=0$, and the exit domain be
$(-1/2,1/2)$. Two equally likely quadratic components,
\[
 f_+(z)=-z^2/2-z/2,\qquad f_-(z)=z^2/4+z/4,
\]
give SGD maps $z\mapsto2z+1/2$ and $z\mapsto z/2-1/4$.
Writing $x=z+1/2$ gives $x_0=1/2$ and maps $x\mapsto2x$ and
$x\mapsto x/2$. Every reachable surviving state $x=2^{-k}$ has
a finite exit word, but their closure contains the common boundary
fixed point $x=0$, which has no strictly exterior word.
Until exit, $k=-\log_2x$ is a symmetric random walk starting at one,
with exit at zero. Its first hitting time $\sigma_m$ of $\{0,m\}$
satisfies $\Pp_1(k_{\sigma_m}=m)=1/m$ and
$\E_1\sigma_m=m-1$, as follows from the finite difference equations
for $k/m$ and $k(m-k)$. Finite-interval absorption occurs almost
surely. Thus $\Pp(\tau=\infty)\le1/m$ for all $m$, while
$\sigma_m\le\tau$ implies $\E\tau\ge m-1$ for all $m$.
Hence exit is almost sure but its mean is infinite. This scalar,
zero-momentum example isolates the missing uniformity; it is not
a counterexample for the original ViT or the positive tail theorems.

\paragraph{Audit and remaining certificate.}
Exact rational checks cover 3,654 finite words and 90 probability
normalizations; independent enumeration of hold/refresh decisions
checks 450 word probabilities. Independent absorbing-chain solves
and Catalan first-passage counts check the random-walk calculations.
For the 128 archived ViT paths, we also audit the corresponding
\emph{unordered subset} words with $J=\binom{4096}{64}$.
Their unknown-memory $\log_{10}q_w$ ranges are approximately
$[-89663.58,-76469.41]$ for iid and $[-3166.90,-2016.02]$
for $p=0.02$; the latter words have 14--22 constant-batch runs.
These are probabilities of specified subset sequences, not exit
probabilities, and the logarithms are display approximations.
Canonical-subset reconstruction checks all 91,663 recorded batches;
it does not certify an ordered floating-point network execution.
No new optimizer update or training unit is generated by this audit.

The original-radius ViT still lacks a finite cover of the reachable
closure and rigorous full-network strict-exit witnesses for its cells.
Neither the observed words nor their probabilities supply these
premises. Even a valid cover based on single specified words could
give an impractically loose bound; larger certified event families
would require separate justification. This sufficient condition
therefore does not close the full-MFPT claim or assert a new
compactness or geometric-tail principle.

\section{Statistical scope, reset reversals and a spatial--local factorial control}
\label{app:priority-controls}
These checks address three distinct concerns. The reset analysis uses
already observed trajectories; the factorial uses previously trained
small networks with newly assigned continuation streams. Both are
exploratory. Their analysis plan was recorded before executing these
analyses, after earlier outcomes were known. Neither provides new
independent training replicates or additional confirmatory tests.

\subsection{Which estimand each result supports}
The original Studies 1--8 have only 3--4 training seeds per architecture.
Their empirical RMST ratios, nominal bootstrap ranges and post-hoc
$t$ ranges remain descriptive; historical declaration rules are preserved
without retaining population significance interpretations. Changing the
resampling method or correcting multiplicity cannot establish small-sample
coverage. The newer replication tests a different quantity: the probability
of a finite assigned training/burn-in run jointly with the specified
finite-path empirical direction. The newer mechanism tests similarly
include validity and direction, and the robustness endpoint requires the
direction at all three lengths. None tests a population mean RMST ratio.
The training start, not a continuation path or an architecture/contrast
cell, is the independent unit. Family-wise corrections are specific to
each frozen family and are not pooled into a paper-wide guarantee.

\subsection{Reconstructing the reset reversal}
We re-read all 2,304 robustness paths and reconstruct the first crossings
from 955,421 saved displacement records at both original radii and at
caps 500 and 1000. All crossings match the archived times; the numerical
validity findings and original negative primary results are unchanged.
Table~\ref{tab:reset-cap-diagnostic} reports paired seed differences for
the primary radius. Positive values mean that refresh resets have larger
empirical capped means than random resets. Each network/length uses all
16 assigned starts and eight paired continuations per start.

\begin{table}[ht]
\centering
\caption{Exploratory refresh-minus-random RMST differences in updates.
Counts are positive/tied/negative seed differences, not independent path
tests. Censoring columns refer to refresh/random arms at cap 1000.}
\label{tab:reset-cap-diagnostic}
\small
\begin{tabular}{llrrrrr}
\toprule
Network & Length & $\Delta_{500}$ & Signs$_{500}$ & $\Delta_{1000}$ & Signs$_{1000}$ & Censoring (\%) \\
\midrule
MLP & 0.5 & 134.461 & 16/0/0 & 558.258 & 16/0/0 & 30.47/0.00 \\
MLP & 1 & 102.391 & 16/0/0 & 102.391 & 16/0/0 & 0.00/0.00 \\
MLP & 2 & -35.375 & 0/0/16 & -35.375 & 0/0/16 & 0.00/0.00 \\
CNN & 0.5 & 0.234 & 1/15/0 & 212.117 & 16/0/0 & 96.09/3.91 \\
CNN & 1 & -73.867 & 0/0/16 & -73.867 & 0/0/16 & 0.00/0.00 \\
CNN & 2 & -28.844 & 0/0/16 & -28.844 & 0/0/16 & 0.00/0.00 \\
\bottomrule
\end{tabular}

\end{table}

For MLP, the difference changes from $+102.391$ at calibrated length to
$-35.375$ at twice that length, with zero censoring in both arms and at
both caps. Thus this reversal cannot be explained solely by censoring.
For CNN at half length, 15 of 16 seed differences are tied at cap 500;
at cap 1000 all 16 are positive but refresh-reset censoring is 96.094\%.
The latter is a capped-mean result, not a conclusion about the unobserved
tail. All 24 network/length/radius/cap contrasts, including the secondary
radius, are retained in the accompanying analysis JSON; no selected
window replaces the frozen primary endpoint.

\paragraph{What scalar traces can diagnose.}
Full update vectors were not saved, so adjacent direction correlations,
gradient alignments and exit-direction distributions cannot be recovered.
Writing $r_n=\|\theta_n-\theta_0\|$ and
$\ell_n=\|\theta_n-\theta_{n-1}\|$, the norm identity gives
\[
 c_n=\frac{r_n^2-r_{n-1}^2-\ell_n^2}{2r_{n-1}\ell_n},\qquad
 q_K=\frac{r_K}{\sum_{n=1}^K\ell_n}.
\]
Here $c_n$ is radial alignment when the denominator is nonzero, and
$q_K$ is net displacement divided by traveled distance; neither is
adjacent-step persistence. We evaluate the prespecified exploratory
$K=16$ prefix for every path (the shortest saved path has 18 updates),
avoiding a comparison conditional on surviving to different horizons.
These reconstructions use floating-point norms, not strict certificates.
For MLP, mean $q_{16}$ for refresh/random resets is
$0.544/0.522$, $0.507/0.464$ and $0.475/0.411$ across the three lengths.
Refresh has greater early net progress at all lengths, despite its slower
exit at the first two. This rules out the simple descriptive explanation
that its larger capped mean is always accompanied by less early progress.
For CNN the corresponding pairs are $0.499/0.520$, $0.478/0.455$ and
$0.464/0.414$. Their qualitative agreement with CNN's ranking is not a
mediation analysis. The data establish length- and window-sensitive
ordering, without identifying a unique dynamical cause.

\subsection{Separating spatial restriction from frozen local information}
\label{app:factorial-control}
We reuse all 16 trained 45-parameter tanh inputs of
Appendix~\ref{app:fullparam-followup} (training seeds 200--215).
There is no new training or choice of successful starts. Eight new
continuation streams (identifiers 41000--41007), iid and persistent
sampling with $p=1/4$, yield 1,024 trajectories across four arms.
All arms share each realized batch sequence, center, zero initial
momentum, $\beta=0.9$, $h=10^{-4}$, batch size 8 and cap $J=500$.
Both radii $R=0.001$ and $0.01$ are reported. The event in every arm is
the same 45-dimensional ball around the same center.

Let $e$ be a top eigenvector of the per-example gradient covariance at
the center, $P=ee^\top$, and $z=\theta-\theta_0$. The two forces are
\[
 G_B^{\rm nl}(z)=\frac1{|B|}\sum_{i\in B}\nabla\ell_i(\theta_0+z),
 \qquad
 G_B^{\rm aff}(z)=\frac1{|B|}\sum_{i\in B}
 \{\nabla\ell_i(\theta_0)+\nabla^2\ell_i(\theta_0)z\}.
\]
The full arms use $v_{n+1}=\beta v_n-G_B(z_n)$,
$z_{n+1}=z_n+hv_{n+1}$. The line arms replace the force by $PG_B(z_n)$;
zero initial states then keep them on the line. Thus local freezing
retains each example's Hessian, rather than replacing all gradients by
constants. The original neural S1--S4 also differ in forcing, curvature
and event construction: this controlled intervention is not a numerical
decomposition of their errors. Restricting a trajectory to the line is
also different from observing the projection of a full trajectory.

\begin{table}[ht]
\centering
\caption{Factorial control: descriptive mean seed-level RMST in updates
and censoring percentage in parentheses. FN/FA: full nonlinear/affine;
LN/LA: line nonlinear/affine. All 16 reused training starts are included;
there are no confidence or significance statements.}
\label{tab:factorial-control}
\small
\begin{tabular}{llrrrr}
\toprule
Sampler & Radius & FN & FA & LN & LA \\
\midrule
iid & 0.001 & 41.102 (0.00) & 41.102 (0.00) & 54.477 (0.00) & 54.477 (0.00) \\
iid & 0.01 & 493.406 (92.19) & 493.398 (92.19) & 498.594 (98.44) & 497.602 (97.66) \\
$p=1/4$ & 0.001 & 23.047 (0.00) & 23.047 (0.00) & 28.164 (0.00) & 28.164 (0.00) \\
$p=1/4$ & 0.01 & 284.469 (26.56) & 284.469 (26.56) & 308.102 (35.94) & 308.125 (35.94) \\
\bottomrule
\end{tabular}

\end{table}

At $R=0.001$, no path is censored; full affine and nonlinear arms have
identical integer first exits on all sampled paths, as do the two line
arms. Restricting the nonlinear dynamics increases the average capped
mean by 13.375 updates for iid and 5.117 for persistent sampling.
Equality of grid exit times does not imply identical trajectories or
zero approximation error between updates. At $R=0.01$, mean absolute
seed-level FA--FN discrepancies are 0.0078 and 0.0313 updates, whereas
LN--FN discrepancies are 5.1875 and 23.6328. The iid large-radius cells
are 92.2--98.4\% censored, so small capped-mean errors do not establish
accuracy for complete mean exit time. The four-arm signed interaction
$(\overline T_{\rm LA}-\overline T_{\rm LN})-
(\overline T_{\rm FA}-\overline T_{\rm FN})$ is zero at the small
radius and $-0.9844$ (iid) or $0.0234$ (persistent) at the larger one;
we therefore do not force an additive error attribution.

This example shows that spatial restriction alone can materially change
observed exits while a full-dimensional local approximation remains close
on the same finite paths. It does not establish that projection dominates
for the original deep networks or that freezing is always accurate. The
original smooth-validation failure is unchanged. All paths are finite;
saved first exits and batch schedules were checked, and an independent
autograd implementation replays 24 trajectories over 12,000 updates.
The supplement records numerical tolerances and hashes. Computation used
local CPU only and did not expand the GPU experiment budget.

\section{Supplementary benchmarks and an earlier neural diagnostic}
\label{app:extended-validation}
This appendix collects measurements that the main results do not rely
on: timing and precision studies of the short-memory variant of
Appendix~\ref{sec:shortmemory}, and a momentum-calibrated telegraph
diagnostic whose prospective test failed.

\paragraph{Timing of the short-memory variant.}
With pilot-selected $L=32$, a frozen experiment uses 64 new drivers on
the same eight pools. After separate-stream warm-up, each method is
timed three times per driver in randomized order (384 calls). Timing
includes batch matrices, reference evolution, envelopes and endpoints
at all three radii, plus the short-memory Lyapunov solve. It excludes
pool calibration, driver generation, actual-process validation and I/O.
The median paired online speedup is 2.50, with a driver-stratified bootstrap interval [2.41, 2.59]. Mean full- and short-memory widths are 5.740 and 6.401 updates. Their ratio is 1.115, with interval [1.019, 1.290]. The predefined point targets (speedup at least two and width ratio
at most 1.2) are met, but the latter interval crosses 1.2. Thus this
budget does not establish mean inflation below twenty percent. The
bootstrap resamples drivers within fixed pool--sampler strata; thermal
and system-load dependence remain, precluding a cross-machine guarantee.
Of 192 radius observations,
162 have zero full-memory width and three become positive under short
memory. No stopped-error or envelope-order check fails at the stated
floating-point tolerance. These fresh streams validate a pilot-selected
candidate conditionally on reused pools, not on new training runs.

\paragraph{An independent precision study.}
A separate protocol uses 1,024 new drivers (64 per pool--sampler stratum)
and the selected $L=32$ method; it neither extends the speed benchmark
nor retests width inflation. The target equally weights eight retained
pools, two samplers and three radii. Averaging radii within drivers,
we allocate risk 0.05 across 98 one-sided endpoints (48 cells and one
mixture, each with lower and upper bounds), using the bounded
KL-Chernoff construction in Appendix~\ref{sec:bounded-precision}.
The mixture's conditional RMST enclosure is [101.70, 133.70] updates, of width 32.00, compared with the nominal mean endpoints [113.78, 122.13]. The predefined mixture-width target of 40 updates is met. Only 7 of the 48 individual cell enclosures have width at most 40; their widths range from 22.23 to 113.43. A precise mixture does not imply precise individual settings.
The coverage argument is conditional on fixed calibrated pools and
valid real-arithmetic pathwise bounds. Floating-point norms and root
solves remain unverified by outward rounding. This study's 95\%
family is separate from other confidence statements in the paper, and
it provides no inference over newly trained models.

\paragraph{Cost relative to direct simulation.}
A post-hoc check on 16 saved drivers, with warm-up and randomized
repeated method order, gives median times of 0.00454 seconds for a
200-update actual trajectory and 0.18129 seconds for short-memory
reference and bound computation; the median paired cost ratio is 39.8.
On the 1,024 precision-study drivers, a matched 100-statement allocation
gives the direct-mixture interval [105.84, 129.63], versus
[101.68, 133.71] from model brackets. The narrower direct interval
follows from same-path endpoint ordering and confidence-bound
monotonicity, not an empirical superiority test. The additional cost
buys a reference--process discrepancy bound, not cheaper exit
observations. Thus these low-dimensional results support controlled
comparison, not cheaper mean estimation; the full-memory speedup has
that restricted scope.

\subsection{A momentum-calibrated telegraph diagnostic}
We also tested prospectively a diagnostic that calibrates telegraph
persistence from the momentum filter. Twenty-four MLPs (784--128--10,
ReLU) were trained on MNIST \citep{lecun1998} and Fashion-MNIST
\citep{xiao2017}, six per dataset and momentum $\beta\in\{0.5,0.9\}$,
with learning rate 0.01 and batch size 128. Epochs 1 and 4 give 48
checkpoints, each continued 32 times for 200 updates on a disjoint
example pool (1,536 paths). That pool is disjoint from the training
data, so its mean gradient at a checkpoint need not vanish, and the
exits mix this drift with sampling noise. The event is the exit of the projection on
the frozen descending-momentum direction from the interval on which the
probe loss rises by less than 5\%. The diagnostic is
$\chi=s/[\lambda\min(L_+,L_-)]$, with $\lambda=-\log\beta/2$,
$s=\sqrt{2\lambda D}$, the projected gradient diffusivity $D$ and the
interval half-widths $L_\pm$, all frozen before the continuations. The
primary endpoint is the association of $\chi$ with the absolute error of
a diffusion RMST prediction, $|R^{\rm diff}_{200}-\widehat R^{\rm
real}_{200}|/200$; runs are the inference units and a whole-run
bootstrap keeps both checkpoints of a run together. A first execution
shared optimizer state between branches; it was excluded in full and
rerun after a recorded correctness amendment.

Across 48 valid primary checkpoints from 24 runs, the Spearman correlation between $\chi$ and the absolute diffusion RMST error is $-0.656$, with whole-run bootstrap interval $[-0.780,-0.500]$, whereas the declared diagnostic required a positive association: a larger $\chi$ was to predict a larger diffusion error.
\begin{table}[h]
\centering
\caption{Prospective mean absolute RMST error of each reference model,
divided by 200 and weighted equally over runs.}
\small
\begin{tabular}{@{}lrrrr@{}}
\toprule
Reference model & Diffusion & Telegraph & AR(1) & Block bootstrap\\
\midrule
Error & 0.066 & 0.068 & 0.059 & 0.063\\
\bottomrule
\end{tabular}
\end{table}

The association is negative, so the practical diagnostic claim is
withdrawn. The prespecified absolute normalization differs from the
relative error of the telegraph reference identity: at fixed $\lambda$
and $D$, the absolute correction $1/(\lambda\chi)$ decreases with $\chi$
while the relative correction $\chi/(1+\chi)$ increases. The negative
association therefore does not by itself refute that identity, but the
diagnostic as specified failed. The candidate references also differ in
their velocity state: AR(1) and a block bootstrap inherit the checkpoint
velocity, telegraph its sign, and diffusion none. Exploratory
calibration sensitivity with block lengths 8, 32 and 128 gives intervals
crossing zero for the remaining improvement.

\section{Separating Monte Carlo precision from pathwise width}
\label{sec:bounded-precision}
Fix the calibrated pools and all numerical comparison methods. Let
$Z_1,\ldots,Z_n\in[0,1]$ be independent, not necessarily identically
distributed, and $\mu=n^{-1}\sum_i\mathbb E Z_i$. Write
$d(x,p)=x\log(x/p)+(1-x)\log((1-x)/(1-p))$ with its continuous
boundary conventions. For observed mean $\bar Z$, define $l,u$ by
\[
 d(\bar Z,l)=d(\bar Z,u)=\frac{\log(1/\alpha)}n,
 \qquad 0\leq l\leq\bar Z\leq u\leq1,
\]
using $l=0$ at $\bar Z=0$ and $u=1$ at $\bar Z=1$.
Each endpoint has one-sided error probability at most $\alpha$.
This is the standard bounded-variable Chernoff construction, not a
new concentration result.

\paragraph{Justification.}
Convexity gives $e^{tz}\leq1-z+ze^t$ for $0\leq z\leq1$.
Independence and concavity of $p\mapsto\log(1-p+pe^t)$ imply
$\mathbb E\exp(t\sum_iZ_i)\leq(1-\mu+\mu e^t)^n$.
Optimizing the exponential Markov bound gives
$\Pr(\bar Z\geq x)\leq e^{-nd(x,\mu)}$ for $x>\mu$, and
the corresponding lower tail for $x<\mu$. Inverting these bounds
proves the endpoint statement. A union bound allocates risk across
a fixed collection of endpoints. This is a frequentist coverage
statement over repeated independent drivers, not a posterior probability.

\paragraph{RMST and width are different targets.}
For $J>1$ and a deterministic pathwise comparison $L_i\leq\tau_i\wedge J\leq U_i$
with $L_i,U_i\in[1,J]$, applying the bound to
$(L_i-1)/(J-1)$ and $(U_i-1)/(J-1)$ gives a confidence enclosure for
the actual conditional RMST. Applying it to $(U_i-L_i)/(J-1)$ instead
bounds the expected width of the comparison; that width alone is not
an enclosure for the exit mean. CDF endpoints can be treated as
Bernoulli observations, with separate risk allocation.

For the predeclared equally weighted mixture of retained pools, samplers
and radii, first average all radius observations within a driver.
The resulting driver-level variables remain bounded. Independent
drivers drawn in equally sized fixed strata estimate the equally
weighted mixture even though their distributions differ. Radius
observations and repeated timings must not be counted as extra draws.
This target is conditional on retained pools, not a population of new
training runs. Post-hoc audits of fixed archived methods are labeled
as audits, not new confirmatory experiments. The real-arithmetic
coverage argument does not turn ordinary floating-point pathwise bounds
or root solves into outward-rounded numerical certificates.

\end{document}